\documentclass{article}

\usepackage{microtype}
\usepackage{graphicx}
\usepackage{booktabs} 
\usepackage{subcaption}
\usepackage{hyperref}

\usepackage[accepted]{icml2024}

\usepackage{amsmath}
\usepackage{amssymb}
\usepackage{mathtools}
\usepackage{amsthm}
\usepackage{bbm}
\usepackage{multirow}
\usepackage[capitalize,noabbrev]{cleveref}
\definecolor{LightCyanBG}{rgb}{0.85,0.92,0.97}
\definecolor{CyanBG}{rgb}{0.71,0.85,0.93}
\definecolor{BlueBG}{rgb}{0,0.46,0.71}

\def\*#1{\mathbf{#1}}

\theoremstyle{plain}
\newtheorem{theorem}{Theorem}[section]
\newtheorem{proposition}[theorem]{Proposition}
\newtheorem{lemma}[theorem]{Lemma}
\newtheorem{corollary}[theorem]{Corollary}
\theoremstyle{definition}
\newtheorem{definition}[theorem]{Definition}
\newtheorem{assumption}[theorem]{Assumption}
\theoremstyle{remark}
\newtheorem{remark}[theorem]{Remark}

\ifx\theorem\undefined
\newtheorem{theorem}{Theorem}

\newtheorem{lemma}[theorem]{Lemma}

\newtheorem{definition}[theorem]{Definition}

\newtheorem{assumption}{Assumption}
\fi

\DeclareMathOperator*{\argmax}{argmax}
\DeclareMathOperator*{\argmin}{argmin}

\DeclareMathOperator*{\ACC}{ACC}

\def\cD{\mathcal{D}}

\def\cL{\mathcal{L}}
\def\cR{\mathcal{R}}

\def\R{\mathbb{R}}
\def\E{\mathbb{E}}

\def\nd{d} 

\def\v{\mathbf{v}} 
\def\z{\mathbf{z}} 
 
\def\x{\mathbf{x}}

\def\T{\mathbf{T}} 
\def\S{\mathbf{S}} 
\def\bS{\bar{\mathbf{S}}} 
\def\t{\mathbf{t}} 
\def\s{\mathbf{s}} 
\def\bs{\bar{\mathbf{s}}} 
\def\q{\mathbf{q}} 
\def\K{\mathbf{K}} 
\def\bK{\bar{\mathbf{K}}} %
\def\k{\mathbf{k}} 
\def\bk{\bar{\mathbf{k}}} %

\def\bZ{\bar{\mathbf{Z}}} 
\def\Q{{\mathbf{Q}}} 
\def\bz{\bar{\mathbf{z}}}

\def\V{\mathbf{V}} 

\usepackage[textsize=tiny]{todonotes}

\icmltitlerunning{Understanding Retrieval-Augmented Task Adaptation for Vision-Language Models}

\begin{document}

\twocolumn[
\icmltitle{Understanding Retrieval-Augmented Task Adaptation \\for Vision-Language Models}



\icmlsetsymbol{equal}{*}

\begin{icmlauthorlist}
\icmlauthor{Yifei Ming}{yyy}
\icmlauthor{Yixuan Li}{yyy}
\end{icmlauthorlist}

\icmlaffiliation{yyy}{Department of Computer Sciences, University of Wisconsin-Madison}

\icmlcorrespondingauthor{Yifei Ming}{ming5@wisc.edu}
\icmlcorrespondingauthor{Yixuan Li}{sharonli@cs.wisc.edu}

\icmlkeywords{Machine Learning, ICML}

\vskip 0.3in
]



\printAffiliationsAndNotice{}  

\begin{abstract}
Pre-trained contrastive vision-language models have demonstrated remarkable performance across a wide range of tasks. However, they often struggle on fine-trained datasets with categories not adequately represented during pre-training, which makes adaptation necessary. Recent works have shown promising results by utilizing samples from web-scale databases for retrieval-augmented adaptation, especially in low-data regimes. Despite the empirical success, understanding how retrieval impacts the adaptation of vision-language models remains an open research question. In this work, we adopt a reflective perspective by presenting a systematic study to understand the roles of key components in retrieval-augmented adaptation. We unveil new insights on uni-modal and cross-modal retrieval and highlight the critical
role of logit ensemble for
effective adaptation. We further present theoretical underpinnings that directly support our empirical observations. 
\end{abstract}

\section{Introduction}

Contrastive vision-language pre-training has emerged as a fundamental cornerstone for a wide array of tasks in natural language processing and computer vision~\cite{radford2021learning, jia2021scaling, yang2022unified,li2022supervision,mu2022slip,yu2022coca, sun2023eva, xu2024demystifying}. These models excel in capturing the intricate relationships present in both visual and textual data, enabling them to understand context, semantics, and associations holistically. 
It is now a common practice to employ aligned multi-modal features from web-scale pre-training.  However, a challenge arises when these pre-trained models encounter real-world downstream datasets, particularly in low-data (few-shot) scenarios. Such datasets often encompass fine-grained categories that were not adequately represented during the initial pre-training phase, posing a notable hurdle for the models in adapting to these nuanced distinctions.

In the low-data regime, retrieval-augmented adaptation has demonstrated promise, where a wealth of external resources is readily available on the Internet and can be retrieved efficiently to enhance adaptation. Recent works~\cite{udandarao2023sus, zhang2023prompt} 
showcase encouraging results by leveraging large-scale text and image databases~\cite{schuhmann2022laion}.
Retrieval-augmented adaptation involves two main steps: first retrieving the most relevant data from an external source, and then adapting to downstream task based on the retrieved samples. While existing works have primarily focused on developing new adaptation algorithms or integrating different knowledge sources, \emph{\textbf{there remains a notable gap in understanding how retrieval augmentation impacts adaptation for vision-language models}}.  Such an understanding is imperative to guide the future development of effective algorithms.

In this work, we adopt a reflective perspective by presenting a systematic study to understand retrieval-augmented adaptation, and establishing new theoretical underpinnings.   
Our empirical analysis reveals key insights revolving around two aspects: \textbf{(1)} the impact of the retrieval method, and \textbf{(2)} how retrieved samples help adaptation. First, we show that image-to-image (I2I) retrieval consistently outperforms text-to-image (T2I) retrieval for a wide range of downstream tasks. Under the same retrieval budget, these two retrieval methods differ by the query samples used: I2I employs a few seed images from the target data distribution, whereas T2I employs the textual description of each class label. While both I2I and T2I retrieval introduce distributional shifts \emph{w.r.t.} the target data, we show that I2I achieves strong performance that matches more closely with the oracle when we directly retrieve from the target distribution (\emph{i.e.}, no distributional shifts). Secondly, we show that ensembling the zero-shot prediction together with I2I retrieved samples is the key to improved adaptation performance. For a given test sample, the ensembling is achieved by taking a weighted average between the logit from the retrieved feature cache and the logit of the zero-shot inference. We empirically find that without ensembling, the performance of retrieval-augmented adaptation significantly degrades.  This new observation complements previous studies that often attribute the success of
retrieval to the diversity and quality of samples.

\begin{figure*}[!ht]
\centering
    \includegraphics[width=\textwidth]{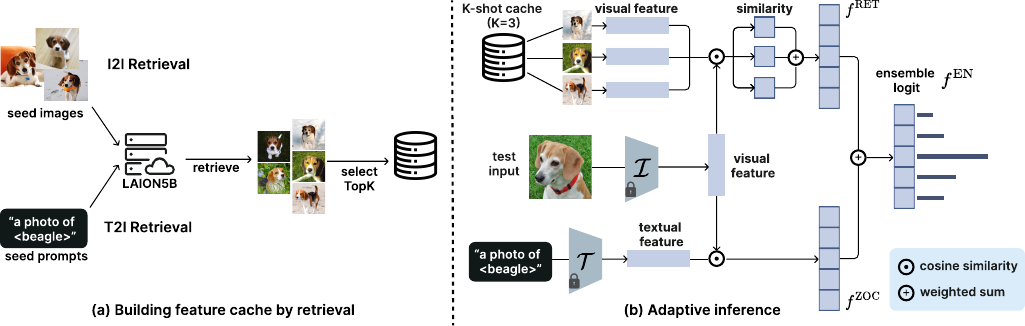}
    \vspace{-0.5cm}
\caption{Illustration of the retrieval-augmented task adaptation framework for CLIP-like models. (a): Given a downstream target dataset, we first retrieve relevant samples from a web-scale database using seed prompts (T2I) or seed images (I2I). We can then build a K-shot cache by selecting the Top-K similar images per class based on CLIP embeddings. (b) At inference time, the final logit $f^{\text{EN}}$ of a test input is an ensemble (weighted sum) of logits from the zero-shot model $f^{\text{ZOC}}$ and the few-shot cache $f^{\text{RET}}$. }
\label{fig:teaser}
\end{figure*}

Going beyond empirical analysis, we provide theoretical insights that directly support our empirical observations above. 
We formalize T2I and I2I retrieval by characterizing the multi-modal feature space with each retrieval scheme. Under realistic assumptions, we analyze how retrieval impacts the modality gap and the shift between the retrieved and target distributions. In particular, we prove that I2I retrieval is superior to T2I retrieval (\textbf{Theorem~\ref{i2i}}) and that logit ensemble is critical for improving CLIP-based adaptation (\textbf{Theorem~\ref{ensemble}}) by better leveraging the knowledge encoded in different modalities. Our theoretical results shed light on the key factors in the design of effective retrieval-augmented adaptation algorithms for vision-language models.

Our main contributions are summarized as follows:
\begin{itemize}
    \item We conduct a timely and systematic investigation into the retrieval-augmented adaptation of vision-language models, where we highlight key components such as the retrieval methods and logit ensemble.
    \item We provide a finer-grained empirical study with in-depth analysis. We unveil new insights on the critical role of uni-modal retrieval and logit ensemble for effective CLIP-based adaptation in low-data scenarios.   
    \item We develop a novel theoretical framework for retrieval-augmented adaptation and present theoretical results that directly support our empirical observations.  
    \item We further provide a comprehensive ablation study and discuss alternative design choices such as the impact of model architectures, adaptation with a finetuned feature cache, and adaptation with data mixtures.
\end{itemize}

\section{Retrieval-Augmented Task Adaptation}
In this section, we first discuss the preliminaries of contrastive vision-language models as well as the external databases employed for retrieval (Section~\ref{sec:prelim}). Next, we illustrate the two main steps for retrieval-augmented task adaptation: building a feature cache by retrieving relevant samples from the external database (Section~\ref{sec:cache}), and performing task adaptation based on retrieved samples (Section~\ref{sec:adaptation}). An illustration of the pipeline is shown in Figure~\ref{fig:teaser}.
\subsection{Preliminaries}
\label{sec:prelim}
Popular contrastive vision-language models such as CLIP~\citep{radford2021learning}
adopt a dual-stream architecture with one text encoder $\mathcal{T}: t \rightarrow \mathbb{R}^d$ and one image encoder $\mathcal{I}: \*x \rightarrow \mathbb{R}^d$. The model is pre-trained on a massive web-scale image-caption dataset with a multi-modal contrastive loss, which aligns features from different modalities. 
This alignment of multi-modal embeddings offers distinct advantages for contemporary large-scale multi-modal vector databases~\cite{schuhmann2022laion}, enabling efficient retrieval based on semantic similarity. 

\paragraph{Zero-shot inference.} At inference time, given a test input $\*x$, we can obtain the cosine similarity  $f^{\text{ZOC}}_c(\*x) = \text{sim}(\mathcal{I}(\*x), \mathcal{T}(t_c))$ between the visual embedding $\mathcal{I}(\*x)$ and  contextualized representations  $\mathcal{T}(t_c)$ for each label $c\in \{1, 2,...,C\}$. Here the context $t_c$ can be either a generic template such as
``\texttt{a photo of <CLASS>}'' or a textual description of the class. We denote the logit vector of the zero-shot model as $f^{\text{ZOC}}(\x)\in\mathbb{R}^C$, which consists of $C$ cosine similarities. The class prediction can be made based on the maximum cosine similarity among $C$ classes.

\paragraph{External web-scale knowledge base.} Pre-trained CLIP models often struggle for downstream datasets with finer-grained categories, which are not well represented in the pre-training dataset. To adapt CLIP models to finer-grained datasets in a low-data scheme, recent works~\cite{liu2023learning} demonstrate promising performance by utilizing external resources such as LAION~\citep{schuhmann2022laion}, a web-scale knowledge base which consists of billions of image-text pairs $\mathcal{S}_{L} = \{(\*x_i,t_i)\}_{i=1}^N$ covering a diverse range of concepts in the real world. Given a fixed budget, we can efficiently build a few-shot cache by retrieving relevant samples from the knowledge base with approximate KNN search~\cite{johnson2019billion}. We provide details as follows.

\subsection{Building Feature Cache by Retrieval}
\label{sec:cache}
Given a downstream dataset with $C$ classes: $\mathcal{Y}=\{1, 2,...,C\}$ and a budget size of $KC$, we can retrieve $K$ samples per class to build a cache of size $KC$. For vision-language models, the retrieval methods be categorized as uni-modal and cross-modal retrieval, formalized as follows:

\paragraph{Uni-modal retrieval.} We mainly consider image-to-image (I2I) retrieval due to its popularity. For I2I retrieval, we assume access to a small set of query images from the downstream dataset. The query set $\mathcal{Q}_I = \bigcup_{c=1}^C \mathcal{Q}_I^c$, where $\mathcal{Q}_I^c = \left\{ \*x_{c,1}, \*x_{c,2}, \ldots, \*x_{c,n_c} \right\}$ contains $n_c$ seed images for each class $c\in\mathcal{Y}$. We then retrieve top-$K$ similar images from $\mathcal{S}_{L}$ per class: 
$$
\mathcal{R}^\text{I2I}(c) = \text{top}_K \left\{ \*x \in \mathcal{S}_{L} : \text{sim}(\mathcal{I}(\*x), \mathcal{I}(\*x_{c,i})), \*x_{c,i} \in \mathcal{Q}_I \right\},
$$
where $\text{sim}(\mathcal{I}(\*x), \mathcal{I}(\*x_{c,i}))$ is the cosine similarity between the image embedding of $\*x$ from retrieval database and the query image $\*x_{c,i}$, and $\text{top}_K$ denotes the operation of selecting the top-$K$ items.
We can build a $K$-shot cache for I2I retrieval by taking the union of these sets across all classes:
 $$\mathcal{S}^\text{I2I}_R = \bigcup_{c \in \mathcal{C}} \left\{ (\*x, t)\in \mathcal{S}_{L} : \*x \in \mathcal{R}^\text{I2I}(c) \right\}.$$

\paragraph{Cross-modal retrieval.} We mainly consider text-to-image (T2I) retrieval. We assume access to class names in the target dataset, also known as ``name-only transfer''  ~\citep{udandarao2023sus}. The query set $\mathcal{Q}_T = \left\{ t_c\right\}_{c=1}^C$, where $t_c$ is a generic textual description of  class $c$. The retrieved $K$ samples for class $c$ is: 
$$
\mathcal{R}^\text{T2I}(c) = \text{top}_K \left\{\*x \in \mathcal{S}_{L} : \text{sim}(\mathcal{I}(\*x),\mathcal{T}(t_c)), t_c \in \mathcal{Q}_T \right\},
$$
where $\text{sim}(\mathcal{I}(\*x),\mathcal{T}(t_c))$ is the cosine similarity between the image embedding of $\*x$ and the text embedding for class $c$. The $K$-shot cache for T2I retrieval is denoted as:  
$$
\mathcal{S}^\text{T2I}_R = \bigcup_{c \in \mathcal{C}} \left\{ (\*x, t)\in \mathcal{S}_{L} : \*x \in \mathcal{R}^\text{T2I}(c) \right\}.
$$

\subsection{Task Adaptation with Retrieved Samples} 
\label{sec:adaptation}
Given a $K$-shot cache ($\mathcal{S}^\text{I2I}_R$ or $\mathcal{S}^\text{T2I}_R$) and pre-trained CLIP image and text encoders $\mathcal{I}$ and $\mathcal{T}$, we can perform adaptation \emph{w.r.t.} a fine-grained target dataset.
To better understand the effects of retrieved samples, we consider zero-shot adaptation in Section~\ref{sec:fine-grain}, where the cache only consists of retrieved samples. We discuss few-shot adaptation in Section~\ref{sec:discuss}, where the cache contains a mixture of samples in the target training set and retrieved samples.

\paragraph{Retrieval-based adaptation.} A variety of cache-based adaptation methods have been recently proposed~\cite{zhang2022tip, zhang2023prompt, udandarao2023sus}. At the core, these methods typically obtain a logit ensemble for each test input based on two sources: (1) a logit from the zero-shot CLIP model, and (2) a logit from the cache. Without loss of generality, we consider a representative adaptation framework {TipAdaptor}  ~\citep{zhang2022tip}.
Specifically, given the cache of size $CK$ (consisting of $C$ classes with $K$ retrieved samples per class), we denote the collection of the visual features as $\mathbf{K} = [\*k_{1,1},\*k_{1,2},\cdots, \*k_{C,K}]\in \mathbb{R}^{d\times CK}$ where $\*k_{c,i} = \mathcal{I}(\*x_{c,i})$.
For each test input $\*x$, we can obtain $CK$ cosine similarities 
$s_{c,i}(\*x) = \text{sim}(\mathcal{I}(\*x),\*k_{c,i})$. The cosine similarities are then scaled by an exponential function $\tilde{s}: s \mapsto \exp(-\omega + \omega s)$  with a hyperparameter $\omega$ that modulates the sharpness. Accordingly,  we can obtain an average similarity vector for each class based on visual features, $f^{\text{RET}}_c(\*x) = { 1\over K}\sum_{i=1}^K \tilde{s}_{c,i}(\*x) $. The final logit of the test sample is an ensemble of logits from the feature cache and zero-shot CLIP prediction:
$$f^{\text{EN}}(\x) = \alpha f^{\text{ZOC}}(\x)  + \gamma f^{\text{RET}}(\x),$$
where $\alpha,\gamma$ weigh the relative importance between two logits. Such a logit ensemble scheme has also been commonly adopted in recent works~\cite{zhang2023prompt}.  For completeness, we also discuss learning-based adaptation by setting visual features in $\mathbf{K}$ as learnable parameters. We denote the method as \texttt{Ensemble(F)}, where F stands for fine-tuning.

\begin{figure*}[!ht]
\centering
    \includegraphics[width=\textwidth]{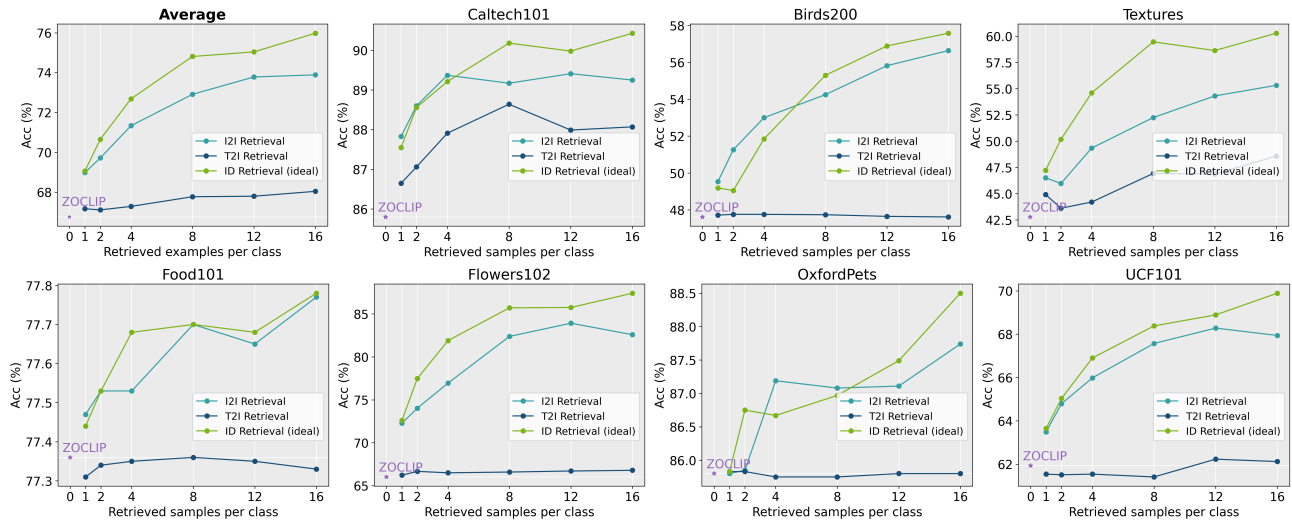}
\caption{Comparison of adaptation performance (in accuracy) of different retrieval methods.  Compared to the zero-shot model (purple star), I2I retrieval significantly improves the performance and consistently outperforms T2I retrieval across shots and datasets.}
\label{fig:tf-adapt}
\end{figure*}

\section{A Finer-Grained Analysis of Retrieval-Augmented Adaptation}
\label{sec:fine-grain}

Different from recent works on algorithm design and incorporation of new knowledge sources~\cite{zhang2023prompt, iscen2023retrieval, udandarao2023sus}, the goal of our 
 work is to present a systematic analysis with theoretical insights on how retrieval augmentation impacts adaptation for vision-language models. In this section, we present empirical analysis focusing on the impact of two aspects: retrieval method (Section~\ref{sec:retrieval_method}) and logit ensemble with retrieved samples (Section~\ref{sec:ensemble}). 
 We will provide theoretical analysis to support these empirical findings in Section~\ref{sec:theory}. 
 We discuss alternative design choices and ablation studies in Section~\ref{sec:discuss}.

\subsection{Settings}
\paragraph{Datasets.} Following prior works~\cite{zhang2022tip}, we consider a wide range of real-world datasets that span both common and finer-grained categories: Caltech101~\citep{FeiFei2004LearningGV}, Birds200~\citep{WahCUB_200_2011}, Food101~\citep{bossard2014food},  OxfordPets~\citep{parkhi12a}, Flowers102~\cite{nilsback2008automated}, Textures~\cite{cimpoi2014describing}, and UCF101~\cite{soomro2012ucf101}.

\vspace{-0.2cm}
\paragraph{Implementation details.} We use LAION-5B~\cite{schuhmann2022laion} as the retrieval database, which consists of 5.85 billion image-text pairs. For T2I retrieval, the default query set contains class descriptions with a prompt template.  For I2I retrieval, by default, we use 8 seed images per class as the query set. Based on the query set, we use the {clip-retrieval} tool\footnote{\url{https://github.com/rom1504/clip-retrieval}} for efficient retrieval from LAION-5B. We vary the number of retrieved samples per class $K\in\{1,2,4,8,16\}$. For adaptation, we use pre-trained CLIP with RN50 backbone as the default. Unless otherwise specified, each reported result is averaged over three independent runs. The ensemble weights of two logits $\alpha,\gamma$ are tuned on the validation set. Ablation studies on the number of seed images and alternative backbones are in Section~\ref{sec:discuss}. Further implementation details can be seen in Appendix~\ref{sec:add_detail}.

\subsection{Impact of Retrieval Method}
\label{sec:retrieval_method}

\paragraph{I2I retrieval consistently outperforms T2I retrieval.} To better understand the impact of the retrieval method, we compare the adaptation performance (in Accuracy) using I2I and T2I retrieval. The results are shown in Figure~\ref{fig:tf-adapt}, where the horizontal axis indicates the number of retrieved samples for each class (shot). As both I2I and T2I retrieval introduce distributional shifts \emph{w.r.t.} the target distribution, we also plot the oracle performance when retrieving samples from the target training set for reference, denoted as ID retrieval (green). Directly retrieving from the target training set can be viewed as performance upper bound. 

We observe several salient trends: \textbf{(1)} I2I retrieval consistently outperforms T2I retrieval across all shots and datasets. In particular, the gap between I2I and T2I increases when increasing the shot. \textbf{(2)} Compared to the zero-shot inference without knowledge augmentation (purple star), I2I retrieval significantly improves the performance. Notably, the gap between I2I retrieval and ID-retrieval (ideal) can be as small as 1\% on average (12 shots), highlighting the potential of utilizing retrieved samples in the extremely low-data scheme where one does not have training data in the target dataset. \textbf{(3)} While T2I retrieval obtains a diverse collection of samples, the performance gain compared to the zero-shot CLIP for multiple datasets can be marginal. We investigate the reasons by a detailed examination of retrieved samples next and provide theoretical understanding in Section~\ref{sec:theory} (Theorem~\ref{i2i}). Similar trends also hold for training-based adaptation, where we finetune the cache features as in ~\citet{zhang2022tip} (see Figure~\ref{fig:t-adapt} in Appendix~\ref{sec:train_adapt}).

\begin{figure}[!ht]
\centering
    \includegraphics[width=0.5\textwidth]{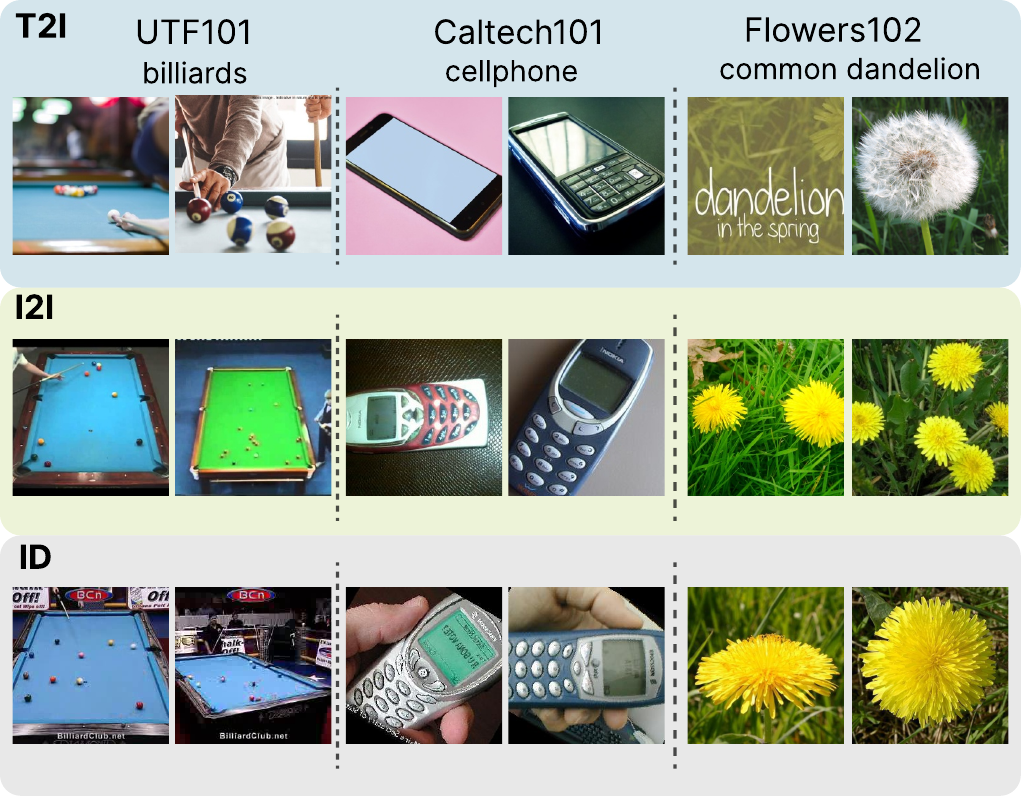}
    \vspace{-0.2cm}
    \caption{Samples from T2I and I2I retrieval. Top row: the main source of noise for T2I retrieval is semantic ambiguity, as the textual queries (\emph{e.g.}, \texttt{a photo of a cellphone}) may not accurately describe the images from target distributions (\emph{e.g.}, cellphones typical in the early 2000s). Middle row: samples retrieved by I2I matches more closely with ID data. Bottom row: images sampled from the target (ID) distribution. More examples can be seen in Appendix~\ref{sec:more_sample}. }
    \label{fig:error_sample}
\end{figure}

\begin{figure*}[t]
\centering

\begin{subfigure}[b]{0.24\textwidth}
    \includegraphics[width=\textwidth]{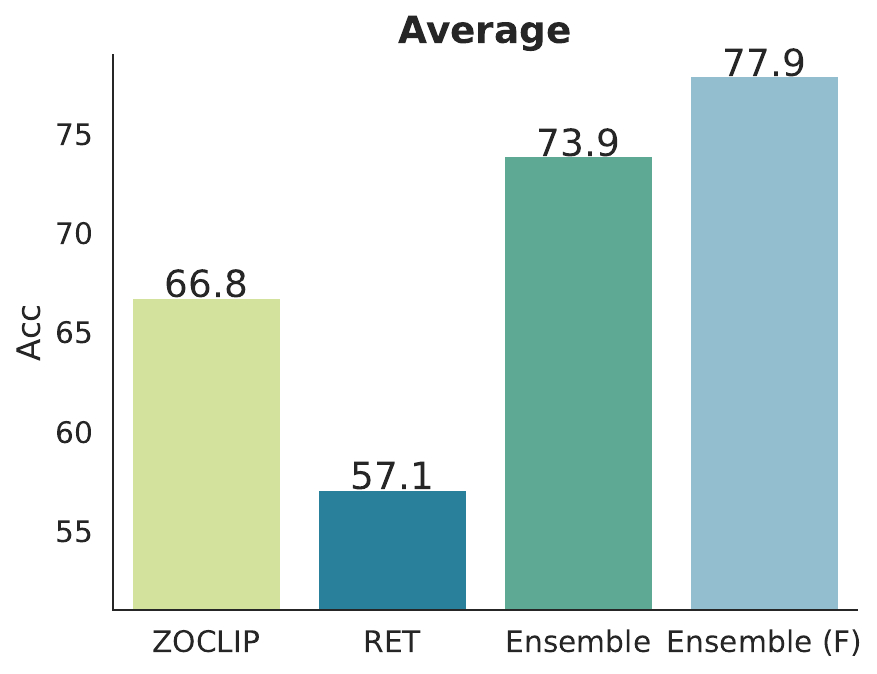}
\end{subfigure}
\begin{subfigure}[b]{0.24\textwidth}
    \includegraphics[width=\textwidth]{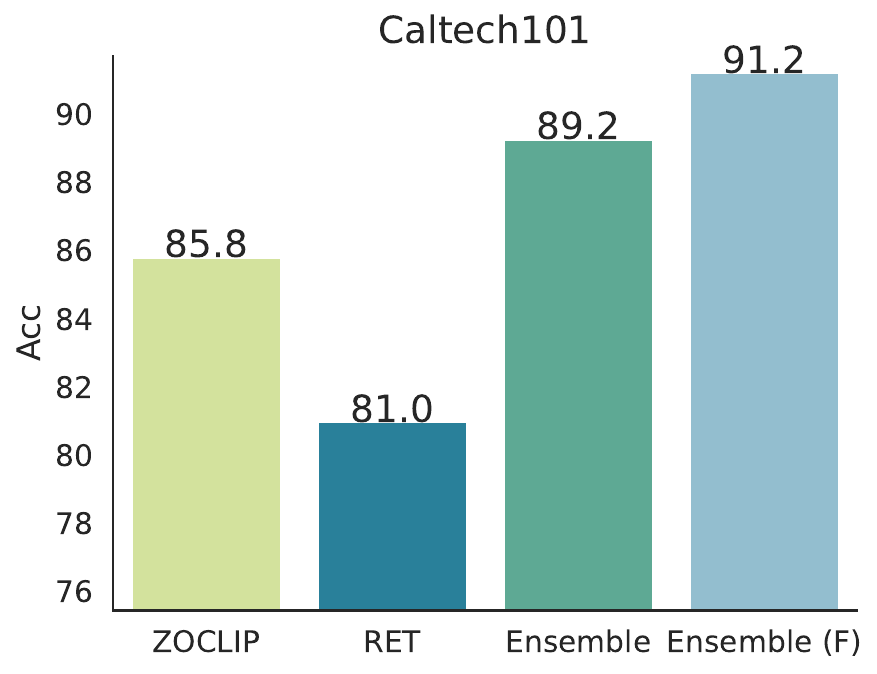}
\end{subfigure}
\begin{subfigure}[b]{0.24\textwidth}
    \includegraphics[width=\textwidth]{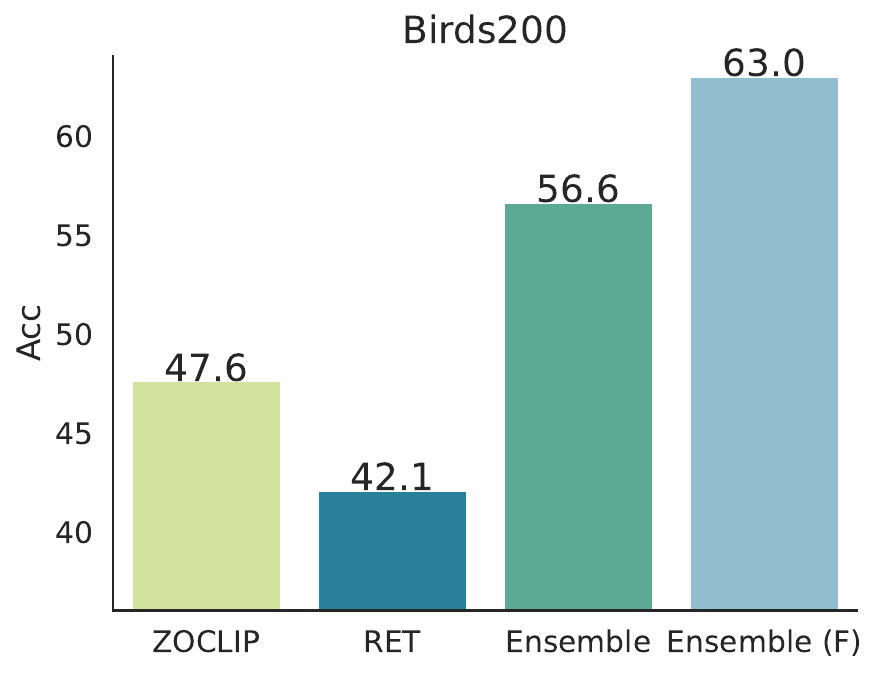}
\end{subfigure}
\begin{subfigure}[b]{0.24\textwidth}
    \includegraphics[width=\textwidth]{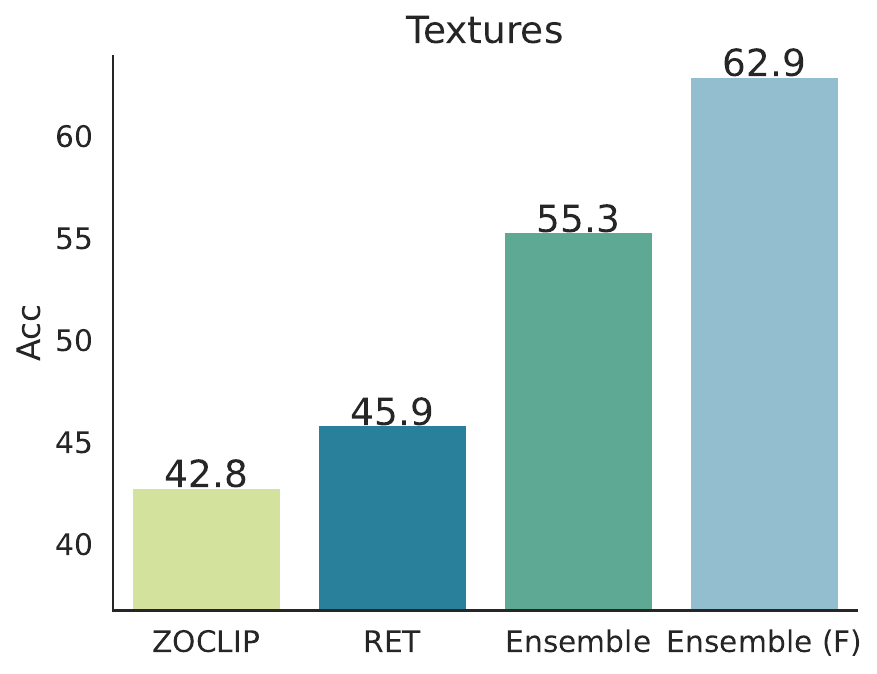}
\end{subfigure}


\begin{subfigure}[b]{0.24\textwidth}
    \includegraphics[width=\textwidth]{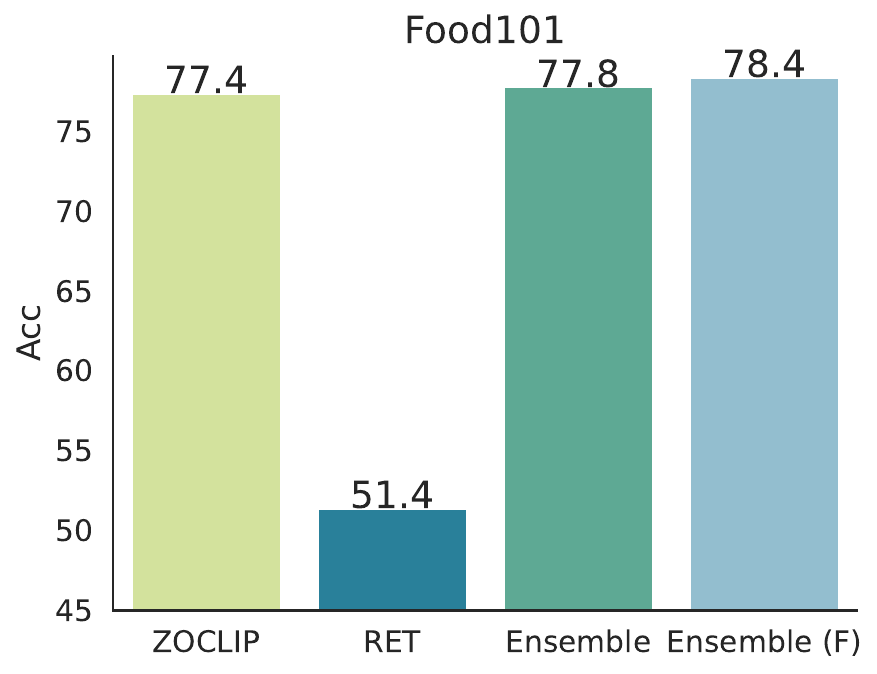}
\end{subfigure}
\begin{subfigure}[b]{0.24\textwidth}
    \includegraphics[width=\textwidth]{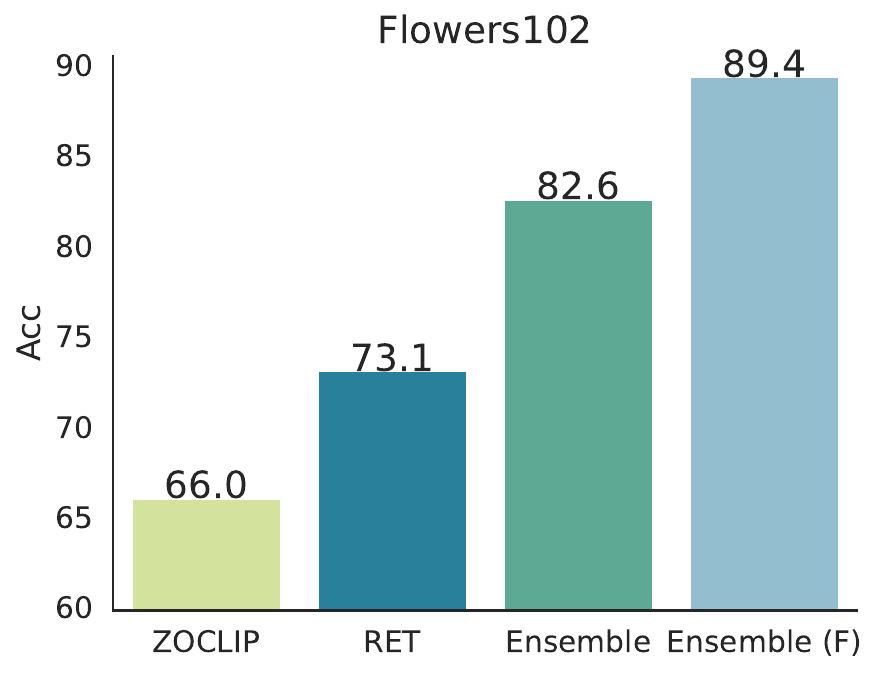}
\end{subfigure}
\begin{subfigure}[b]{0.24\textwidth}
    \includegraphics[width=\textwidth]{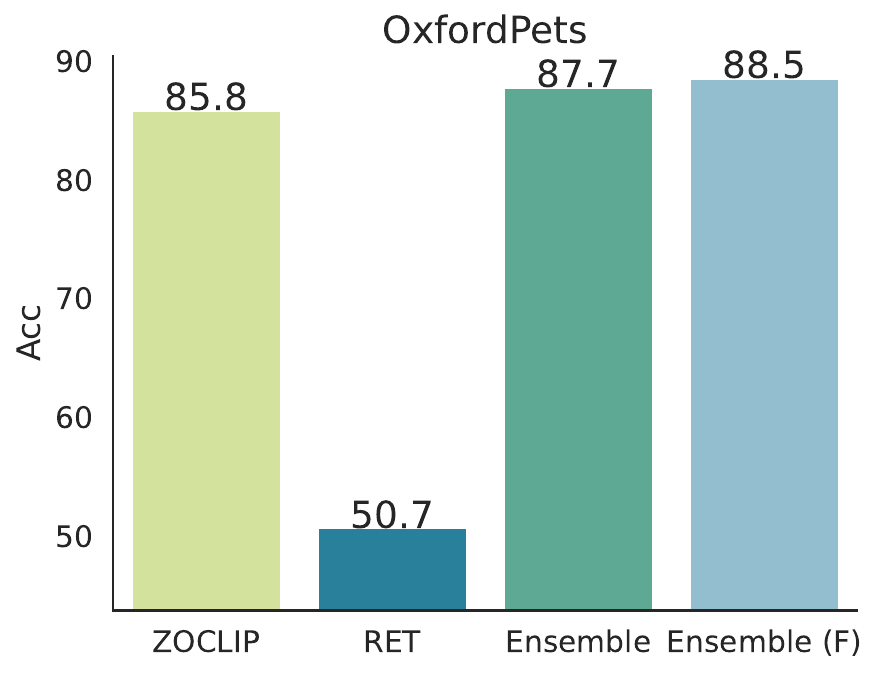}
\end{subfigure}
\begin{subfigure}[b]{0.24\textwidth}
    \includegraphics[width=\textwidth]{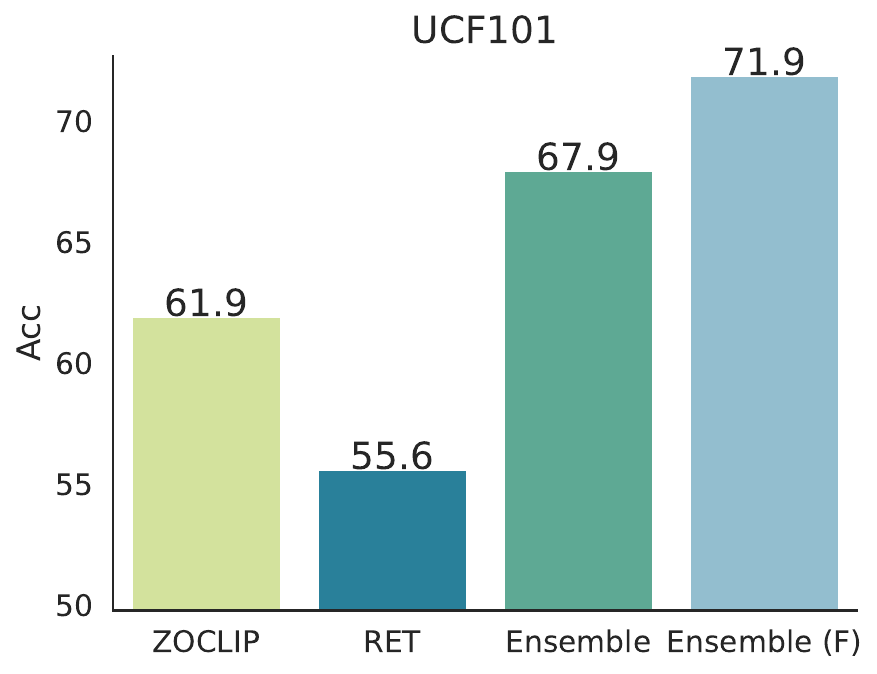}
\end{subfigure}
\vspace{-0.2cm}
\caption{Importance of ensemble for I2I retrieval. Ensemble corresponds to the default logit ensemble: $f^{\text{EN}} = \alpha f^{\text{ZOC}}  + \gamma f^{\text{RET}}$ with $\alpha,\gamma \in(0,1)$. RET denotes only using $f^{\text{RET}}$ ($\alpha=0,\gamma=1$) and ZOCLIP denotes only using $f^{\text{ZOC}}$ ($\alpha=1,\gamma=0$). 
By ensembling the prediction with retrieved samples ($K=16$), the performance improvement over zero-shot prediction is significant for most datasets.}
\label{fig:ret}
\end{figure*}

\paragraph{A closer look at retrieved samples.} To better understand the effects of retrieval, we examine the samples retrieved by T2I and I2I respectively. The results are shown in Figure~\ref{fig:error_sample}. While T2I retrieval often results in a diverse collection of images corresponding to the class semantics, we find that such diversity may not always be desirable for target task adaptation.  For example, when using the query \texttt{a photo of a cellphone}, we retrieve images with a broad range of cellphone types. However, the downstream dataset contains cellphones typical in the 2000s with physical keypads. The same phenomenon widely exists in the suite of datasets commonly used in the literature (see Appendix~\ref{sec:more_sample} for more extensive examples) As a result, T2I retrieval can lead to undesirable performance due to semantic ambiguity. In contrast, I2I retrieval mitigates such ambiguity. For example, when using an image of a cellphone with smaller screens and physical keypads, one can retrieve images of older models of cellphones with similar layouts (middle row).

\subsection{How Do Retrieved Samples Help Adaptation?}
\label{sec:ensemble}

\paragraph{Ensemble with zero-shot prediction is the key.} 
We show that ensembling the zero-shot prediction together with I2I-retrieved samples is the key to improved adaptation performance. 
The results are shown in Figure~\ref{fig:ret}, where ensemble denotes using $f^{\text{EN}} = \alpha f^{\text{ZOC}}  + \gamma f^{\text{RET}}$ with $\alpha,\gamma \in(0,1)$, RET denotes only using $f^{\text{RET}}$ ($\alpha=0, \gamma=1$), and ZOCLIP means only using $f^{\text{ZOC}}$ ($\alpha=1, \gamma =0$). 
This interesting phenomenon highlights the importance of logit ensembling for adapting vision-language models to downstream tasks. The benefits can also be seen by examining the class-wise performance of RET and Ensemble (see Figure~\ref{fig:classwise} in Appendix~\ref{sec:classwise}). Similar trends also hold for training-based adaptation, denoted as Ensemble (F), where we finetune the cache features as in ~\citet{zhang2022tip}. Next, we provide further theoretical explanations (Theorem~\ref{ensemble}).

\section{Theoretical Understanding}
\label{sec:theory}
We now provide theory to support our empirical observations and formally understand retrieval-augmented task adaptation.  As an overview, 
\textbf{Theorem~\ref{i2i}} shows why I2I retrieval is superior to T2I retrieval.
We further prove that logit ensemble is the key for retrieval-augmented adaptation in \textbf{Theorem~\ref{ensemble}}. These two theorems justify our empirical results in Section~\ref{sec:fine-grain}. Full proof is in Appendix~\ref{sec:theory_appendix}.  

\subsection{Problem Setup}  \label{sec:theory_setup}
Given a downstream task with $C$ classes, let $[C] := \{1,2,\cdots, C\}$. $\T = [\t_1, \dots, \t_C] \in \R^{\nd \times C}$ denotes the text embedding matrix for all classes, where $\t_c := \mathcal{T}(t_c) \in \mathbb{R}^d$ and $t_c$ is a generic textual description of class $c$. Recall that $\K = [\*k_{1,1}, \*k_{1,2} \dots, \*k_{C,K}] \in \R^{\nd \times CK}$ denotes the embedding matrix for retrieved images, where $\*k_{c,i} := \mathcal{I}(\mathbf{x}_{c,i}) \in \R^\nd$. For notational simplicity, we assume text and image features are $\ell_2$ normalized. Let $\bK =  {\K \V^\top \over K} \in \R^{\nd \times C}$ contain the average retrieved feature for each class. $\V \in \R^{C \times CK}$ is a sparse matrix containing the one-hot labels for retrieved samples with entries $\V_{i, j} = \mathbbm{1}\{i=\tilde{j}\}$ for $i\in[C], j \in [CK]$, where $\tilde{j}:=\left\lceil \frac{j}{K} \right\rceil$~\cite{zhang2022tip}. For example, when $K=2, C=3$, we have:
$$
\V =\begin{bmatrix}
1 & 1 & 0 & 0 & 0 & 0 \\
0 & 0 & 1 & 1 & 0 & 0 \\
0 & 0 & 0 & 0 & 1 & 1 \\
\end{bmatrix}.
$$

At inference time, let $(\x, y) \sim \cD_T$ be a test sample from the target distribution $\cD_T$ with label $y \in [C]$ and its visual feature $\z := \mathcal{I}(\x)$. The final logit for the test sample can be represented as a weighted sum (ensemble) of logits from the zero-shot CLIP and the feature cache from retrieval:
\begin{align*}
    f(\x) = (\alpha \T +  \gamma \bK)^\top \z,
\end{align*}
where $0\leq\alpha,\gamma\leq 1$.

Given a loss function $\ell$ (e.g., cross-entropy), the risk on the downstream distribution is $\cL(f) := \E_{(\x, y) \sim \cD_T} [\ell(f(\x), y)]$. To simplify notations, we denote the risk as $\cR(\Q) := \E \left[\ell(\Q^\top \z, y) \right]$ for some $\Q\in\mathbb{R}^{d\times C}$. For example, the risk of logit ensemble is $\cR(\alpha \T + \gamma \bK)$.

\noindent\textbf{Modality gap and retrieval distribution shift.} To understand the impact of retrieval, we characterize the distributional shift between the retrieved data and downstream data in the feature space. We define $\bs_c := \E_{(\x, y) \sim \cD_{T}}[\mathcal{I}(\x)| y = c]$ as the image representation of class $c\in[C]$ based on the downstream distribution. Let $\bS := [\bs_1, \dots, \bs_C]$. 
We define the distributional shift between the retrieved data and target data for T2I and I2I retrieval as $\xi_{c}^{\text{T2I}}$ and $\xi_{c}^{\text{I2I}}$ for class $c$. Let $\xi_{\t} := \max_{c\in [C]} \xi_{c}^{\text{T2I}}$ and $\xi_{\s} := \max_{c\in [C]} \xi_{c}^{\text{I2I}}$ (\Cref{def:ret_shift}). We can obtain an upper bound for $\xi_{\s}$ and a lower bound for $\xi_{\t}$ by \Cref{lem:uni}.

\subsection{Main Results}
Under realistic assumptions of T2I and I2I retrieval on the pre-trained feature space, we present two key results below. The detailed versions with full proof are in Appendix~\ref{sec:theory_appendix}.

\begin{figure*}[ht]
\centering
\begin{subfigure}[b]{0.24\textwidth}
    \includegraphics[width=\textwidth]{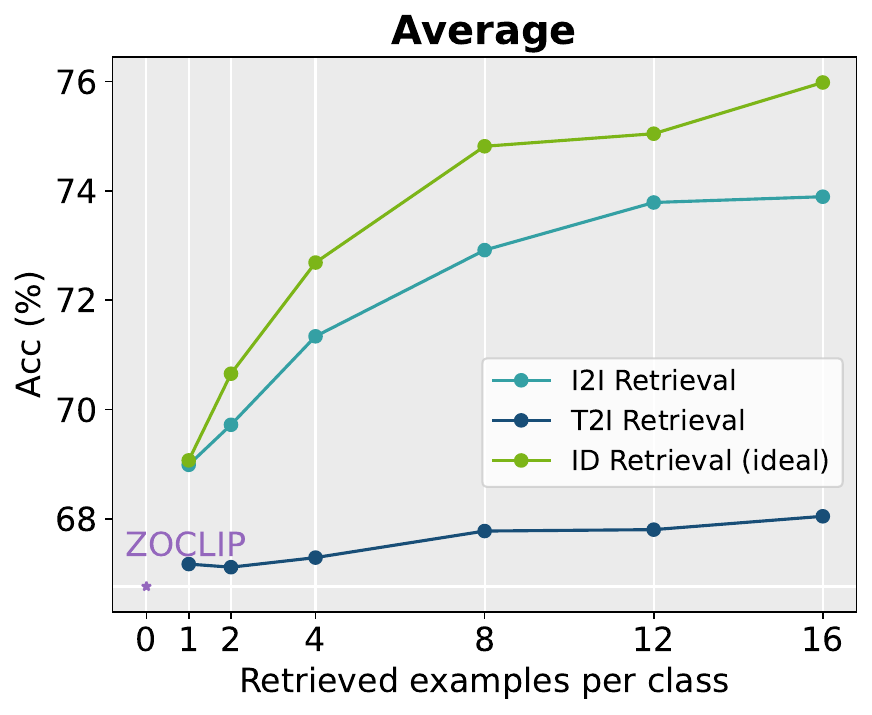}
    \caption{RN50} 
\end{subfigure}
\begin{subfigure}[b]{0.24\textwidth}
    \includegraphics[width=\textwidth]{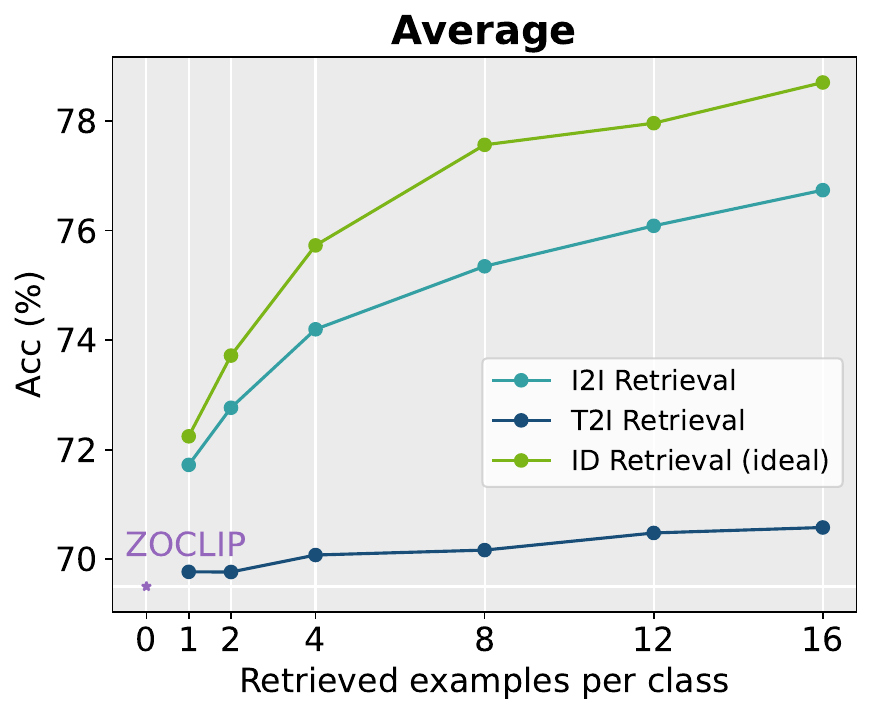}
    \caption{ViT-B/32} 
\end{subfigure}
\begin{subfigure}[b]{0.24\textwidth}
    \includegraphics[width=\textwidth]{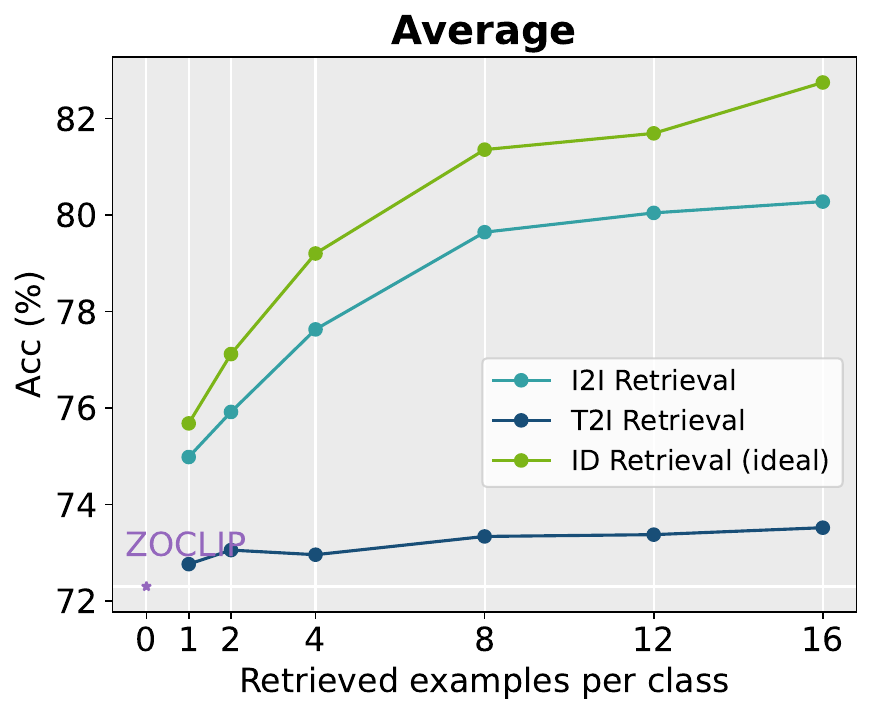}
    \caption{ViT-B/16} 
\end{subfigure}
\begin{subfigure}[b]{0.24\textwidth}
    \includegraphics[width=\textwidth]{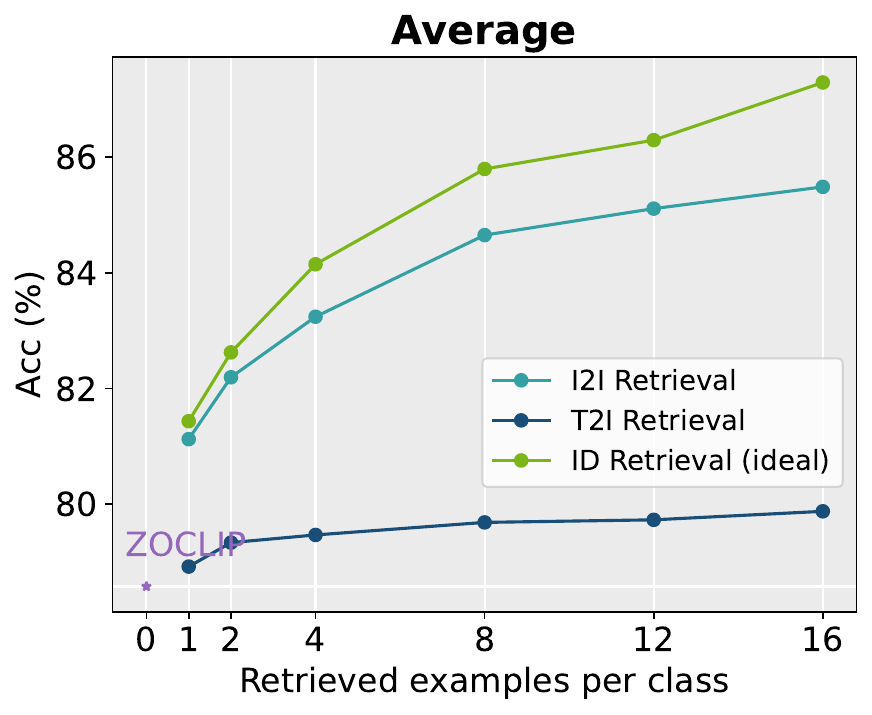}
    \caption{ViT-L/14} 
\end{subfigure}
\vspace{-0.2cm}
\caption{Impact of architecture. We report the average performance (over all datasets) for I2I retrieval and T2I retrieval under different CLIP backbones and observe consistent trends. Results for individual datasets can be seen in Appendix~\ref{sec:arch_ind}.}
\label{fig:arch}
\end{figure*}

\begin{theorem}[Benefit of uni-modal retrieval]\label{i2i}
With probability at least $1-\delta$, the following upper bound of the ensemble risk holds:
\begin{align*}
    & \cR(\alpha \T + \gamma \bK) - \cR(\bS)  
    \\ \le  & L \Bigg(\alpha \underbrace{\|(\T - \bS)^\top \z\|_2 } _{\text{modality gap}} + \underbrace{\gamma\kappa\sqrt{{8C  \over K} \log{C\over \delta} }}_{\text{retrieval sample complexity}} + \underbrace{\gamma\sqrt{2C \xi}}_{\text{retrieval shift}} \Bigg),
\end{align*}
where $L \le \sqrt{\exp(2) + 1}$, $\kappa$ characterizes the inner-class feature concentration (\Cref{def:concentration}), and $\xi$ is either $\xi_{\s}$ for I2I retrieval or $\xi_{\t}$ for T2I retrieval.
\end{theorem}  

\textbf{Interpretations:} The above upper bound consists of three terms: the gap between the textual and visual modality, the sample complexity of retrieved features which decreases as we increase $K$, and a term related to the distributional shift induced by the retrieval method. By \Cref{lem:uni}, we can further show that I2I provably outperforms T2I retrieval due to a smaller $\xi$.

Further, to understand the benefit of logit ensemble, we define the following three events: 
\begin{align*}
    E_1 := & \{ (\x,y)\sim\mathcal{D}_T: y \neq \argmax_{c\in[C]} \t_c^\top \z \text{ and } y \neq \argmax_{c\in[C]} \bk_c^\top \z\}\\
    E_2 := & \{(\x,y)\sim\mathcal{D}_T: y = \argmax_{c\in[C]} \t_c^\top \z \text{ and } y \neq \argmax_{c\in[C]} \bk_c^\top \z\}\\
    E_3 := & \{(\x,y)\sim\mathcal{D}_T: y \neq \argmax_{c\in[C]} \t_c^\top \z \text{ and } y = \argmax_{c\in[C]} \bk_c^\top \z\}
\end{align*}
Here $E_1$ indicates that both $f^{\text{ZOC}}$ and $f^{\text{RET}}$ incorrectly classify the test sample, while $E_2$ and $E_3$ denote the event where only one of them makes a correct prediction. 
We can see that $\cR_{0-1}(f^{\text{ZOC}}) = \Pr(E_1) + \Pr(E_3)$ and $\cR_{0-1}(f^{\text{RET}}) = \Pr(E_1) + \Pr(E_2)$.
\begin{theorem}[Benefit of logit ensemble]\label{ensemble} Under realistic assumptions for I2I retrieval, when $\alpha=\gamma = {1\over 2}$, we can upper bound the 0-1 risk of logit ensemble:
\begin{align*}
    \cR_{0-1}\left(f\right) \le  \Pr\left(E_1\right)  + 
    C_1(\Pr(E_2)+\Pr(E_3)) + \rho_c 
\end{align*}
 where $C_1 := \rho_d\max\{6\kappa-\nu,2\kappa  + \tau\}$ is a term related to modality gap, inner-class feature concentration, and inter-class separation. $\rho_c$ characterizes the ratio of outliers. See Appendix~\ref{sec:theory_appendix} for detailed definitions of $\kappa, \tau,\nu,\rho_c$, and $\rho_d$.
\end{theorem}

\textbf{Interpretations:}  The above theorem characterizes the 0-1 risk upper bound by the modality gap and key properties of retrieved and target distributions. Moreover, logit ensemble utilizes knowledge encoded in different modalities to benefit each other. We can further show that under some conditions (detailed in Appendix~\ref{sec:theory_appendix}), logit ensemble leads to a lower 0-1 risk (\emph{i.e.,} higher accuracy) than the zero-shot model.

 \section{Discussion of Design Choices}
\label{sec:discuss} 
In this section, we discuss the impact of other design choices for retrieval-augmented adaptation.

\paragraph{Impact of model architecture.} We conduct an ablation study on the impact of model architectures. We consider CLIP with ResNet (RN50) and ViT~\cite{dosovitskiy2021an} backbones (CLIP-B/32, CLIP-B/16, CLIP-L/14), where the vision encoder is based on ViT-B/32 and ViT-L/14, respectively. The results are shown in Figure~\ref{fig:arch}. We observe that a similar trend holds for CLIP with various backbones, where I2I retrieval consistently outperforms T2I retrieval. In particular, larger backbones such as CLIP-L/14 lead to overall superior performance compared to smaller backbones across the number of retrieved samples per class.

\paragraph{Impact of the number of seed images.} To investigate the impact of seed images on I2I retrieval, we adjust the number of seed images per class from 2 to 8. The results are shown in Table~\ref{tab:seed} based on Textures ($K=16$). We can see that increasing the number of seed images improves the adaptation performance because it is less prone to overfitting to limited retrieved samples. Similar trends also hold for other datasets in the test suite.

\begin{table}[ht]
\centering
\resizebox{0.45\textwidth}{!}{
\begin{tabular}{lcccc}
\toprule
\multirow{2}{*}{Seed \#} & \multicolumn{4}{c}{\textbf{Method}}               \\
                                 & ZOCLIP & RET   & Ensemble & Ensemble (F) \\
                                 \midrule
2                                & 42.79  & 38.48 & 51.77    & 57.98        \\
4                                & 42.79  & 44.09 & 52.96    & 58.57        \\
8                                & 42.79  & 45.86 & 55.32    & 62.94 \\
\bottomrule      
\end{tabular}
}
\vspace{-0.1cm}
\caption{The impact of the number of seed images (per class) for I2I retrieval. Results are based on RN50 backbone with $K=16$.}
\label{tab:seed}
\end{table}

\paragraph{Adaptation with a mixture of ID and retrieved samples.} Previously, we have considered only using retrieved samples in the feature cache to better understand the effects of retrieval. When we have access to the few-shot (ID) training set, another practical scenario is to use a mixture of retrieved and ID samples. The results are shown in Figure~\ref{fig:mixture}. We report the average performance (over 7 datasets) for I2I retrieval ($K=16$). EN denotes logit ensemble with only retrieved samples. MIX denotes logit ensemble with a mixture of ID samples and retrieved samples. EN (F) and MIX (F) stand for the finetuned variants. The mixture ratio is 1:1. We observe that mixing ID and retrieved samples further leads to improved performance compared to only using few-shot ID samples. Our observations are consistent with prior works~\cite{udandarao2023sus,zhang2023prompt} under different logit ensemble schemes, which highlight the potential of retrieval-augmented few-shot adaptation. 

\begin{figure}[ht]
\centering
    \includegraphics[width=0.3\textwidth]{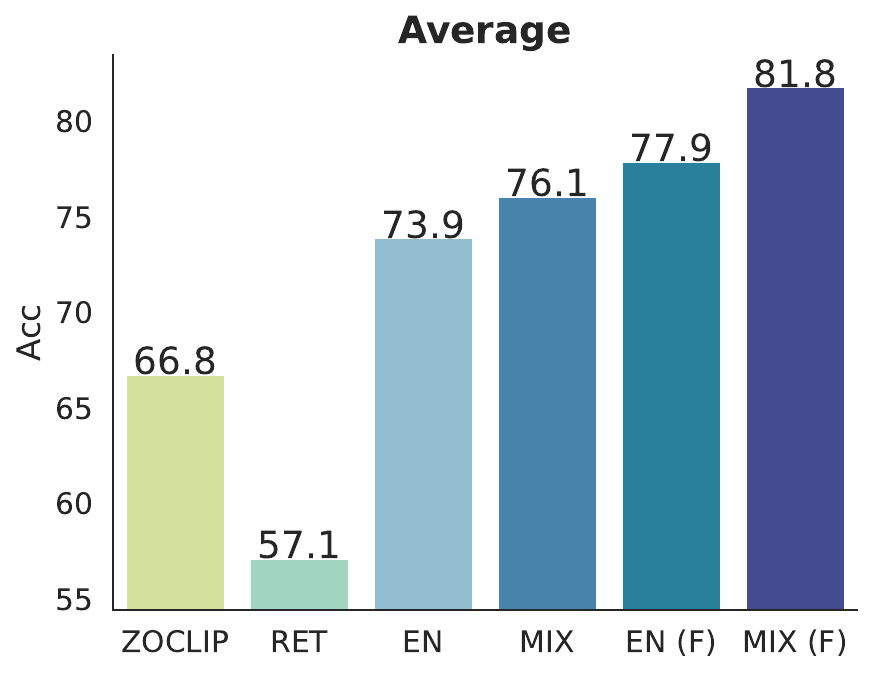}
    \vspace{-0.2cm}
\caption{Impact of Mixture of retrieved samples with few-shot ID data. We report the average performance (over all datasets) for I2I retrieval ($K=16$). EN denotes logit ensemble with only retrieved samples. MIX denotes logit ensemble with a mixture of ID samples and retrieved samples. The mixture ratio is 1:1.}
\vspace{-0.3cm}
\label{fig:mixture}
\end{figure}

\paragraph{Adaptation with finetuned feature cache.} For completeness, we discuss learning-based adaptation by setting the visual features in the cache $\mathbf{K}$ as learnable parameters after initializing from the pre-trained CLIP model. We denote the variant as \texttt{Ensemble(F)}, where F stands for fine-tuning. We follow the hyperparameter tuning scheme in ~\citet{zhang2022tip} and show the results (averaged across all datasets) in Figure~\ref{fig:finetune_avg}. We can see that a similar trend holds for training-based adaptation, where I2I retrieval significantly outperforms zero-shot CLIP and T2I retrieval. In the low-shot setting ($K=1$ or $2$), the performance is close to the ideal case (ID retrieval). Full results for individual datasets can be seen in Appendix~\ref{sec:train_adapt}.

\begin{figure}[ht]
\centering
    \includegraphics[width=0.3\textwidth]{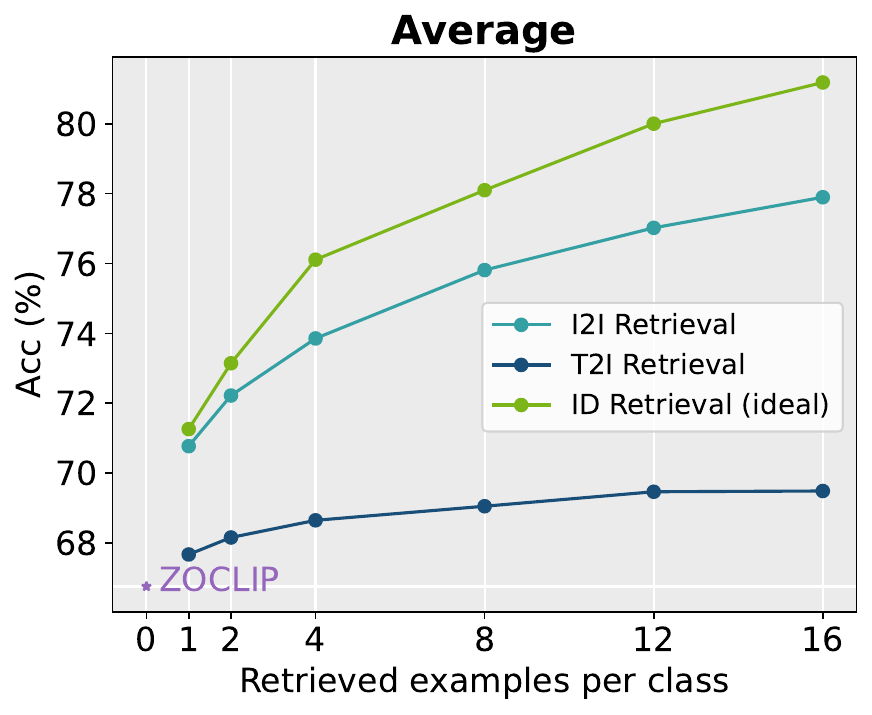}
    \vspace{-0.2cm}
\caption{Adaptation with finetuned feature cache.  We observe a similar trend as training-free adaptation. }
\vspace{-0.3cm}
\label{fig:finetune_avg}
\end{figure}
Due to space constraints, we provide additional ablation studies in the Appendix.

\section{Related Works} 
\paragraph{Few-shot task adaptation for vision-language models.} Recent years have witnessed the popularity of contrastive language-image pre-training (CLIP)~\cite{radford2021learning,jia2021scaling,yang2022unified,li2022supervision,mu2022slip,yu2022coca,zhai2022lit,sun2023eva, zhai2023sigmoid, xu2024demystifying}, etc. While CLIP-like models learn aligned multi-modal features, they often struggle on fine-trained datasets with categories not adequately represented during pre-training, which makes adaptation necessary. Recent works propose various promising solutions for adapting the vision-language model in the low-data (few-shot) scheme such as tuning textual prompts~\cite{zhou2022cocoop,zhou2022coop}, visual prompts~\cite{bahng2022visual, chen2023vpt}, multi-modal prompts~\cite{khattak2023maple}. \citet{zhang2022neural} use neural architecture search to optimize prompt modules. \citet{lu2022prompt} optimize prompts by learning prompt distributions. Alternatively,~\citet{yu2023task} tune an additional task residual layer. Another line of work utilizes adaptor~\cite{zhang2022tip, gao2023clip, zhang2023prompt, udandarao2023sus} by maintaining a memory cache that stores the features of few-shot data. \citet{zhang2022tip} uses an additive logit ensemble with a feature cache from the target training set. In contrast, we focus on the impact of retrieval and build the cache with retrieved samples, rather than the downstream dataset.

\paragraph{Knowledge-augmented adaptation for CLIP.} A natural idea for task adaptation is to utilize external knowledge sources by retrieval or synthesis. Sampling from external datasets has shown promising performance in adapting vision models to fine-grained datasets~\cite{liu2022improved, kim2023coreset}. For CLIP-based adaptation, existing methods can be categorized into two regimes, based on the amount of external data utilized. In the high-data regime, \citet{liu2023learning} demonstrates promising zero-shot performance by first constructing a large-scale dataset (10M) containing relevant samples retrieved from web-scale databases and then fine-tuning CLIP models on the retrieved dataset. \citet{xie2023ra} propose a Retrieval Augmented Module to augment CLIP pre-training on 1.6M retrieved samples. Recently, \citet{iscen2023retrieval} advocated uni-modal search but cross-modal fusion for CLIP adaptation, where the fusion model is trained on 10M samples. \citet{long2022retrieval} demonstrate the promise of retrieval for long-tail visual recognition tasks.
In the low-data regime, recent works also enhance the retrieval augmentation pipeline with synthetic samples from pre-trained generative models~\cite{udandarao2023sus, zhang2023prompt}. Beyond augmenting the visual modality, \citet{shen2022k} leverage external text knowledge sources such as WordNet~\cite{fellbaum1998wordnet} and Wiktionary~\cite{meyer2012wiktionary} to augment captions with class-specific descriptions, while \citet{pratt2023does} perform augmentation by querying large language models. \citet{el2023learning} use the language guidance to find similar visual nearest neighbors. \citet{li2022elevater} establish a benchmark for evaluating the transfer learning performance of language-augmented visual models. In this work, we adopt a reflective perspective and provide a systematic study to understand retrieval-augmented adaptation in the low-data regime and establish new theoretical insights.

\paragraph{Theoretical understanding of multi-modal learning.}
A few works provide theoretical explanations for multi-modal learning~\cite{zadeh2020foundations, huang2021makes, furst2022cloob, chen2023understanding}. For CLIP models, \citet{liang2022mind} demonstrate and provide a systematic analysis of the modality gap between the features of two modalities.
\citet{nakada2023understanding} establish the connection between CLIP and singular value decomposition (SVD) under linear representations. \citet{chen2023understanding} develop a theoretical framework to understand the zero-shot transfer mechanism of CLIP. Different from prior works, we focus on the theoretical understanding of retrieval-augmented task adaptation. 

\section{Conclusion} 
In this work, we present a timely and systematic investigation for retrieval-augmented adaptation of vision-language models in the low-data regime. Our work offers a finer-grained empirical study, unveiling insights into the impact of cross-modal and uni-modal retrieval. In addition, we highlight the necessity of logit ensemble. We also develop a novel theoretical framework that supports our empirical findings and provides a deeper understanding of retrieval-augmented adaptation. Additionally, our comprehensive ablation study explores various design choices in the retrieval augmentation pipeline. We hope our work will serve as a springboard for future research on algorithm design and theoretical understanding for effective adaptation of vision-language models.

\section*{Impact Statements}
The main purpose of this work is to provide a systematic investigation of existing approaches with theoretical understanding. The work can help guide the development of effective and reliable algorithms for retrieval-augmented adaptation of vision-language models. We conducted a thorough manual review to ensure that the retrieved samples do not contain illegal or inappropriate content, and we foresee no immediate negative ethical impact.

\section*{Acknowledgement}
We would like to thank Zhenmei Shi for in-depth discussions on theoretical analysis. We gratefully acknowledge ICML anonymous reviewers for their helpful feedback and the support from the AFOSR
Young Investigator Program under award number FA9550-23-1-0184, National Science Foundation
(NSF) Award No. IIS-2237037 \& IIS-2331669, Office of Naval Research under grant number
N00014-23-1-2643, Philanthropic Fund from SFF, and faculty research awards/gifts from Google
and Meta. 

\bibliography{ref}
\bibliographystyle{icml2024}

\newpage
\appendix
\onecolumn

\begin{center}
	\textbf{\LARGE Appendix }
\end{center}

\section{Experimental Details}
\label{sec:add_detail}

\noindent\textbf{Hardware and software.} We run all experiments on NVIDIA GeForce RTX-A6000 GPU.  To retrieve samples from the LAION5B database, we build a semantics-based retrieval system with \texttt{clip-retrieval} (\url{https://github.com/rom1504/clip-retrieval}) for fast T2I and I2I retrieval based on textual and visual embeddings of pre-trained CLIP. Our implementation is based on PyTorch 1.12.

\noindent\textbf{Retrieval dataset.} We adopt LAION5B as the database for retrieval for three main reasons: (1) Scale: in contrast to prior works that use smaller-scale datasets such as WebVision~\cite{li2017webvision}, Conceptual Captions~\cite{sharma2018conceptual}, and ImageNet-21k~\cite{ridnik2021imagenet}, LAION5B is a web-scale open-source dataset that contains 5,85 billion CLIP-filtered image-text pairs covering a wide range of concepts in the real world. The diverse concept coverage makes it a reliable source for retrieval~\cite{udandarao2023sus}. (2) Multi-modal retrieval: one major advantage of LAION is that it computes the textual and visual embeddings of the text-image pairs based on pre-trained CLIP. This provides the foundation for us to conduct a systematic study on both T2I and I2I retrieval. (3) Retrieval efficiency: the development of distributed inference tools such as \texttt{clip-retrieval} enable fast index building and efficient retrieval from LAION5B based on approximate KNN search. Such community support for LAION5B makes retrieval more practical compared to alternatives.

\noindent\textbf{Prompts for T2I retrieval.} In this work, we use dataset-specific prompts in T2I retrieval to mitigate semantic ambiguity. For example, for Bird200~\citep{WahCUB_200_2011}, the prompt for T2I retrieval is \texttt{A photo of a <CLS>, a type of bird}. The prompts for other datasets can be seen in Table~\ref{tab:t2i_prompt}. In a recent work~\citep{udandarao2023sus}, language model-based prompts are used for retrieval. For instance, the prompt for the class \texttt{baklava} becomes \texttt{baklava is a rich, sweet pastry made with layers of filo dough, nuts, and syrup.} We found that using knowledge-augmented prompts improved the performance of T2I retrieval. However, the performance gain from these prompts was consistently less significant than that observed with I2I retrieval. For example, with an 8-shot setting, the performance (averaged over all datasets) with knowledge-augmented prompts is summarized in Table~\ref{tab:kg_prompt}:

\begin{table}[hbt]
\centering
\begin{tabular}{lcccc}
\toprule
\textbf{Method}  & ZOC   & T2I (original) & T2I (knowledge-augmented) & I2I   \\
\midrule
\textbf{AVG ACC} & 66.76 & 67.77          & 68.86                     & 72.91 \\
\bottomrule
\end{tabular}
\caption{The impact of knowledge-augmented prompts for T2I retrieval. }
\label{tab:kg_prompt}
\end{table}

The results are based on the default setting with RN50 as the vision backbone. The same two key observations hold under alternative prompt strategies: (1) I2I retrieval outperforms T2I retrieval; (2) logit ensemble is essential for superior retrieval performance. By examining the retrieved samples, we identified issues similar to those depicted in Figure~\ref{fig:error_sample} when using knowledge-augmented prompts, particularly when the target class contains characteristics not captured by the class names and their descriptions.

\begin{table}[hbt]
\centering
\begin{tabular}{lc}
\toprule
\textbf{Dataset} & \textbf{Prompt}\\
\midrule
Caltech101~\citep{FeiFei2004LearningGV} & \texttt{A photo of a <CLS>}\\
Birds200~\citep{WahCUB_200_2011} & \texttt{A photo of a <CLS>, a type of bird}\\
Food101~\citep{bossard2014food} &  \texttt{A photo of <CLS>, a type of food}\\
OxfordPets~\citep{parkhi12a} &  \texttt{A photo of a <CLS> pet}\\
Flowers102~\cite{nilsback2008automated} &  \texttt{A photo of a <CLS> flower}\\
Textures~\cite{cimpoi2014describing} & \texttt{A photo of <CLS> texture}\\
UCF101~\cite{soomro2012ucf101} & \texttt{A photo of <CLS> in action}\\
\bottomrule
\end{tabular}
\caption{Default prompts for T2I retrieval. In this work, we use dataset-specific prompts to mitigate semantic ambiguity. }
\label{tab:t2i_prompt}
\end{table}

\noindent\textbf{Fine-tuning details.} As our work focuses on the impact of retrieval, we adopt the fine-tuning scheme in ~\citet{zhang2022tip} for training-based adaptation, where we set features in the retrieval cache as learnable. For each target dataset, the train, validation, and test split also follow ~\citep{zhang2022tip}. Specifically, we use AdamW~\cite{adamw} as the optimizer with a cosine scheduler. The initial learning rate is set as $0.001$ and we finetune for $20$ epochs. The hyperparameters such as $\alpha, \omega, \gamma$ are determined based on the validation split of each target dataset.

\section{A Closer Look at Logit Ensemble via Classwise Performance}
\label{sec:classwise}
In Section~\ref{sec:ensemble}, we have shown that logit ensemble is essential to CLIP-based adaptive inference with the few-shot cache obtained by retrieval. In this section, we take a finer-grained view by examining the change of accuracy for each class before and after logit ensemble. For better visualization, we use Textures~\cite{cimpoi2014describing}, a dataset with 47 classes. The results are shown in Figure~\ref{fig:classwise}, where green indicates an increase in accuracy while orange denotes a decrease in accuracy. The result for RET vs. ZOCLIP (\emph{i.e.}, before ensemble) is shown in Figure~\ref{fig:classwisea} and Ensemble vs. ZOCLIP is shown in Figure~\ref{fig:classwiseb}. We can clearly observe that (1) before ensemble, RET is inferior to ZOCLIP for multiple classes such as \texttt{blotchy} and \texttt{freckled}, and \texttt{pleated}, as a result of retrieval ambiguity. (2) Logit ensemble significantly mitigates such issue and results in an overall larger proportion of green bars compared to orange bars, as shown in Figure~\ref{fig:classwiseb}.
\begin{figure*}[!ht]
\centering
\begin{subfigure}[b]{0.95\textwidth}
    \includegraphics[width=\textwidth]{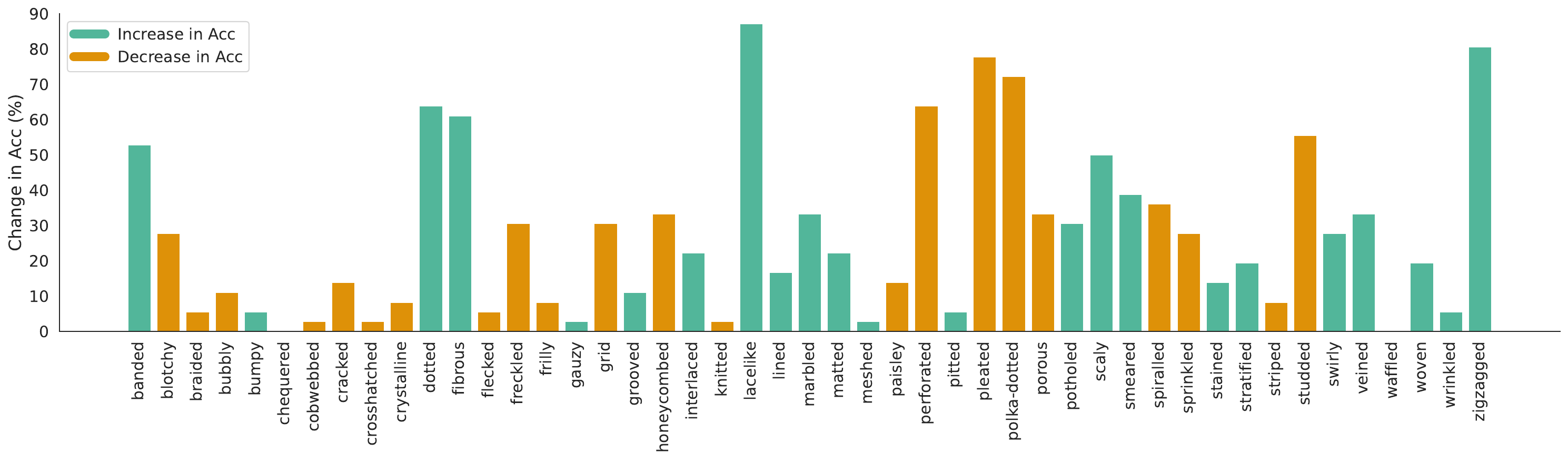} 
    \caption{RET over ZOCLIP (average improvement in accuracy: $3.1\%$)}
    \label{fig:classwisea}
\end{subfigure}
\begin{subfigure}[b]{0.95\textwidth}
    \includegraphics[width=\textwidth]{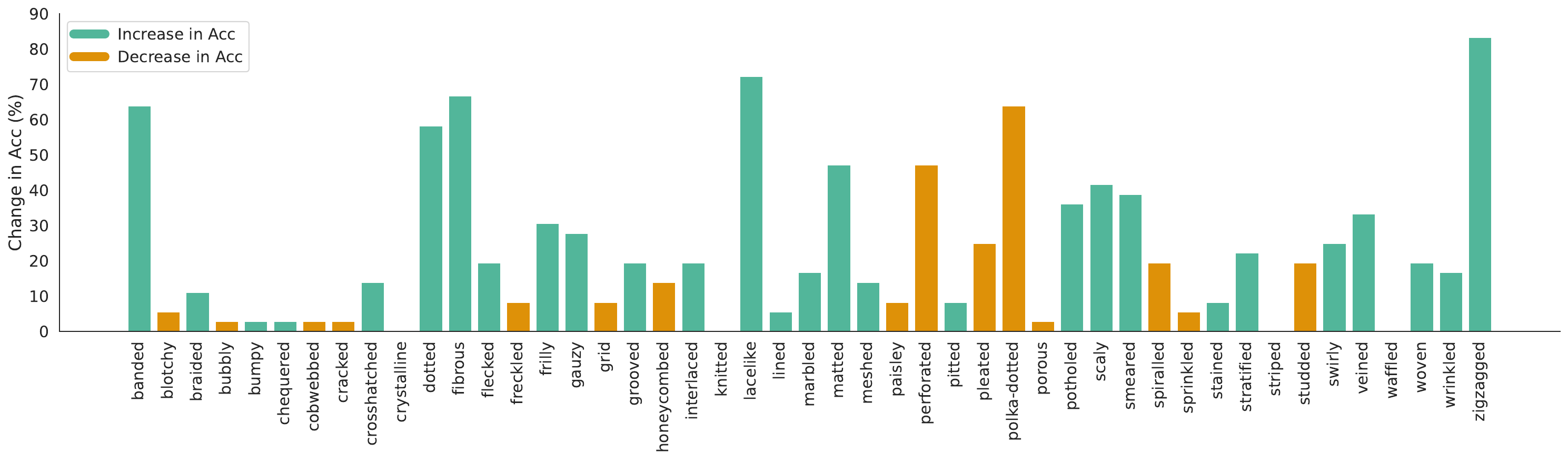}
    \caption{Ensemble over ZOCLIP (average improvement in accuracy: $12.5\%$)} 
    \label{fig:classwiseb}
\end{subfigure}
\caption{Change of classwise accuracy before and after logit ensemble. For better visualization, the results are based on Textures~\cite{cimpoi2014describing}, a dataset with 47 classes. We use I2I retrieval to obtain the few-shot feature cache. We plot the change of accuracy over ZOCLIP for each class before (top row) and after logit ensemble (bottom row). Blue bars indicate an increase in accuracy while orange denotes a decrease in accuracy.   (a) Comparison of RET versus ZOCLIP. On average, RET achieves a $3.1\%$ improvement in accuracy compared to ZOCLIP. (b) Comparison of Ensemble versus ZOCLIP. On average, Ensemble achieves a $12.5\%$ improvement in accuracy compared to ZOCLIP. This further highlights the importance of logit ensemble for retrieval-augmented adaptation.}
\label{fig:classwise}
\end{figure*}

\section{Qualitative Analysis of Retrieved Samples}
\label{sec:more_sample}
In Section~\ref{sec:retrieval_method}, we examined the retrieved samples from I2I and T2I retrieval to identify the main sources of errors. Here, we present additional retrieved samples for diverse datasets. The results are depicted in Figure~\ref{fig:appendix_error_sample}, where we contrast samples from T2I retrieval (top row), I2I retrieval (middle row), and the downstream dataset (bottom row). We have two salient observations: (1) As discussed in Section~\ref{sec:retrieval_method}, T2I retrieval often yields a diverse set of images that match the class semantics. However, this diversity may not always be beneficial for adapting to the target dataset, especially in the few-shot retrieval setting where one is under a limited budget. For example, using the query \texttt{a photo of a lobster}, we may not retrieve images of cooked lobsters that often appear in the target dataset. (2) Since T2I retrieval utilizes the class name in the query, it occasionally retrieves images with text on them, rather than images of the actual object. For instance, we retrieve images that feature the text ``summer tanager'' or ``dandelion'' (as seen in the 4th and 3rd columns of Figures~\ref{fig:appendix_error_sample} and \ref{fig:error_sample}, respectively). This occurs because the cosine similarity between pairs of \texttt{(class name, image of the actual object)} and \texttt{(class name, image with the text <class name>)} is similar, based on pre-trained CLIP models. This highlights a prevalent challenge in web-scale cross-modal retrieval systems, such as LAION5B. Conversely, this type of misalignment is rarely encountered in I2I retrieval. Therefore, samples from T2I retrieval can introduce undesirable inductive biases, resulting in limited performance gains over the zero-shot model.

\begin{figure}[!ht]
\centering
    \includegraphics[width=\textwidth]{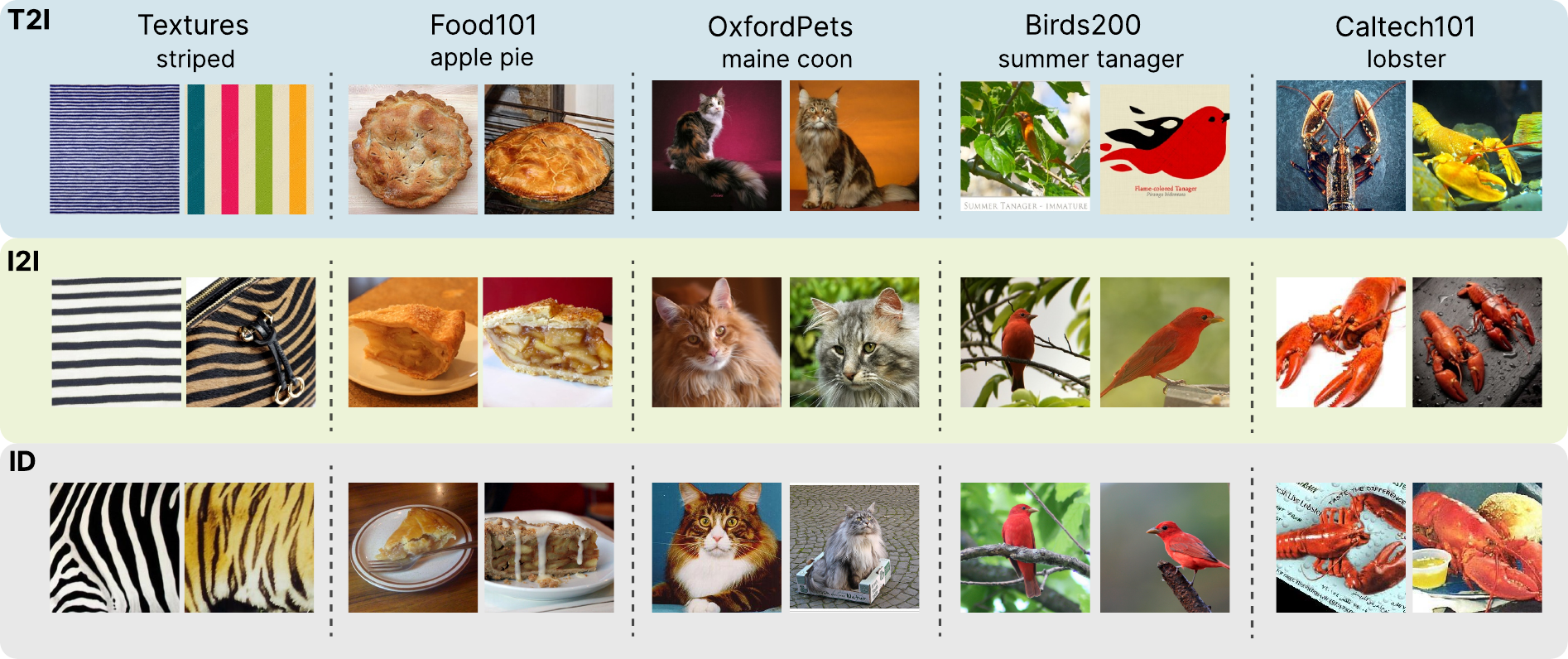}
    \caption{More samples from T2I and I2I retrieval. Top row: the main source of noise for T2I retrieval is semantic ambiguity, as the textual queries (\emph{e.g.}, \texttt{striped texture}) may not accurately describe the images from target distributions (bottom row). Middle row: samples retrieved by I2I matches more closely with ID data. Bottom row: images sampled from the target (ID) distribution.} 
    \label{fig:appendix_error_sample}
\end{figure}

\newpage\null
\section{Theoretical Understanding}  \label{sec:theory_appendix}
In this section, we provide details on the problem setup, introduce relevant definitions and lemmas, and provide the complete proof for our theoretical results discussed in Section~\ref{sec:theory}. Common notations can be seen in Table~\ref{tab:notations}.

\begin{table}[h]
\centering
\begin{tabular}{cl}
\hline
\textbf{Notation} & \textbf{Description} \\
\hline
$[C]$ & The set $\{1, 2, \dots, C\}$ \\
$\mathbbm{1}[\text{condition}]$ & Indicator function, equals 1 if the condition is true, 0 otherwise \\
$\mathcal{T}$ &  $\mathcal{T}: t \rightarrow \mathbb{R}^d$ is the text encoder of CLIP \\
$\mathcal{I}$ &  $\mathcal{I}: \*x \rightarrow \mathbb{R}^d$ is the image encoder of CLIP \\
\hline
\end{tabular}
\caption{Common notations.}
\label{tab:notations}
\end{table}

\subsection{Problem Setup}  \label{def:setup}
We consider a pre-trained CLIP model~\citep{radford2021learning} with one text encoder $\mathcal{T}: t \rightarrow \mathbb{R}^d$ and one image encoder $\mathcal{I}: \*x \rightarrow \mathbb{R}^d$.
We use $\T = [\t_1, \dots, \t_C] \in \R^{\nd \times C}$ to denote the text embedding matrix for all classes, where $\t_c := \mathcal{T}(\t_c) \in \mathbb{R}^d$ and $t_c$ is a generic textual description of class $c$ such as ``\texttt{a photo of <CLASS c>}''. For theoretical analysis, we consider training-free adaptation based on retrieved samples. We use the terms ``downstream'' and ``target'' dataset interchangeably which refer to the dataset a pre-trained CLIP model is adapted to. 
 
 \noindent\textbf{Building feature cache by retrieval.}  Given a downstream dataset with $C$ classes: $\mathcal{Y}=\{1, 2,...,C\}$ and a retrieval budget size of $KC$, we can retrieve $K$ samples per class to build a cache of size $KC$. Recall that $\K = [\*k_{1,1}, \*k_{1,2} \dots, \*k_{C,K}] \in \R^{\nd \times CK}$ denotes the embedding matrix for retrieved images, where $\*k_{c,i} := \mathcal{I}(\mathbf{x}_{c,i}) \in \R^\nd$. For notational simplicity, we assume text and image features are $\ell_2$ normalized~\cite{radford2021learning}. In other words, we have $\|\z\|_2 = \|\t_c\|_2 = 1$ for any $\z = \mathcal{I}(\x)$ and $\t_c = \mathcal{T}(t_c)$.
 
 Let $\tilde{\*K} =  {\K \V^\top \over K} = [\tilde{\*k}_{1}, \tilde{\*k}_{2}, \dots, \tilde{\*k}_{C}] \in \R^{\nd \times C}$ contain the average retrieved feature for each class. $\V \in \R^{C \times CK}$ is a sparse matrix containing the one-hot labels for retrieved samples with entries $\V_{i, j} = \mathbbm{1}\{i=\tilde{j}\}$ for $i\in[C], j \in [CK]$, where $\tilde{j}:=\left\lceil \frac{j}{K} \right\rceil$~\cite{zhang2022tip}. For example, when $K=2, C=3$, we have:
$$
\V =\begin{bmatrix}
1 & 1 & 0 & 0 & 0 & 0 \\
0 & 0 & 1 & 1 & 0 & 0 \\
0 & 0 & 0 & 0 & 1 & 1 \\
\end{bmatrix}.
$$
We define $\bK := [\bk_{1}, \bk_{2}, \dots, \bk_{C}]$ as the normalized version where $\bk_{i} = {\tilde{\*k}_i\over \|\tilde{\*k}_i\|_2}$, which will be used in the rest of the section. Note that here the notations are slightly different from Section~\ref{sec:theory_setup} and are more rigorous.

\noindent\textbf{Task adaptation with retrieved cache.} At inference time, let $(\x, y) \sim \cD_T$ be a test sample from the target distribution $\cD_T$ with label $y \in [C]$ and its visual feature $\z := \mathcal{I}(\x)$. In some cases, beyond retrieved samples, one also has access to a cache consisting of few-shot training samples from the target distribution. For theoretical analysis, we consider one-shot and denote the feature cache as $\S := [\s_1, \dots, \s_C] \in \R^{\nd \times C}$. 
The final logit for the test sample can be represented as a weighted sum (ensemble) of logits from the zero-shot CLIP and the feature cache from retrieved and training samples\footnote{For theoretical analysis, we omit the exponential scaling function to better focus on the effects of ensembling.}: 
\begin{align*}
    f(\x) = (\alpha \T + \beta \S + \gamma \bK)^\top \z,
\end{align*}
where $0\leq\alpha,\beta,\gamma\leq 1$. Without loss of generality, we assume $\alpha + \beta + \gamma = 1$.

In particular, the zero-shot logit $f^{\text{ZOC}}(\x) := \T^\top\z $ and the retrieval logit $f^{\text{RET}}(\x) := \bK^\top\z $. In the main paper, we mainly focus on $\beta = 0$ ($\emph{i.e.}$, one only has access to retrieved samples). We denote the corresponding ensemble logit as $f^{\text{EN}}(\x) = (\alpha \T  + \gamma \bK)^\top \z$.

\noindent\textbf{Evaluation metric.} Given a loss function $\ell(\v, y)$ such as the cross-entropy:
\begin{align*}
    \ell(\v, y) = -\log{\exp(\v_y) \over \sum_{i \in [C]} \exp(\v_i) },
\end{align*}
the population risk on the target distribution is:
\begin{align*}
    \cL(f) = & \E_{(\x, y) \sim \cD_T} [\ell(f(\x), y)]. 
\end{align*}

To simplify notations, we denote the risk as $\cR(\Q) := \E \left[\ell(\Q^\top \z, y) \right]$ for some $\Q\in\mathbb{R}^{d\times C}$. For example, the risk of logit ensemble is $\cR(\alpha \T + \gamma \bK)$. We also have the error risk $\cR_{0-1}$ defined as:
\begin{align*}
    \cR_{0-1}(f) = 1-\E_{(\x, y) \sim \cD_T} \left[\mathbbm{1}\{\argmax_{i\in[C]} f(\x)_i = y\}\right].
\end{align*}

\subsection{Definitions and Assumptions} \label{def:def_assump}
Before presenting the main theoretical results, we first introduce the following definitions and assumptions to formalize the retrieval augmented adaptation process based on pre-trained CLIP models. 

For class $i \in [C]$, we define $\tilde{\s}_i := \E_{(\x, y) \sim \cD_{T}}[\mathcal{I}(\x)| y = i]$, which is the image representation of class $i$ based on the downstream distribution and $ \bs_i = {\tilde{\s}_i \over \|\tilde{\s}_i\|_2 }$ the $\ell_2$ normalized version\footnote{For any two non-zero vectors $\v_1, \v_2$ with unit norms, we have $\left\|{\v_1 } - {\v_2 }\right\|_2 = \sqrt{2 - 2\v_1^\top \v_2}$.}. 
Let $\bS := \left[{\bs_1}, {\bs_2},\dots, {\bs_C}\right]$. 

\begin{definition}[Inner-class concentration and inter-class separation]\label{def:concentration}
We define the inter-class feature separation as $\nu: = 1 - \max_{i \neq j} {\bs_i^\top\bs_j} $. We use $\rho_c$ to denote the inner-class feature concentration: $$\rho_c := \max_{i \in [C]} \Pr\left( \left\|\mathcal{I}(\x) - \bs_i\right\|_2 \ge \kappa \middle| y = i \right) $$ 
for some positive constant $\kappa$.
\end{definition}

\begin{definition}
Let  $\bZ = [\bz_1, \dots, \bz_C] \in \R^{\nd \times C}$. We define the optimal representations as
\begin{align*}
     \bZ^* = \argmin_{\bZ \in \R^{\nd \times C}; \forall i \in [C], \|\bz_i\| = 1} \E [\ell(\bZ^{*\top} \z, y)].  
\end{align*}
\end{definition}

\begin{definition}[Modality gap]\label{def:gap}  We define the modality gap between the pre-trained text distribution and the target distribution (in the visual modality) as $\tau := \max_{i \neq j}  (\t_j - \t_i)^\top {\bs_i}$, where $i,j\in [C]$. 
\end{definition}

\begin{definition}[Retrieval distribution shift]\label{def:ret_shift}
We denote the retrieval distribution based on the (text or image) query (denote $\t_c$ or $\s_c$ as $\q_c$) from class $c$ as $\cD_{R|{\q_c}}$. $\tilde{\k}_{\q_c} := \E_{\x \sim \cD_{R|{\q_c}}}[\mathcal{I}(\x)]$ is the average retrieved feature from class $c$. $\bk_{\q_c} := {\tilde{\k}_{\q_c} \over \|\tilde{\k}_{\q_c} \|_2} $ denotes the normalized version. We define the distributional shift between target data and T2I and I2I retrieval data for class $c$ as $\xi_{c}^{\text{T2I}} := 1 - \bk_{\t_c}^\top\bs_c $ and $\xi_{c}^{\text{I2I}} := 1 - \bk_{\s_c}^\top\bs_c $.
Let, $\xi_{\t} := \max_{c\in [C]} \xi_{c}^{\text{T2I}}$ and $\xi_{\s} := \max_{c\in [C]} \xi_{c}^{\text{I2I}}$.
\end{definition}

\noindent\textbf{Remarks:}
    Note that $\tilde{\k}_{\q_c}$ is the expected version, while $\tilde{\k}_c$ (defined in \Cref{def:setup}) is the empirical mean of retrieved samples for class $c\in[C]$. 

At inference time, for a test sample $(\x,y)\sim\mathcal{D}_T$ with image feature $\z = \mathcal{I}(\x)$, one of the following four events can happen: 

\begin{align*}
    E_1 := & \{(\x,y)\sim\mathcal{D}_T: y \neq \argmax_{i\in[C]} \t_i^\top \z \text{ and } y \neq \argmax_{i\in[C]} \bk_i^\top \z\}\\
    E_2 := & \{(\x,y)\sim\mathcal{D}_T: y = \argmax_{i\in[C]} \t_i^\top \z \text{ and } y \neq \argmax_{i\in[C]} \bk_i^\top \z\}\\
    E_3 := & \{(\x,y)\sim\mathcal{D}_T: y \neq \argmax_{i\in[C]} \t_i^\top \z \text{ and } y = \argmax_{i\in[C]} \bk_i^\top \z\}\\
    E_4 := & \{ (\x,y)\sim\mathcal{D}_T: y = \argmax_{i\in[C]} \t_i^\top \z \text{ and } y = \argmax_{i\in[C]} \bk_i^\top \z\}.
\end{align*}

We can see that $\cR_{0-1}(f^{\text{ZOC}}) = \Pr(E_1) + \Pr(E_3)$ and $\cR_{0-1}(f^{\text{RET}}) = \Pr(E_1) + \Pr(E_2)$.
Next, we formalize the intuitions in \Cref{fig:error_sample} as the following definition:

\begin{definition}[Knowledge encoded in different modalities]\label{def:diff}
For a vector $\v \in \R^C$  and a scalar $i \in [C]$, We define $\phi(\v,i,z): = \{j | \v_i - \v_j \le z \}$.  
Consider $(\x, y) \sim \cD_{T}$  and $\z = \mathcal{I}(\x)$.
We define the conditional probability $\rho_d(z)$ as  $$\rho_d(z) = \Pr\left( \phi(\T^\top\z,y,z) \cap \phi(\bK^\top\z,y,z) \neq \{y\} \middle| E_2 \text{ or } E_3 \right).$$
\end{definition}

\noindent\textbf{Remarks:} $\phi(\v,i,z)$ identifies elements in vector $\v$ that are within a threshold $z$ of the $i$-th element of $\v$. $\rho_d(z)$ represents the likelihood that, given events $E_2$ or $E_3$, the transformed data $\z$ is associated with an incorrect class by both $\T$ and $\bK$.
In practical scenarios, $\rho_d(z)$ is typically small. This is because different modalities usually represent knowledge in distinct ways and, as a result, have different patterns of confusion or error.

\begin{assumption}[Sample representativeness]\label{ass:good_support}
We assume that the sample for each class is relatively representative, i.e., $\forall i \in [C], \left\|\s_i - {\bs_i}\right\|_2 \le \kappa$ for some constant $\kappa$. 
\end{assumption}

\begin{assumption}[Retrieved data distribution]\label{ass:retrieve}
We assume that for each class the distribution of retrieved samples is composed of clusters, which exhibit $\nu$ separation and $\kappa$ concentration as defined in \Cref{def:concentration}. We assume that the retrieval process for a query sample is uniformly sampling from its closest retrieval cluster.  
\end{assumption}

\subsection{Main Results and Analysis}
\begin{lemma}\label{lem:soln_good} 
We can upper bound the risk $\cR(\bS)$ as follows:
\begin{align}
    \cR(\bS) \le (1-\rho_c)\log\left(1 + {(C-1) \exp\left(  2 \kappa - \nu \right)} \right) + \rho_c \log\left(1 + { (C-1) \exp\left( 2\right)} \right)
\end{align}
where $\rho_c, \kappa, \nu$ defined in~\Cref{def:concentration} characterize the inner-class concentration and inter-class separation.
\end{lemma}
\begin{proof}
For a test sample $(\x, y) \sim \cD_{T}$  with $\z = \mathcal{I}(\x)$. Let $\z = \v + {\bs_{y} }$. 
By \Cref{def:concentration}, we have $\Pr\left( \left\|\v\right\|_2 \ge \kappa \right) \le \rho_c $. Thus, we have
\begin{align}
    \cR(\bS) = & \E \left[\ell(\bS^\top \z, y) \right]\\
    = & \E \left[-\log{\exp\left({\bs_{y}^\top \z }\right) \over \sum_{i \in [C]} \exp\left({\bs_{i}^\top \z }\right) } \right] \\
    = & \E \left[\log\left(1 + {\sum_{i \neq y} \exp\left({\bs_{i}^\top \z } - {\bs_{y}^\top \z }\right)} \right)\right] \\
    = & \E \left[\log\left(1 + {\sum_{i \neq y} \exp\left({\bs_{i}^\top } \left(\v + {\bs_{y} }\right) - {\bs_{y}^\top }\left(\v + {\bs_{y} }\right)\right)} \right)\right] \\
    \le & (1-\rho_c)\E \left[\log\left(1 + {\sum_{i \neq y} \exp\left({\bs_{i}^\top \v } + 1 - \nu - {\bs_{y}^\top\v }-1\right)} \right) \middle| \left\|\v\right\|_2 \le \kappa \right] \\
    & + \rho_c\E \left[\log\left(1 + {\sum_{i \neq y} \exp\left(  2 \right)} \right)\middle| \left\|\v\right\|_2 \ge \kappa\right]
    \\
    \le & (1-\rho_c)\E \left[\log\left(1 + {\sum_{i \neq y} \exp\left(  2 \|\v\|_2 - \nu \right)} \right) \middle| \left\|\v\right\|_2 \le \kappa \right] + \rho_c \log\left(1 + { (C-1) \exp\left( 2\right)} \right) \\
    \le & (1-\rho_c)\log\left(1 + {(C-1) \exp\left(  2 \kappa - \nu \right)} \right) + \rho_c \log\left(1 + { (C-1) \exp\left( 2\right)} \right).
\end{align}
\end{proof}

\noindent\textbf{Remarks:}
\Cref{lem:soln_good} is a tight upper bound. We give a simple toy example here for illustration: consider binary classification on two data points $(\x_1, y_1)$ and $(\x_2, y_2)$. Suppose $\z_1 = \mathcal{I}(\x_1) = -\z_2 = -\mathcal{I}(\x_2)$, we can see that $\cR(\bS) = \cR(\bZ^*) = \log\left(1 + { \exp\left( - 2 \right)} \right)$, where $C = 2, \rho_c = \kappa =0, \nu = 2$.

\begin{lemma}\label{lem:top_acc}
For a test sample $(\x, y) \sim \cD_{T}$ and its image feature $\z = \mathcal{I}(\x)$, with probability at least $1-\rho_c$, we have $$\max_{i\neq y} \s_i^\top\z - \s_y^\top\z \le 4 \kappa - \nu.$$
\end{lemma}

\begin{proof}[Proof of \Cref{lem:top_acc}]
Let $\z = \v + {\bs_{y} }$. 
By \Cref{def:concentration} and \Cref{ass:good_support}, we have $\Pr\left( \left\|\v\right\|_2 \ge \kappa \right) \le \rho_c $. Thus, we have with probability at least $1-\rho_c$ such that
\begin{align}
    \max_{i\neq y} \s_i^\top\z - \s_y^\top\z = & \max_{i\neq y} \left(\s_i - {\bs_i } + {\bs_i }\right)^\top \left(\v+{\bs_{y} }\right) - \left(\s_y - {\bs_y } + {\bs_y }\right)^\top\left(\v+{\bs_{y} }\right) \\
    = & \max_{i\neq y} \s_i^\top \v + \left(\s_i - {\bs_i } \right)^\top {\bs_{y} } +  {\bs_i }^\top{\bs_{y} }\\
    & \quad\quad - \s_y^\top \v - \left(\s_y - {\bs_y } \right)^\top {\bs_{y} }- {\bs_y }^\top {\bs_{y} }\\
    \le & \max_{i\neq y} \kappa + \kappa + 1 - \nu + \kappa + \kappa - 1\\
    = & 4 \kappa - \nu.
\end{align}
\end{proof}
\noindent\textbf{Remarks:} From the above lemma, we can see that if $4 \kappa < \nu $, the accuracy of $f^{\text{RET}}(\cdot)$ is at least $1-\rho_c$.

\begin{lemma}[Retrieval distribution shift bound]\label{lem:uni}
Under \Cref{ass:good_support} and \Cref{ass:retrieve} and suppose that $\s_i$ is in the support of $\cD_R$, we have $\xi_{\s} \le 2 \kappa^2$. Furthermore, when the retrieval cluster for $\t_i$ and $\s_i$ are different for any $i \in [C]$, we have $\xi_{\t} \ge \nu - 2 \kappa.$
\end{lemma}  
\begin{proof}[Proof of \Cref{lem:uni}]
By \Cref{ass:good_support} and \Cref{ass:retrieve}, for any $i \in [C]$, we have
\begin{align}
    \xi_{i}^{\text{I2I}} =  & 1 - {\bk_{\s_i}^\top \bs_i }\\
    = & {1\over 2} \left\|{\bs_i  } - {\bk_{\s_i}  }\right\|_2^2\\
    \le & 2 \kappa^2. 
\end{align}
Furthermore, when the retrieval clusters for $\t_i$ and $\s_i$ are different, by \Cref{ass:retrieve}, we have 
\begin{align}
    \xi_{i}^{\text{T2I}} = & 1 - {\bk_{\t_i}^\top \bs_i }\\
    = & 1 - \left({\bk_{\t_i}  }\right)^\top\left({\bs_i  } - {\bk_{\s_i} } + {\bk_{\s_i} }\right)\\
    = &  1 - \left({\bk_{\t_i}  } \right)^\top\left({\bs_i  } - {\bk_{\s_i}  }\right) - {\bk_{\t_i}  } ^\top{\bk_{\s_i}  }\\
    \ge &  \nu - \left\|{\bs_i  } - {\bk_{\s_i}  }\right\|_2 \\
    = & \nu - \sqrt{2\xi_{i}^{\text{I2I}}} \\
    \ge & \nu - 2 \kappa. 
\end{align}
\end{proof}

\begin{theorem}[Benefit of uni-modal retrieval]\label{thm:uni}
Assume the same condition as \Cref{lem:uni}, with probability at least $1-\delta$, the following upper bound of the ensemble risk holds:
\begin{align}
    \cR(\alpha \T + \gamma \bK) - \cR(\bS) \le  & L \Bigg(\alpha \underbrace{\|(\T - \bS)^\top \z\|_2 } _{\text{modality gap}} + \underbrace{\gamma\kappa\sqrt{{8C  \over K} \log{C\over \delta} }}_{\text{retrieval sample complexity}} + \underbrace{\gamma\sqrt{2C \xi}}_{\text{retrieval shift}} \Bigg),
\end{align}
where $L = \sqrt{\exp(2) + 1}$, $\kappa$ characterizes the inner-class feature concentration (\Cref{def:concentration}), and $\xi$ is either $\xi_{\s}$ for I2I retrieval or $\xi_{\t}$ for T2I retrieval.
\end{theorem} 
\begin{proof}[Proof of \Cref{thm:uni}]
By \Cref{lem:loss_lip} and \Cref{lem:logits_bound}, let $L = \sqrt{\exp(2) + 1}$, we have:

\begin{align}
    \cR(\alpha \T + \gamma \bK) - \cR(\bS)  \le &   L \left(\alpha \|(\T - \bS)^\top \z\|_2 + \gamma \|(\bK - \bS)^\top \z\|_2 \right). \label{eq:1}
\end{align}
By the vector Bernstein inequality in \Cref{lem:bernstein} and the union bound, with probability at least $1-\delta$, for any $c\in[C]$: 
\begin{align}
    \|\bk_c - \bk_{\q_c}\|_2 \le \kappa\sqrt{{8  \over K} \log{C\over \delta} },
\end{align}
This bound characterizes the retrieval sample complexity. Moreover, from the definition of the retrieval distributional shift, we have $\| \bk_{\q_c} - \bs_c\|_2 = \sqrt{2 - 2\bk_{\q_c} ^\top \bs_c} = \sqrt{2\xi_c}$, where $\q_c = \s_c$ for I2I retrieval and $\q_c=\t_c$ for T2I retrieval. Therefore, we obtain an upper bound of $\|(\bK - \bS)^\top \z\|_2$ as:
\begin{align}
    \|(\bK - \bS)^\top \z\|_2 \leq \kappa\sqrt{{8C  \over K} \log{C\over \delta} } + \sqrt{2C \xi}  \label{eq:2}
\end{align}
We obtain the final bound by putting together Eq.~\eqref{eq:1} and Eq.~\eqref{eq:2}.
\end{proof}

\noindent\textbf{Remarks:} The above upper bound consists of three terms: the gap between the textual and visual modality, the sample complexity of retrieved features which decreases as we increase $K$, and a term related to the distributional shift induced by the retrieval method. By \Cref{lem:uni}, we can see the superiority of I2I over T2I retrieval by comparing $\xi_{\s}$ and $\xi_{\t}$.

\begin{theorem}[Benefit of logit ensemble]\label{thm:ensemble}
Assume the same condition as \Cref{lem:uni}. For I2I retrieval with $\alpha = \gamma = {1\over 2}, \beta =0$, we have
\begin{align}\cR_{0-1}(f) \le \Pr\left(E_1\right)  + 
    (\Pr(E_2)+\Pr(E_3))  \rho_d(\max\{6\kappa-\nu,2\kappa  + \tau\})  + \rho_c
    \end{align}
\end{theorem}
\begin{proof}[Proof of \Cref{thm:ensemble}]
We define the events $E_c = \left\{\left\|\mathcal{I}(\x) - {\bs_i }\right\|_2 \ge \kappa \text{ and } y = i, \forall i \in [C] \right\}$. 
We also define events $E_d(z) = \left\{ \phi(\T^\top\z,y,z) \cap \phi(\bK^\top\z,y,z) = \{y\}  \right\}$. Note that we have $\Pr\left(E_d(z) \middle| E_2 \text{ or } E_3 \right) = 1- \rho_d(z)$.
By \Cref{def:diff}, we have 
\begin{align}
    \max_{ (\x,y) \in E_c, y=i \neq j}  (\t_j - \t_i)^\top \z = & \max_{ (\x,y) \in E_c, y=i \neq j}  (\t_j - \t_i)^\top (\z - {\bs_i } + {\bs_i })\\
    = & \max_{ (\x,y) \in E_c, y=i \neq j}  (\t_j - \t_i)^\top (\z - {\bs_i })  + (\t_j - \t_i)^\top {\bs_i } \\
    \le & 2\kappa  + \tau.
\end{align}
By \Cref{lem:top_acc} and \Cref{ass:good_support} and \Cref{ass:retrieve}, conditional on $E_c$, we have the logits gap $\max_{i\neq y} \bk_i^\top\z - \bk_y^\top\z \le 6 \kappa - \nu$. Let $\ACC(f) = 1 - \cR_{0-1}(f)$. 
Then, we get
\begin{align}
    \ACC(f) = & \Pr\left(y = \argmax_{i} {1\over 2} \t_i^\top \z + {1\over 2} \bk_i^\top \z\right)\\
    \ge & \Pr\left(E_4\right)  + 
    \Pr(E_c \cap E_2)\Pr\left(y = \argmax_{i} \t_i^\top \z +  \bk_i^\top \z\middle |E_c \cap E_2\right) \\
    & +  
    \Pr(E_c \cap E_3)\Pr\left(y = \argmax_{i} \t_i^\top \z +  \bk_i^\top \z\middle |E_c \cap E_3\right) \\
    = & \Pr\left(E_4\right)  + 
    \Pr(E_c \cap E_2)\Pr\left( \max_{y=i\neq j} (\t_j - \t_i)^\top \z +  (\bk_j - \bk_i)^\top \z < 0 \middle |E_c \cap E_2\right) \\
    & +  
    \Pr(E_c \cap E_3)\Pr\left( \max_{y=i\neq j} (\t_j - \t_i)^\top \z +  (\bk_j - \bk_i)^\top \z < 0 \middle |E_c \cap E_3\right). 
\end{align}
Now, we prove that $E_d(6\kappa-\nu) \cap E_c \cap E_2 \subseteq \{ \max_{y=i\neq j} (\t_j - \t_i)^\top \z +  (\bk_j - \bk_i)^\top \z < 0 \} \cap  E_c \cap E_2$.

For any $(\x,y) \in E_d(6\kappa-\nu) \cap E_c \cap E_2$ and $y=i\neq j$, 
\begin{itemize}
    \item if $j \in \phi(\T^\top\z,y,z)$ and $j \notin \phi(\bK^\top\z,y,z)$, we have $(\t_j - \t_i)^\top \z +  (\bk_j - \bk_i)^\top \z < 0 - (6\kappa-\nu)  \le 0$; 
    \item if $j \notin \phi(\T^\top\z,y,z)$ and $j \in \phi(\bK^\top\z,y,z)$, by \Cref{lem:top_acc}, $(\t_j - \t_i)^\top \z +  (\bk_j - \bk_i)^\top \z < - (6\kappa-\nu) + 6\kappa-\nu  = 0$; 
    \item  if $j \notin \phi(\T^\top\z,y,z)$ and $j \notin \phi(\bK^\top\z,y,z)$, we have $(\t_j - \t_i)^\top \z +  (\bk_j - \bk_i)^\top \z < - (6\kappa-\nu) - 6\kappa-\nu  < 0$.
\end{itemize}
Thus, we have $\max_{y=i\neq j} (\t_j - \t_i)^\top \z +  (\bk_j - \bk_i)^\top \z < 0$. Therefore, $E_d(6\kappa-\nu) \cap E_c \cap E_2 \subseteq \{ \max_{y=i\neq j} (\t_j - \t_i)^\top \z +  (\bk_j - \bk_i)^\top \z < 0 \} \cap  E_c \cap E_2$.

Similarly, by $\max_{ (\x,y) \in E_c, y=i \neq j}  (\t_j - \t_i)^\top \z \le 2\kappa  + \tau$, we have $E_d(2\kappa  + \tau) \cap E_c \cap E_3 \subseteq \{ \max_{y=i\neq j} (\t_j - \t_i)^\top \z +  (\bk_j - \bk_i)^\top \z < 0 \} \cap  E_c \cap E_3$. 

Thus, as $E_2$ and $E_3$ are disjoint and union bound, we have 
\begin{align}
    \ACC(f) \ge & \Pr\left(E_4\right)  + 
    \Pr(E_c \cap E_2)\Pr\left( E_d(6\kappa-\nu) \middle |E_c \cap E_2\right) \\
    & +  
    \Pr(E_c \cap E_3)\Pr\left( E_d(2\kappa  + \tau)\middle |E_c \cap E_3\right)\\
    \ge & \Pr\left(E_4\right)  + 
    (\Pr(E_2)+\Pr(E_3)) (1 -\rho_d(\max\{6\kappa-\nu,2\kappa  + \tau\}))  - \rho_c.
\end{align}
We finish the proof by following $\ACC(f) = 1 - \cR_{0-1}(f)$ and $\Pr(E_1) + \Pr(E_2) + \Pr(E_3) + \Pr(E_4) = 1$.
\end{proof}

\textbf{Remarks:}  The above theorem characterizes the 0-1 risk upper bound by the modality gap and key properties of retrieved and target distributions. Moreover, logit ensemble utilizes knowledge encoded in different modalities to benefit each other. When $(\Pr(E_2)+\Pr(E_3)) (1 -\rho_d(\max\{6\kappa-\nu,2\kappa  + \tau\}))  - \rho_c \ge \max\{\Pr(E_2), \Pr(E_3)\}$, we can see that logit ensemble leads to a lower 0-1 risk (\emph{i.e.,} higher accuracy) compared to the zero-shot model. This happens when the modality gap $\tau$ is small and the test data exhibits good clustering properties.

\subsection{Auxiliary Lemmas}
\begin{lemma}[Lipschitz continuity of cross-entropy loss]\label{lem:loss_lip}  
When $y\in[C]$, the cross-entropy loss $\ell(\v,y)$ is $L$-Lipschitz on the hyper-cube, i.e., $\v \in [-1,1]^C$, where $L = \sqrt{\exp(2) + 1}$.
\end{lemma}
\begin{proof}[Proof of~\Cref{lem:loss_lip}]
Note that since $\ell(\cdot,y): \R^C \rightarrow \R$ is differentiable, it is sufficient to find $L$ such that $\|\nabla \ell(\cdot,y)\|_2 \le L$.
Let $s = \sum_{i \in [C]} \exp(\v_i)$. Applying calculus rules we have that
\begin{align}
    {\partial \ell \over \partial \v_y} =  {\exp(\v_y) - s  \over s} \quad \text{and} \quad {\partial \ell \over \partial \v_i} =  {\exp(\v_y+\v_i)\over s} ~~ \forall i \neq y.
\end{align}
Thus,
\begin{align}
    \|\nabla \ell(\cdot,y)\|_2^2 = & { \left(\sum_{i\neq y} \exp(\v_i)\right)^2 + \exp(2\v_y)\left(\sum_{i\neq y} \exp(2\v_i) \right) \over s^2} \\
    \le & { s^2 + \exp(2\v_y)s^2 \over s^2}\\
    \le & \exp(2) + 1.
\end{align}
Thus, we have $L = \sqrt{\exp(2) + 1}$.
\end{proof}

\begin{lemma}[Bounded logits]\label{lem:logits_bound}
For an input with visual feature $\z \in \R^\nd$, if $\Q$ is a convex combination among $\{\T, \S, \bS, \bK\}$, we have $\Q^\top \z \in [-1,1]^C$ .
\end{lemma}
\begin{proof}[Proof of~\Cref{lem:logits_bound}]
From the definitions of matrices $\T, \S, \bS, \bK \in \R^{\nd \times C}$  defined in~\Cref{def:setup} and~\Cref{def:def_assump}, we have that the Euclidean norm of each column in $\T, \S, \bS, \bK$ and $\z$ is smaller or equal to 1. Thus, their convex combination $\Q$ multiplied by $\z$ satisfies $\Q^\top \z \in [-1,1]^C$. 
\end{proof}

\begin{lemma}[Vector Bernstein inequality. Lemma 18 in \citet{kohler2017sub}]\label{lem:bernstein}
    Let $\v_1, . . . , \v_n \in \R^d$  be independent vector-valued random variables and assume that each one is centered, uniformly bounded with variance bounded above:
    \begin{align}
        \E [\v_i] = 0 \text{ and } \|\v_i\|_2 \le B_2 \text{ as well as } \E[ \|\v_i\|_2^2] \le \sigma^2.
    \end{align}

Let $\widehat \v = {1\over n} \sum_{i=1}^n \v_i$.
Then we have for $0 < \epsilon < \sigma^2/B_2$,
\begin{align}
    \Pr( \|\widehat \v\|_2 \ge \epsilon) \le \exp\left(-n \cdot {\epsilon^2 \over 8 \sigma^2} + {1\over 4} \right).
\end{align}
\end{lemma}

\section{Training-based Adaptation}
\label{sec:train_adapt}
In Section~\ref{sec:discuss}, we have shown the average performance of training-based adaptation, where the feature cache is finetuned (based on the RN50 backbone).  In this section, we report the performance for each dataset. The results are shown in Figure~\ref{fig:t-adapt}. The result for each dataset is consistent where I2I retrieval outperforms T2I retrieval and zero-shot CLIP when varying the shot number from 2 to 16.

\begin{figure*}[ht]
\centering
\begin{subfigure}[b]{0.24\textwidth}
    \includegraphics[width=\textwidth]{figures/tipf/average.pdf}
\end{subfigure}
\begin{subfigure}[b]{0.24\textwidth}
    \includegraphics[width=\textwidth]{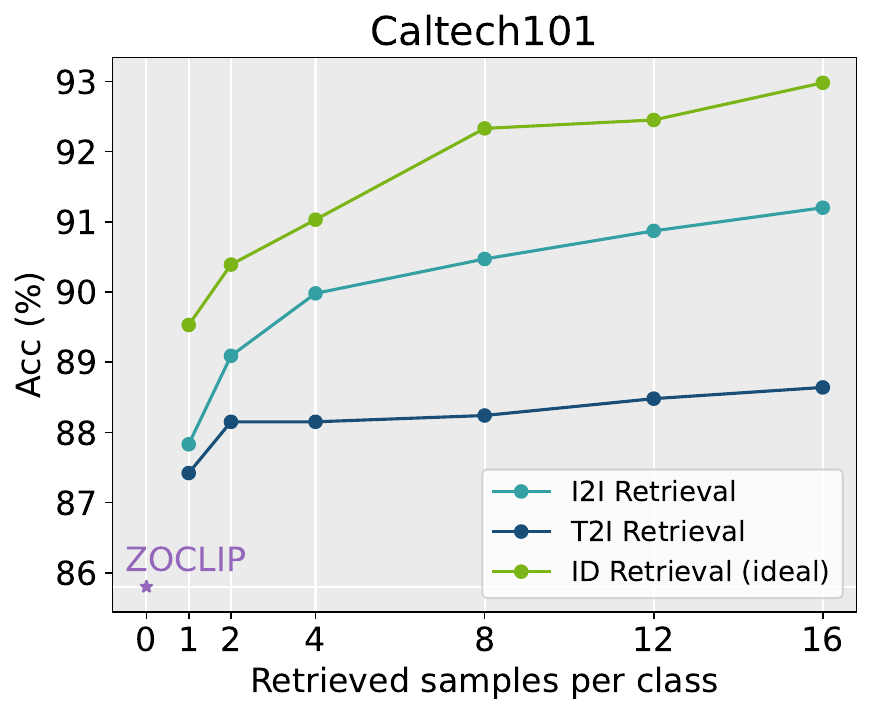}
\end{subfigure}
\begin{subfigure}[b]{0.24\textwidth}
    \includegraphics[width=\textwidth]{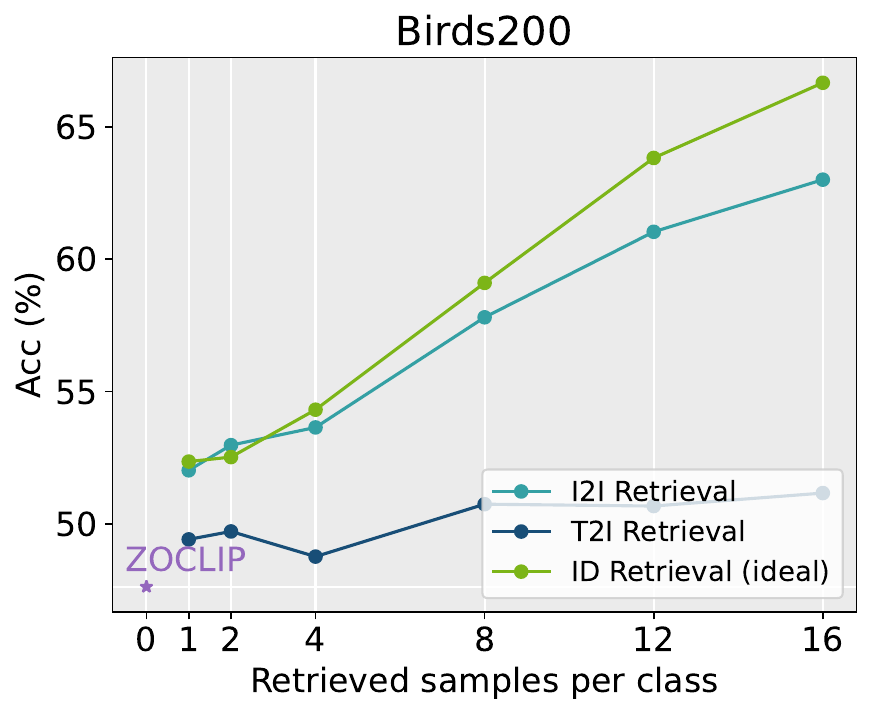}
\end{subfigure}
\begin{subfigure}[b]{0.24\textwidth}
    \includegraphics[width=\textwidth]{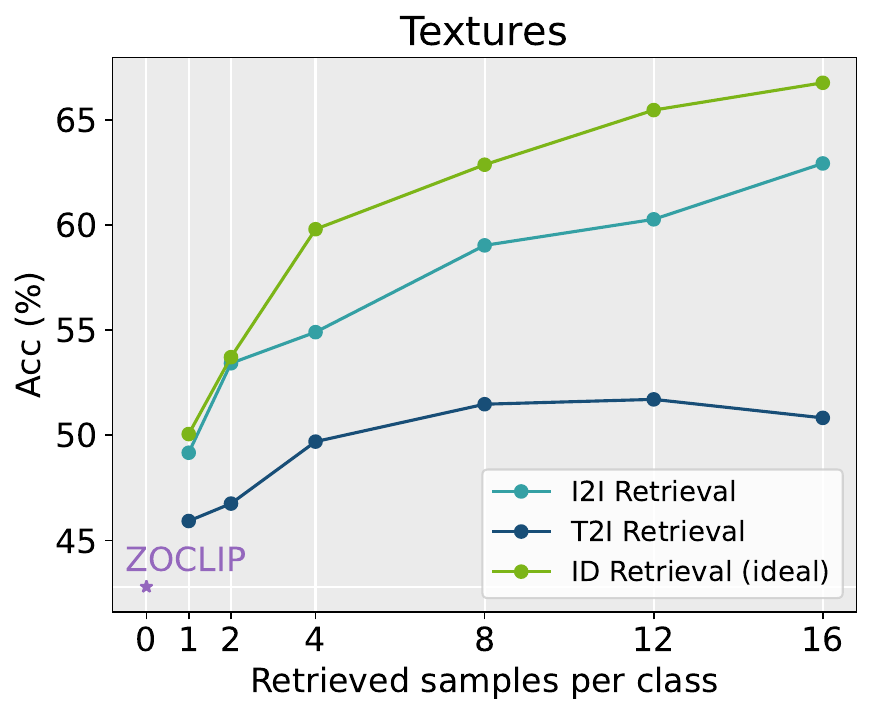}
\end{subfigure}
\begin{subfigure}[b]{0.24\textwidth}
    \includegraphics[width=\textwidth]{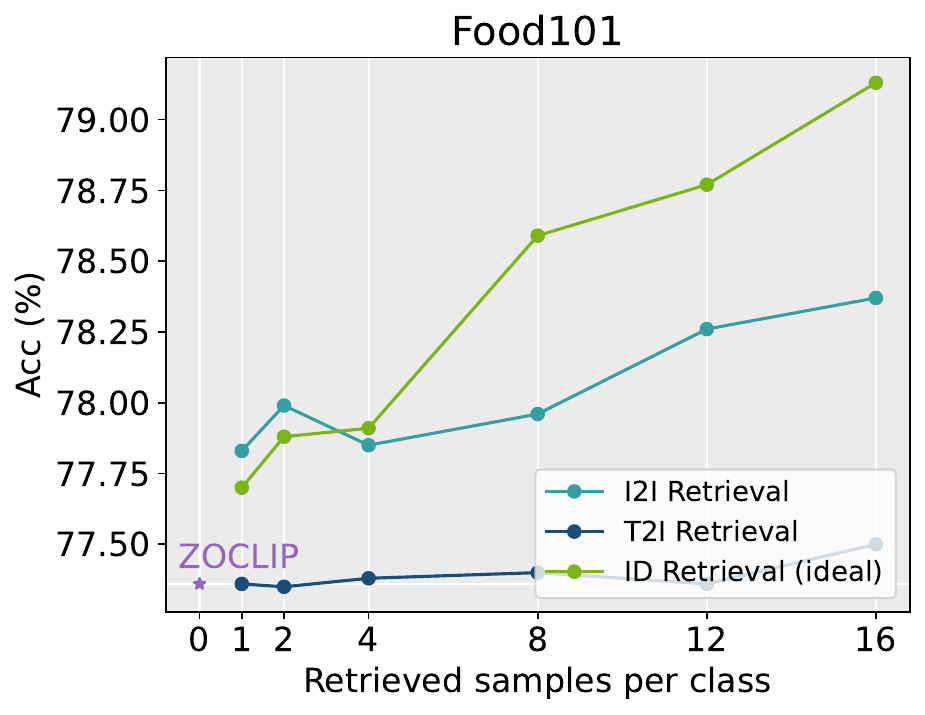}
\end{subfigure}
\begin{subfigure}[b]{0.24\textwidth}
    \includegraphics[width=\textwidth]{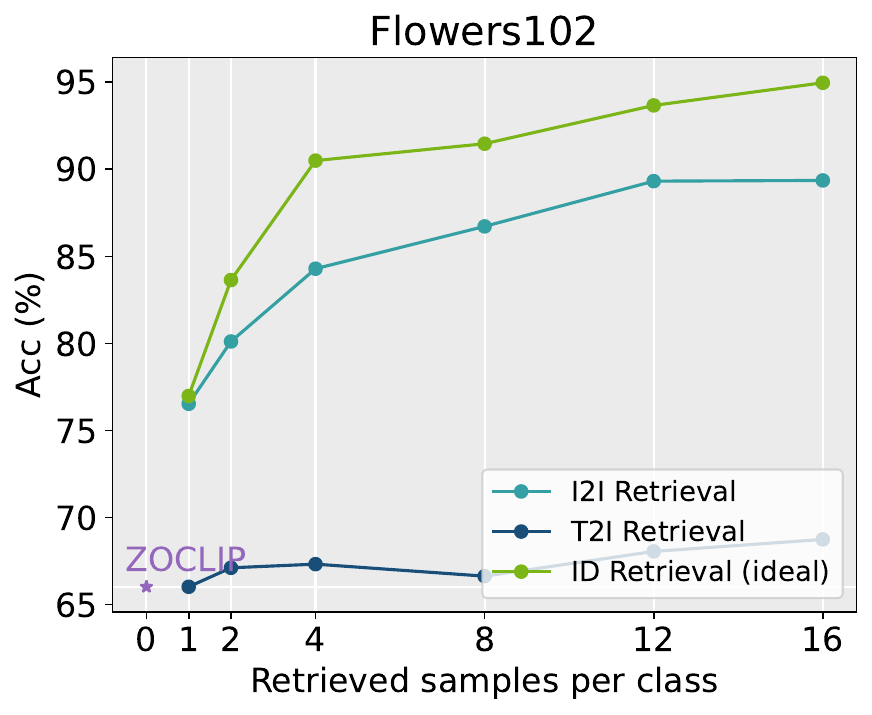}
\end{subfigure}
\begin{subfigure}[b]{0.24\textwidth}
    \includegraphics[width=\textwidth]{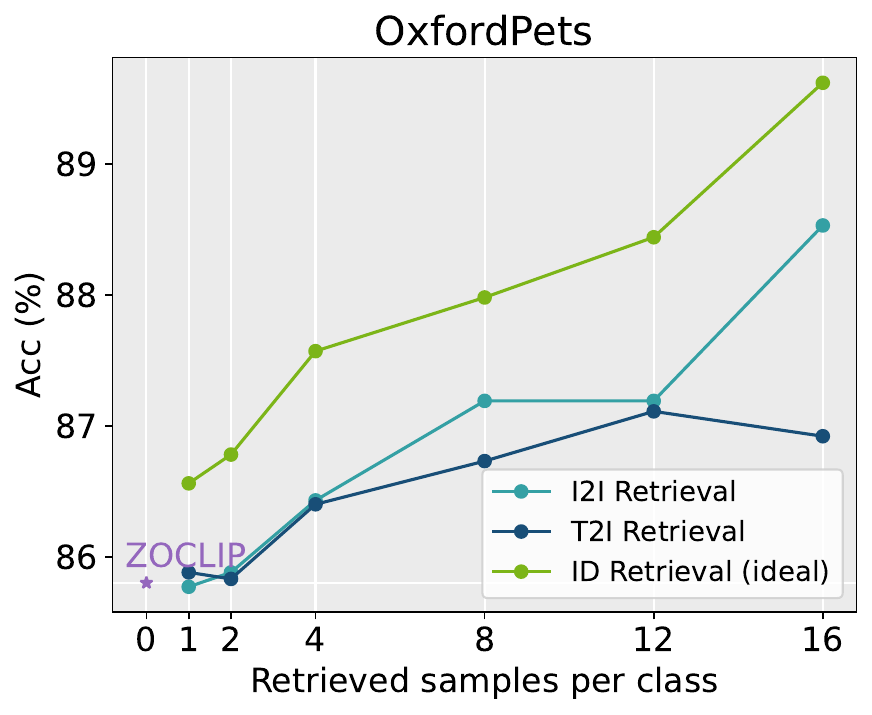}
\end{subfigure}
\begin{subfigure}[b]{0.24\textwidth}
    \includegraphics[width=\textwidth]{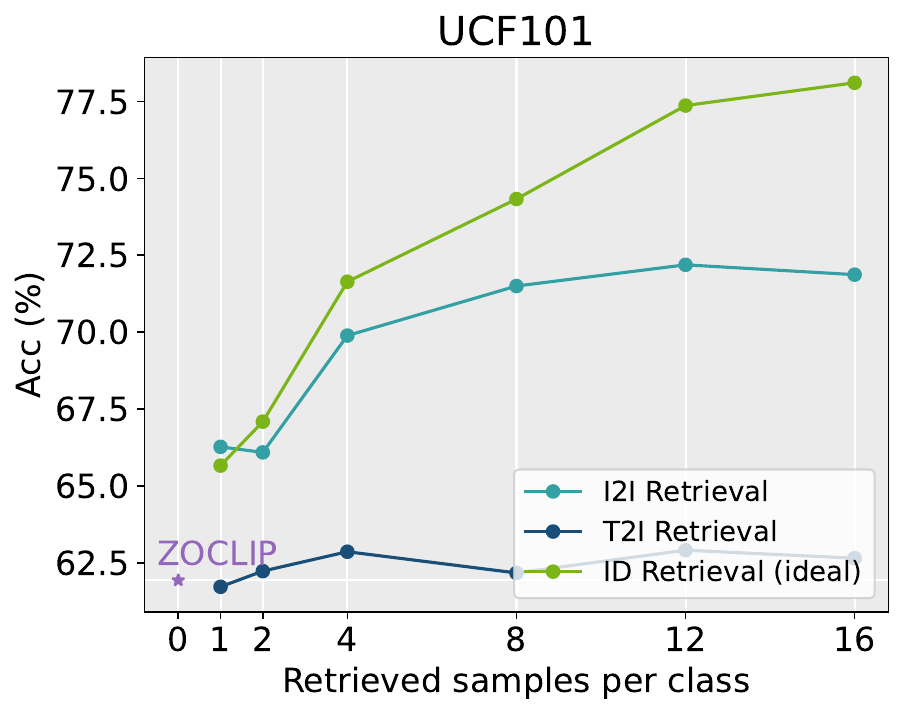}
\end{subfigure}
\caption{Comparison of retrieval method on adaptation with finetuned feature. Results are based on RN50. We observe a trend similar to training-free adaptation, where I2I retrieval consistently outperforms T2I retrieval and zero-shot CLIP.}
\label{fig:t-adapt}
\end{figure*}

\begin{figure*}[ht]
\centering
\begin{subfigure}[b]{0.24\textwidth}
    \includegraphics[width=\textwidth]{figures/ablation/vitb32/tip/average.pdf}
\end{subfigure}
\begin{subfigure}[b]{0.24\textwidth}
    \includegraphics[width=\textwidth]{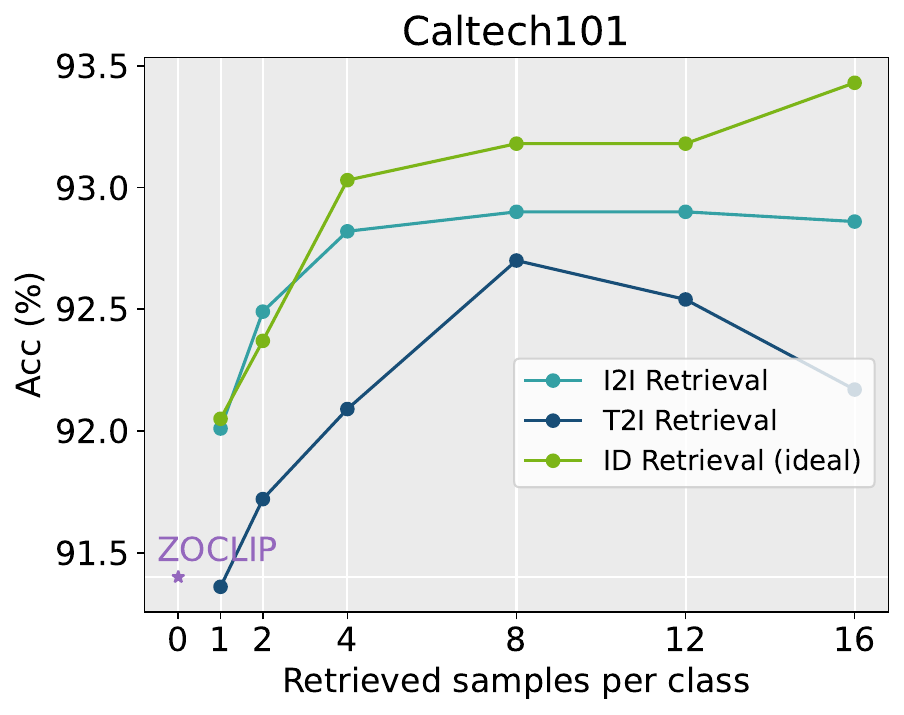}
\end{subfigure}
\begin{subfigure}[b]{0.24\textwidth}
    \includegraphics[width=\textwidth]{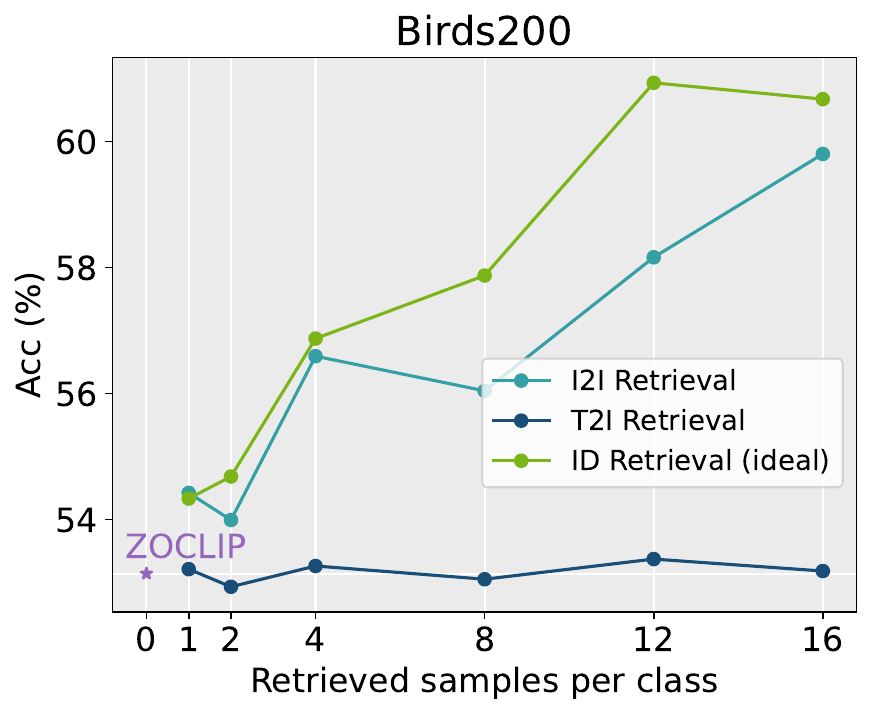}
\end{subfigure}
\begin{subfigure}[b]{0.24\textwidth}
    \includegraphics[width=\textwidth]{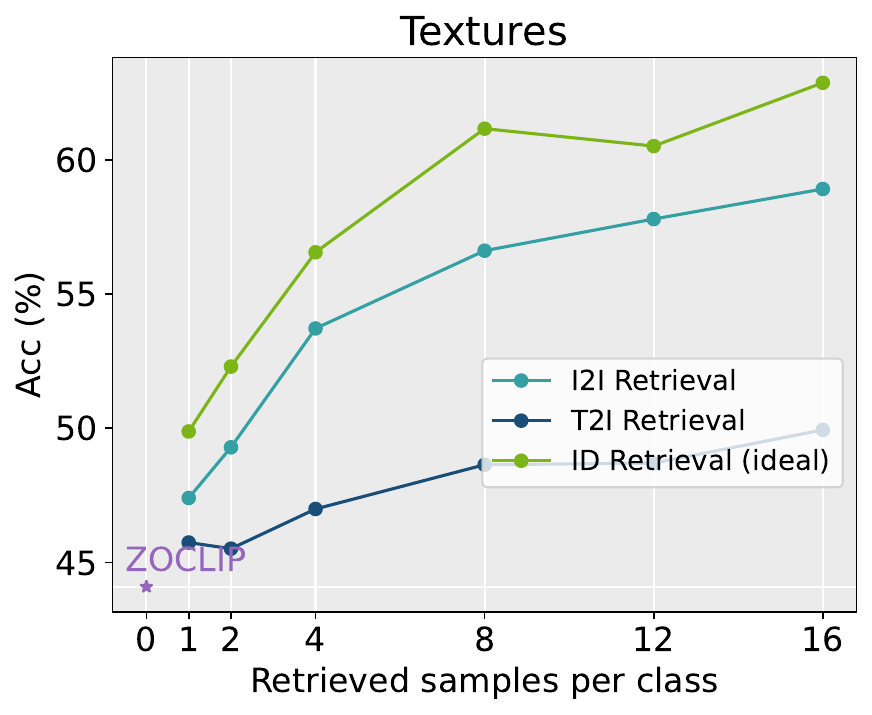}
\end{subfigure}
\begin{subfigure}[b]{0.24\textwidth}
    \includegraphics[width=\textwidth]{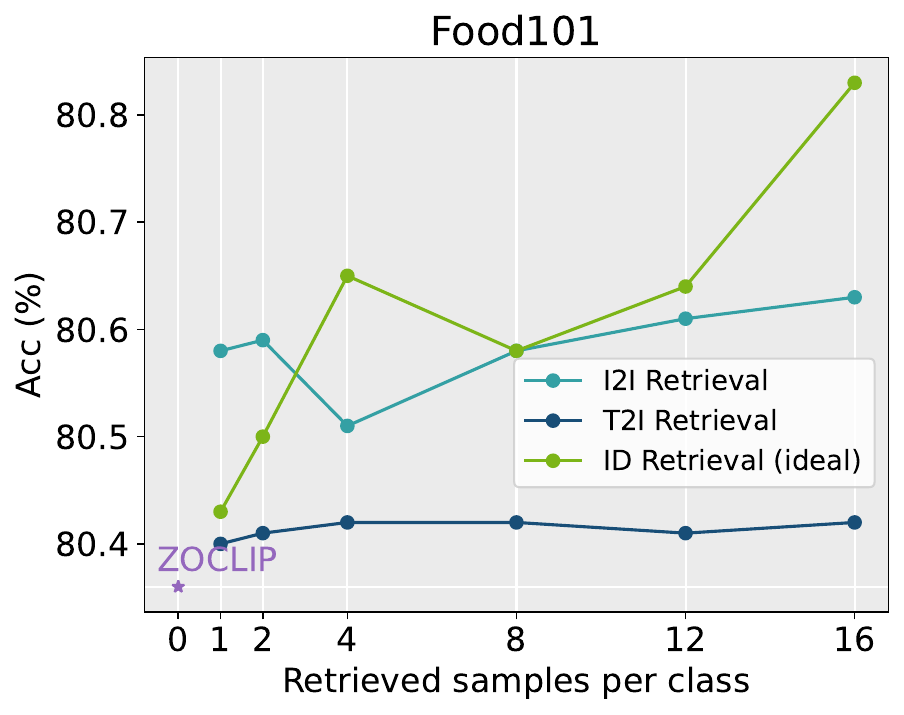}
\end{subfigure}
\begin{subfigure}[b]{0.24\textwidth}
    \includegraphics[width=\textwidth]{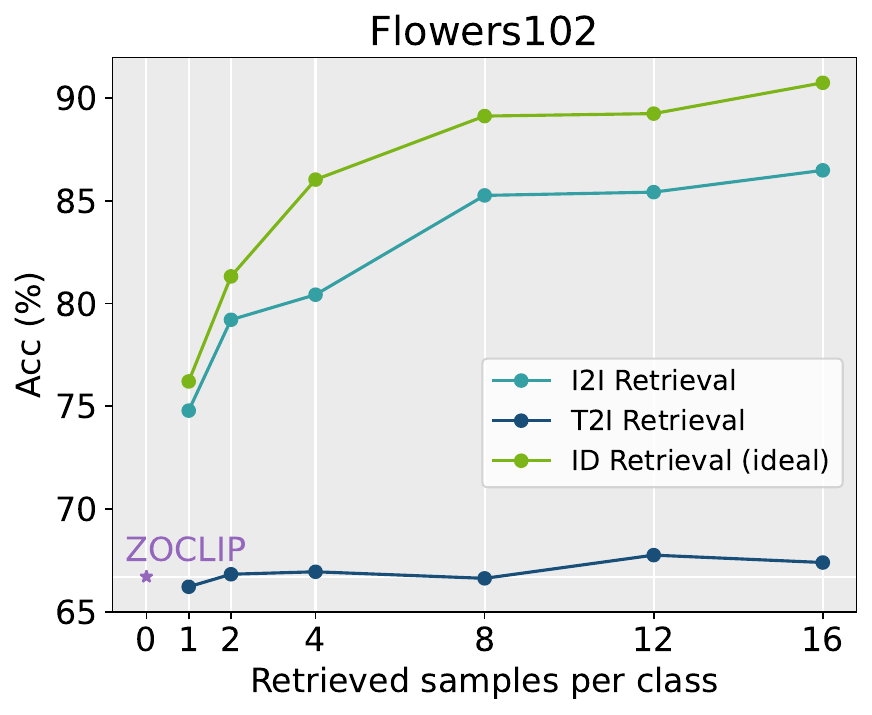}
\end{subfigure}
\begin{subfigure}[b]{0.24\textwidth}
    \includegraphics[width=\textwidth]{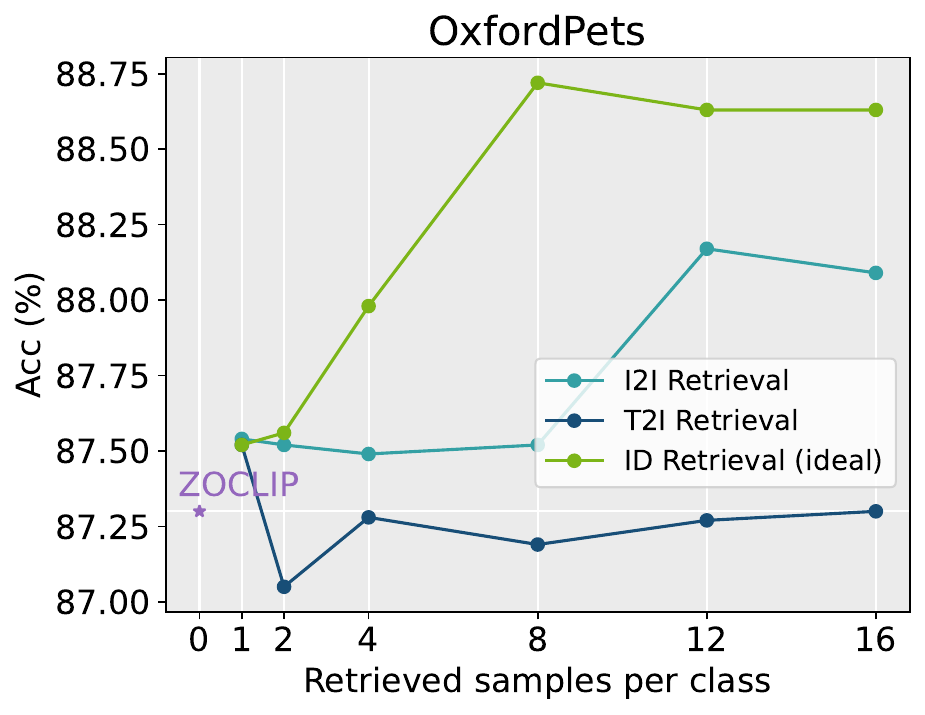}
\end{subfigure}
\begin{subfigure}[b]{0.24\textwidth}
    \includegraphics[width=\textwidth]{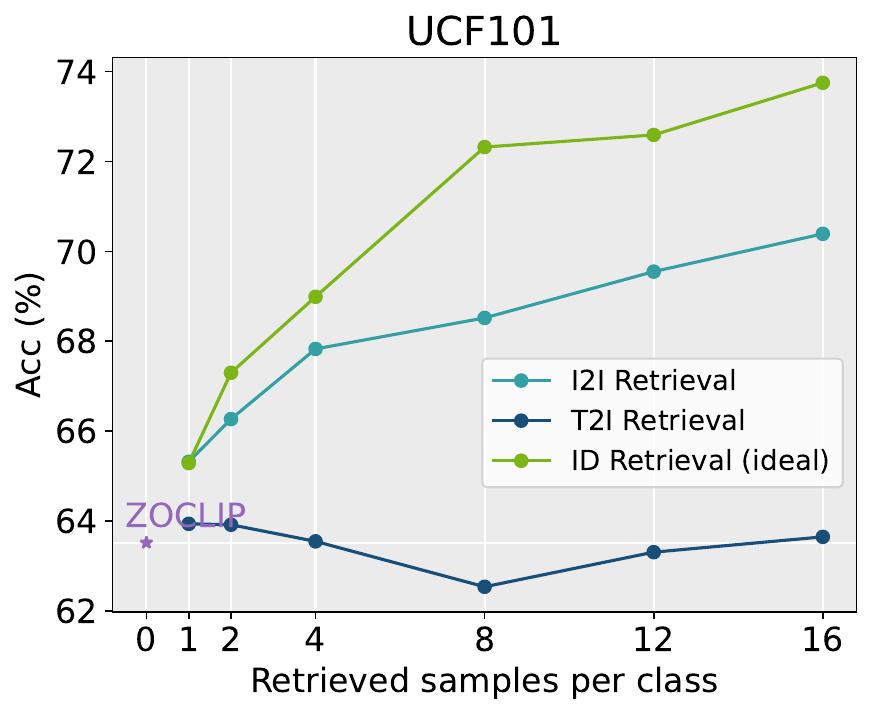}
\end{subfigure}
\caption{Impact of model architecture. Results are based on ViT-B/32 (training-free).}
\label{fig:b32-adapt}
\end{figure*}

\begin{figure*}[ht]
\centering
\begin{subfigure}[b]{0.24\textwidth}
    \includegraphics[width=\textwidth]{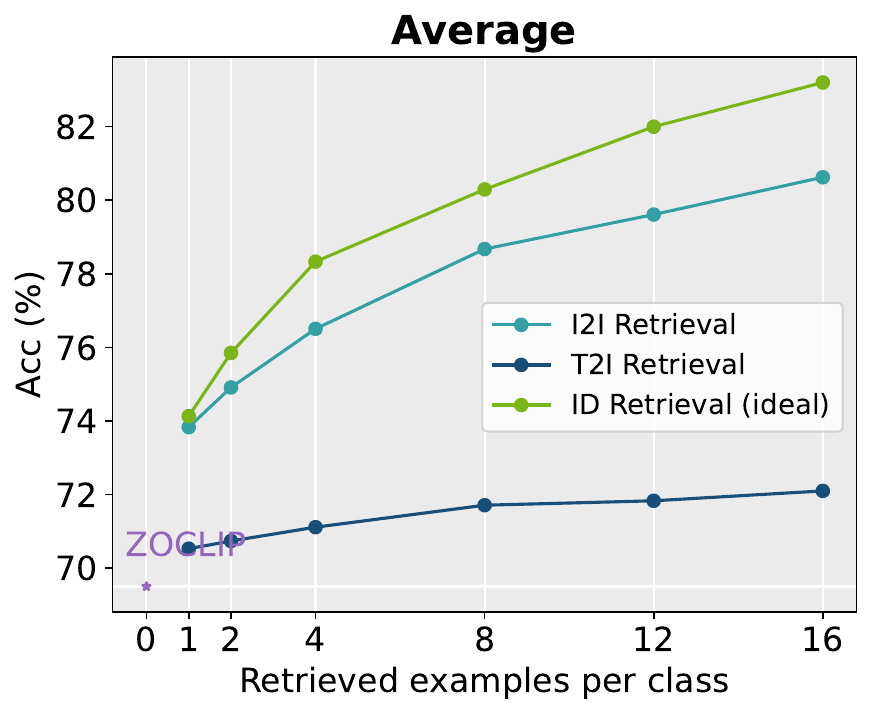}
\end{subfigure}
\begin{subfigure}[b]{0.24\textwidth}
    \includegraphics[width=\textwidth]{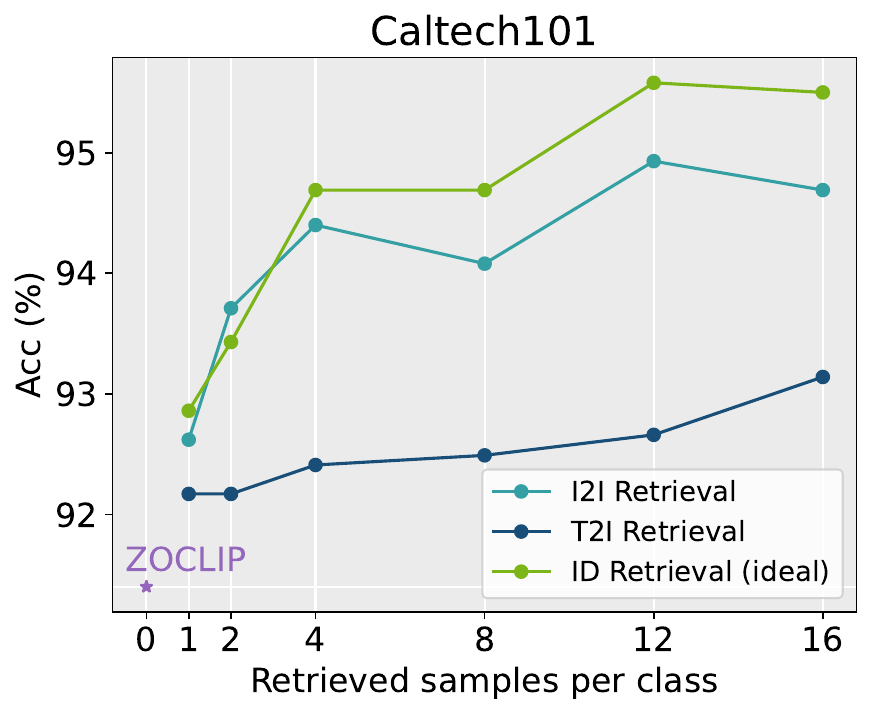}
\end{subfigure}
\begin{subfigure}[b]{0.24\textwidth}
    \includegraphics[width=\textwidth]{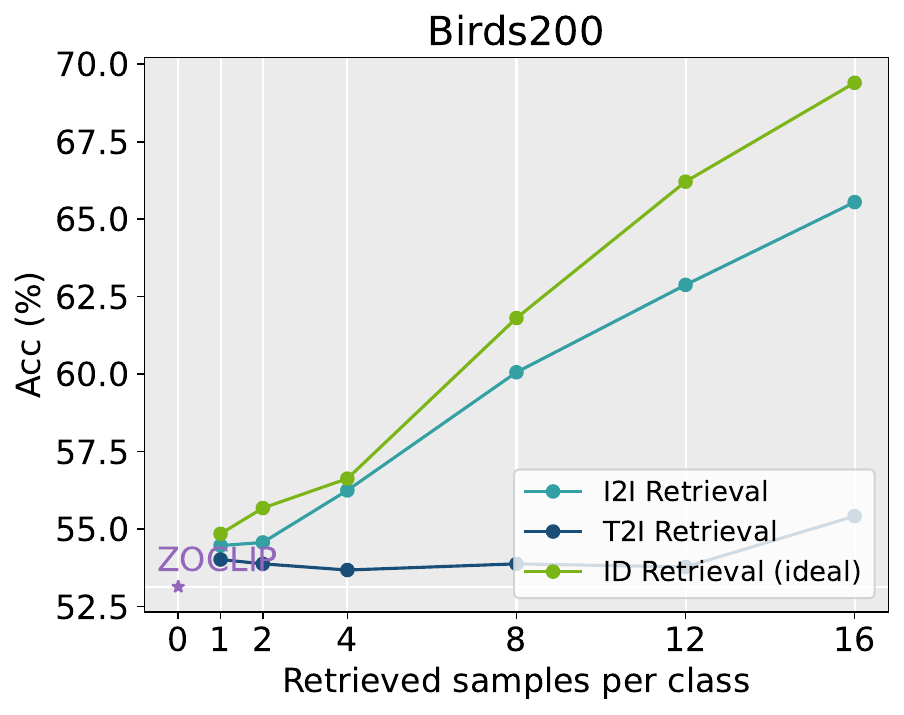}
\end{subfigure}
\begin{subfigure}[b]{0.24\textwidth}
    \includegraphics[width=\textwidth]{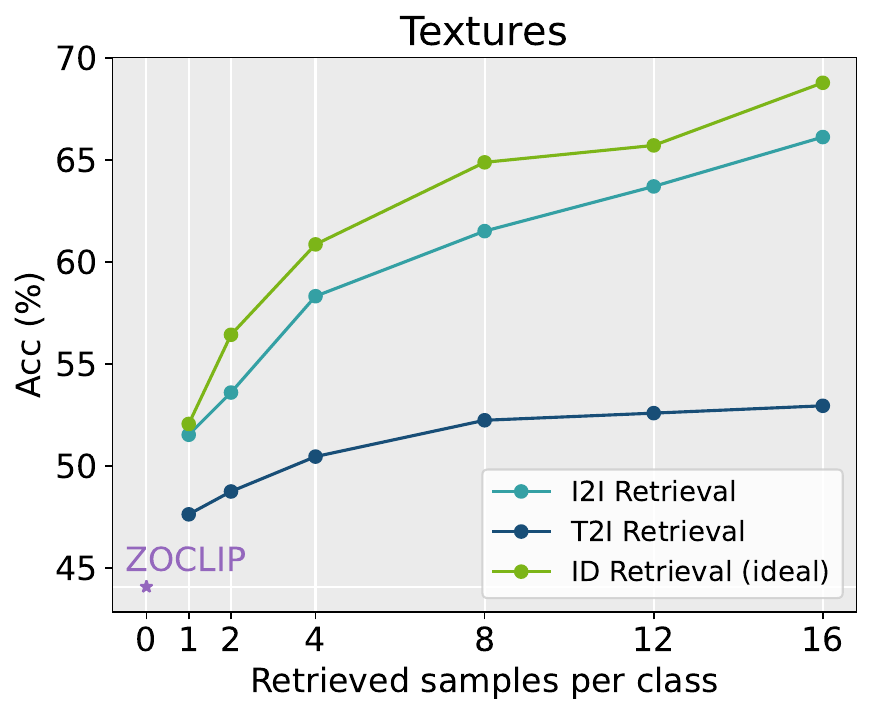}
\end{subfigure}
\begin{subfigure}[b]{0.24\textwidth}
    \includegraphics[width=\textwidth]{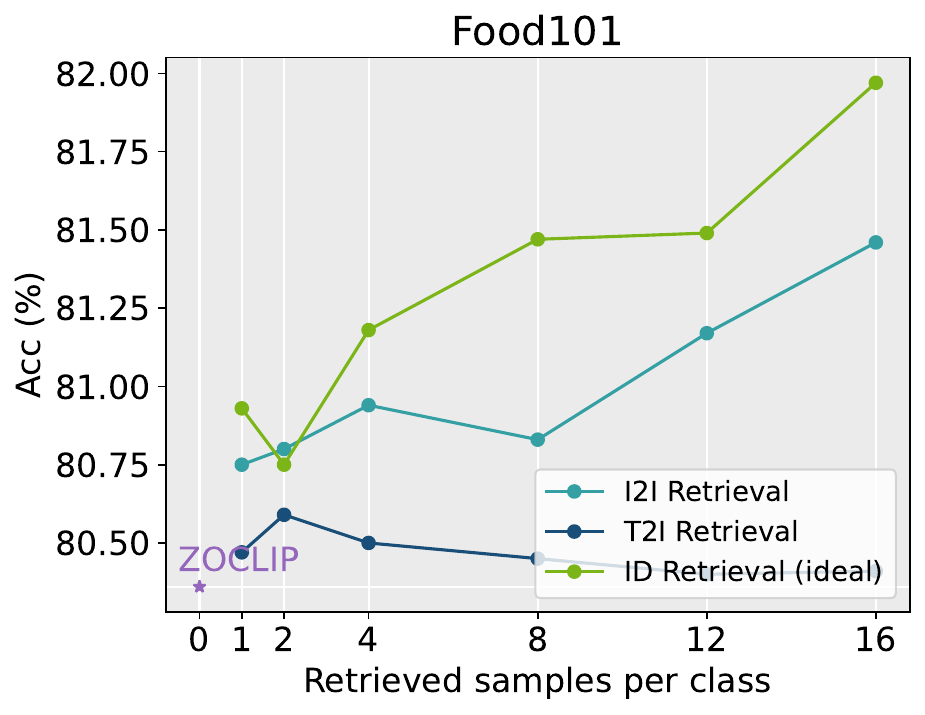}
\end{subfigure}
\begin{subfigure}[b]{0.24\textwidth}
    \includegraphics[width=\textwidth]{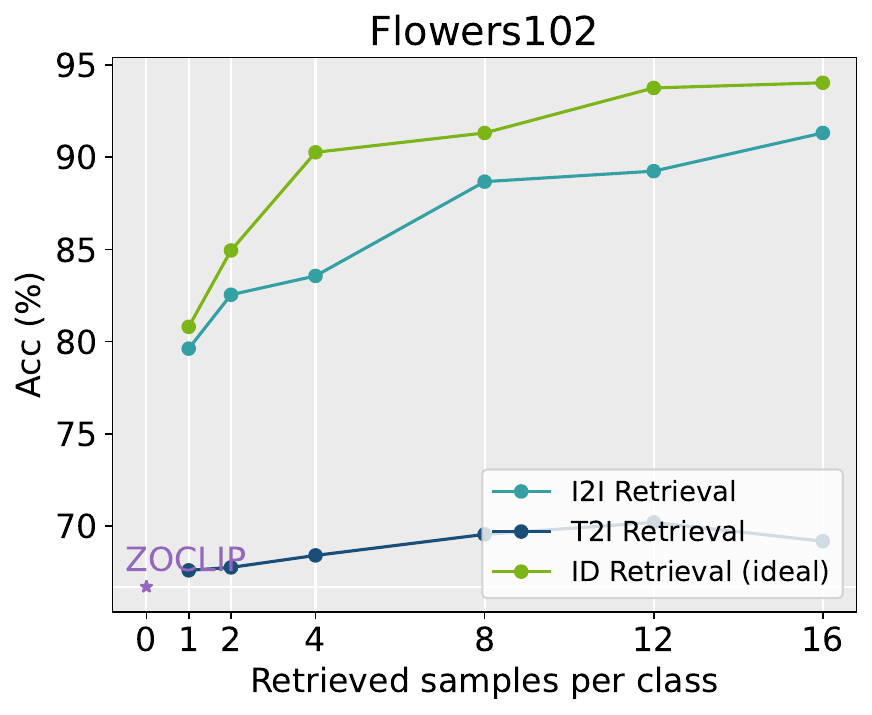}
\end{subfigure}
\begin{subfigure}[b]{0.24\textwidth}
    \includegraphics[width=\textwidth]{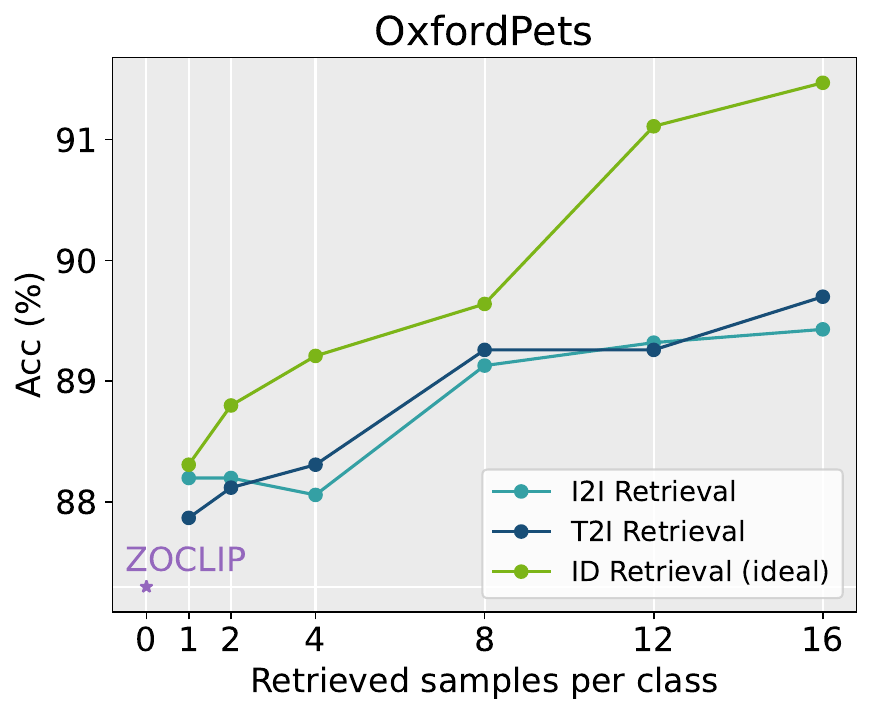}
\end{subfigure}
\begin{subfigure}[b]{0.24\textwidth}
    \includegraphics[width=\textwidth]{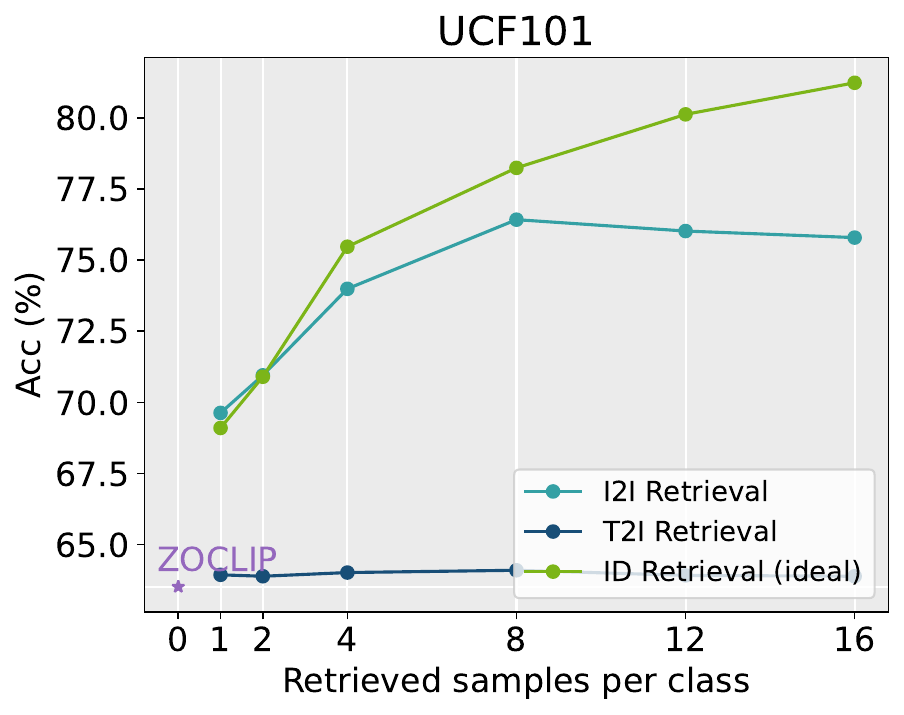}
\end{subfigure}
\caption{Impact of model architecture. Results are based on ViT-B/32 (feature cache finetuned).}
\label{fig:b32-adaptf}
\end{figure*}

\begin{figure*}[ht]
\centering
\begin{subfigure}[b]{0.24\textwidth}
    \includegraphics[width=\textwidth]{figures/ablation/vitb16/tip/average.pdf}
\end{subfigure}
\begin{subfigure}[b]{0.24\textwidth}
    \includegraphics[width=\textwidth]{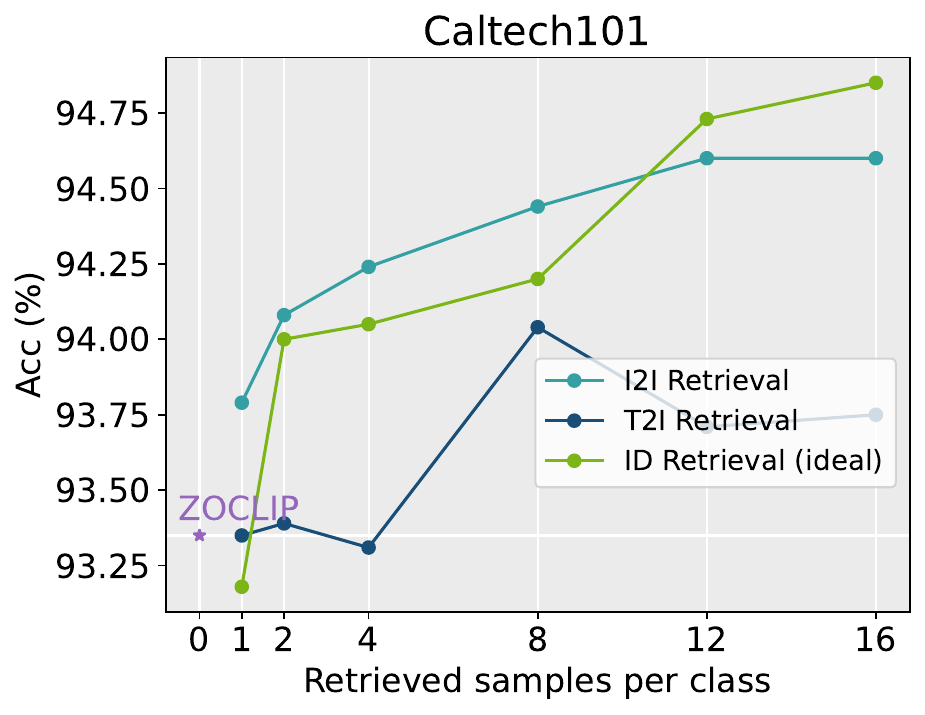}
\end{subfigure}
\begin{subfigure}[b]{0.24\textwidth}
    \includegraphics[width=\textwidth]{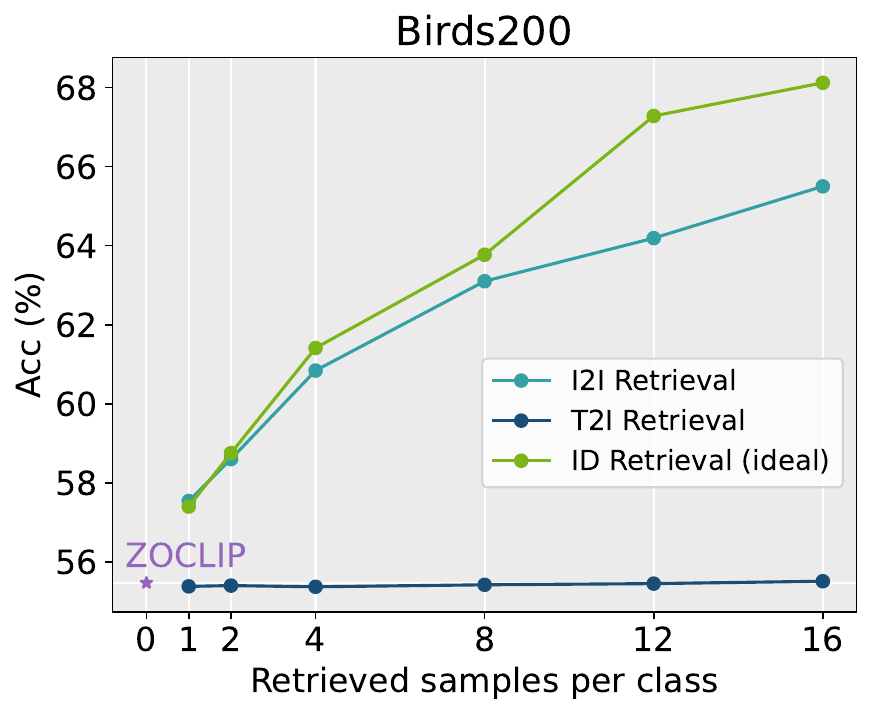}
\end{subfigure}
\begin{subfigure}[b]{0.24\textwidth}
    \includegraphics[width=\textwidth]{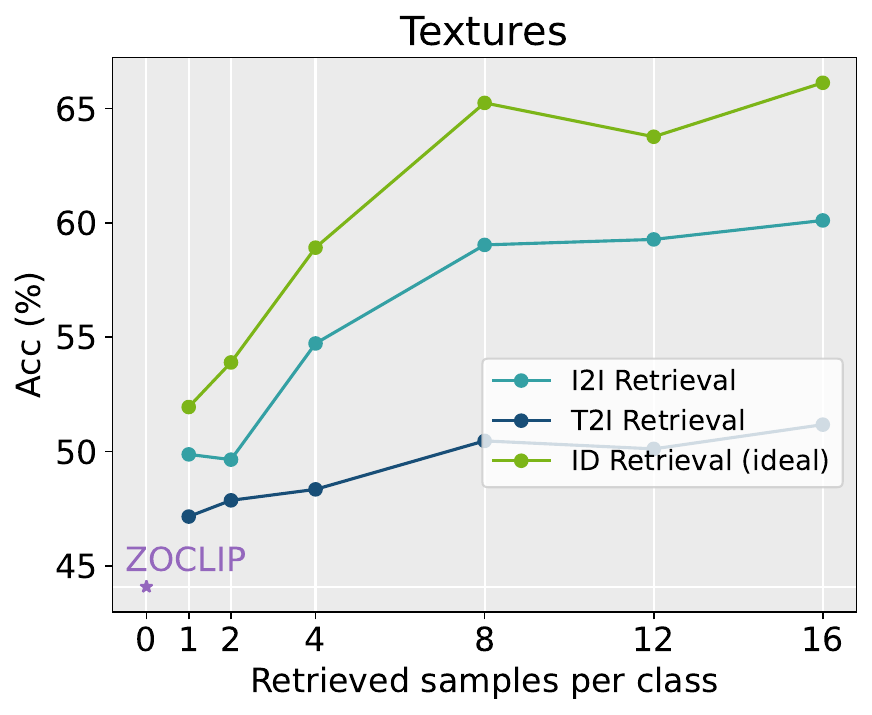}
\end{subfigure}
\begin{subfigure}[b]{0.24\textwidth}
    \includegraphics[width=\textwidth]{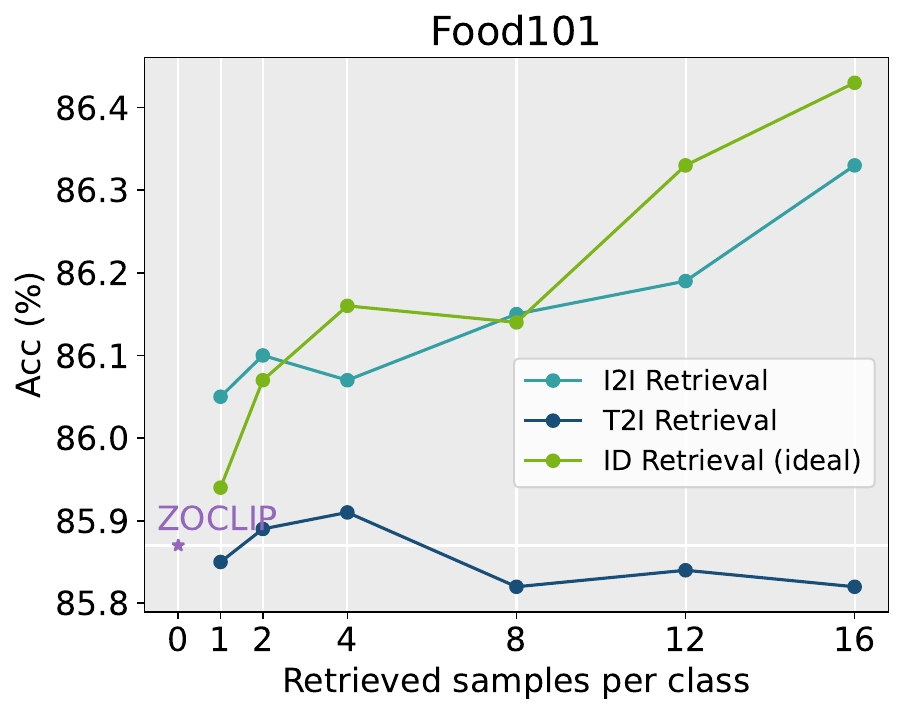}
\end{subfigure}
\begin{subfigure}[b]{0.24\textwidth}
    \includegraphics[width=\textwidth]{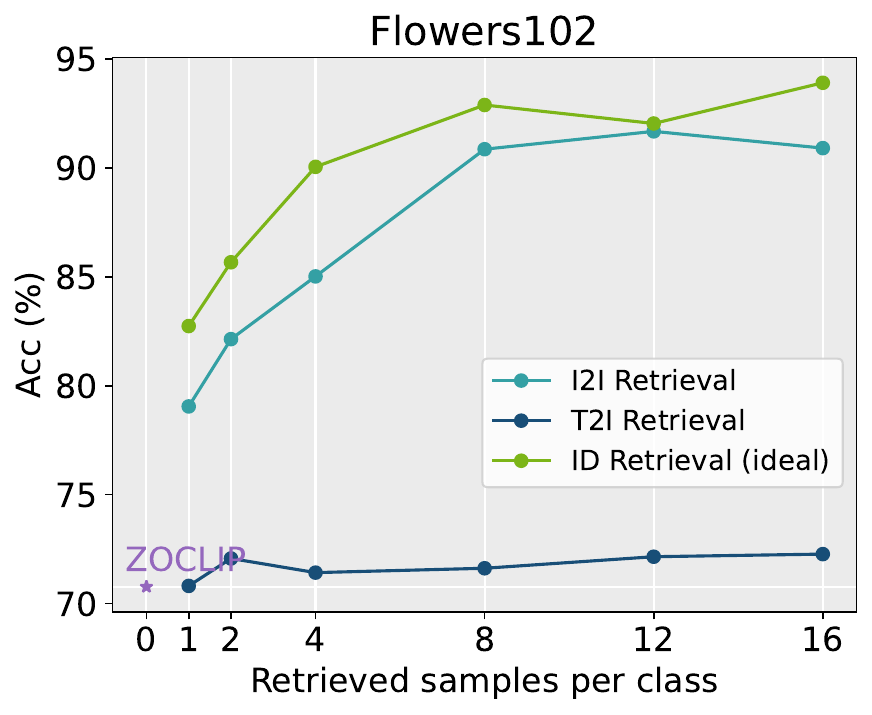}
\end{subfigure}
\begin{subfigure}[b]{0.24\textwidth}
    \includegraphics[width=\textwidth]{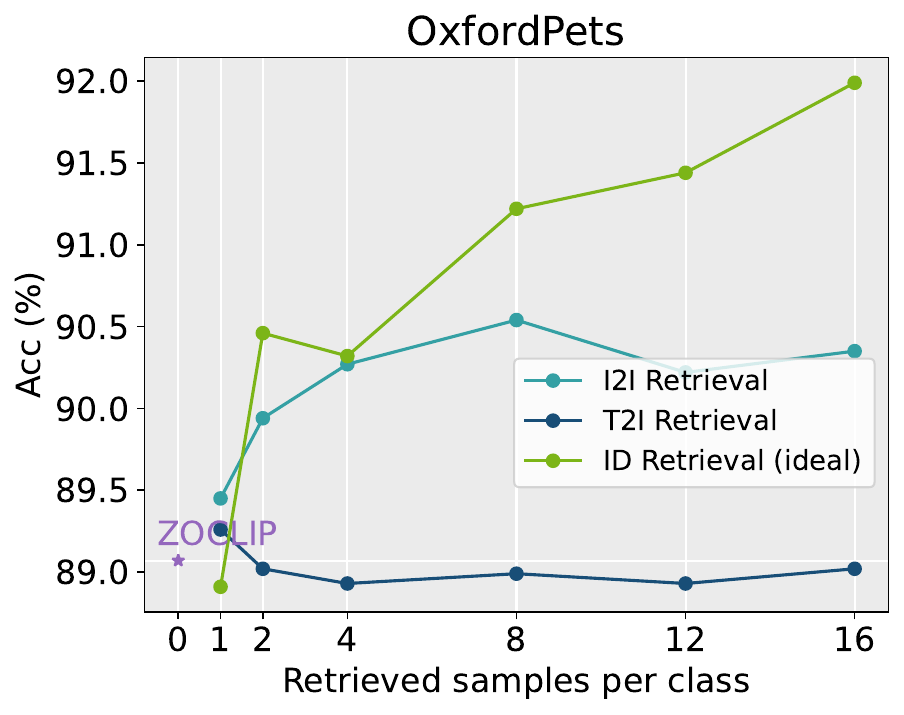}
\end{subfigure}
\begin{subfigure}[b]{0.24\textwidth}
    \includegraphics[width=\textwidth]{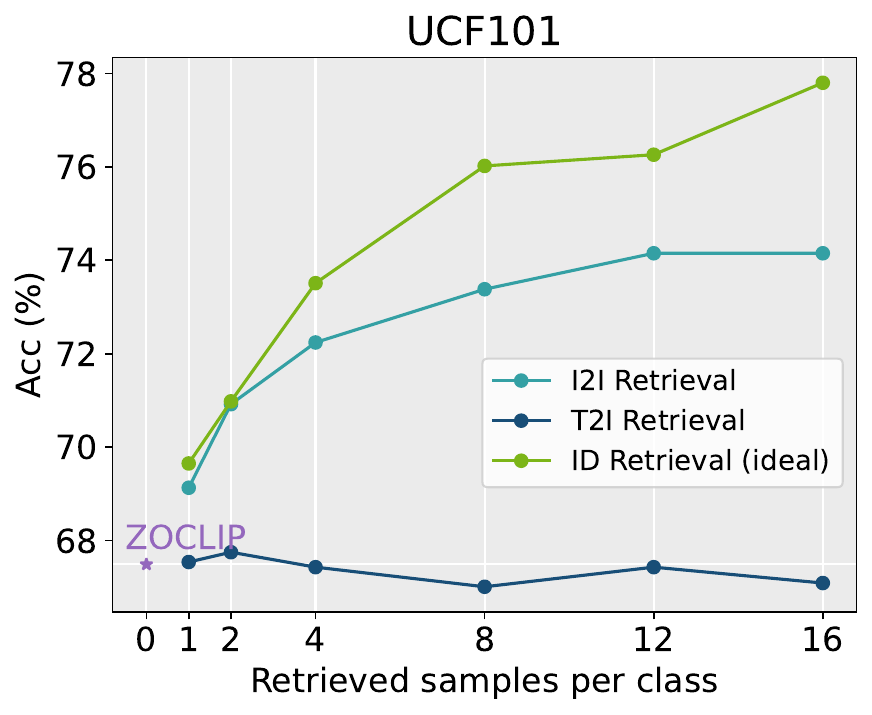}
\end{subfigure}
\caption{Impact of model architecture. Results are based on ViT-B/16 (training-free).}
\label{fig:b16-adapt}
\end{figure*}

\begin{figure*}[ht]
\centering
\begin{subfigure}[b]{0.24\textwidth}
    \includegraphics[width=\textwidth]{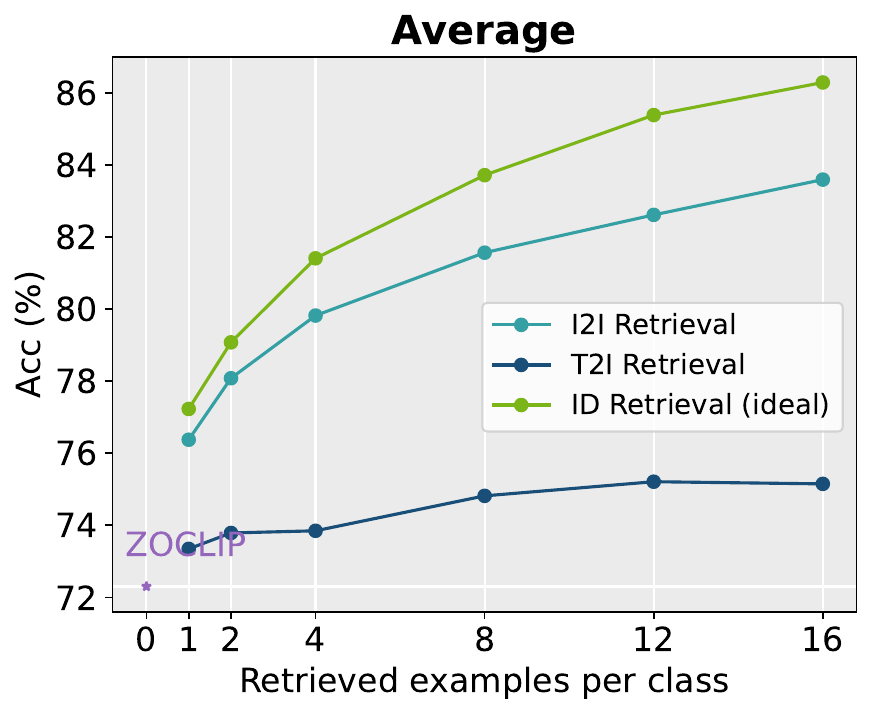}
\end{subfigure}
\begin{subfigure}[b]{0.24\textwidth}
    \includegraphics[width=\textwidth]{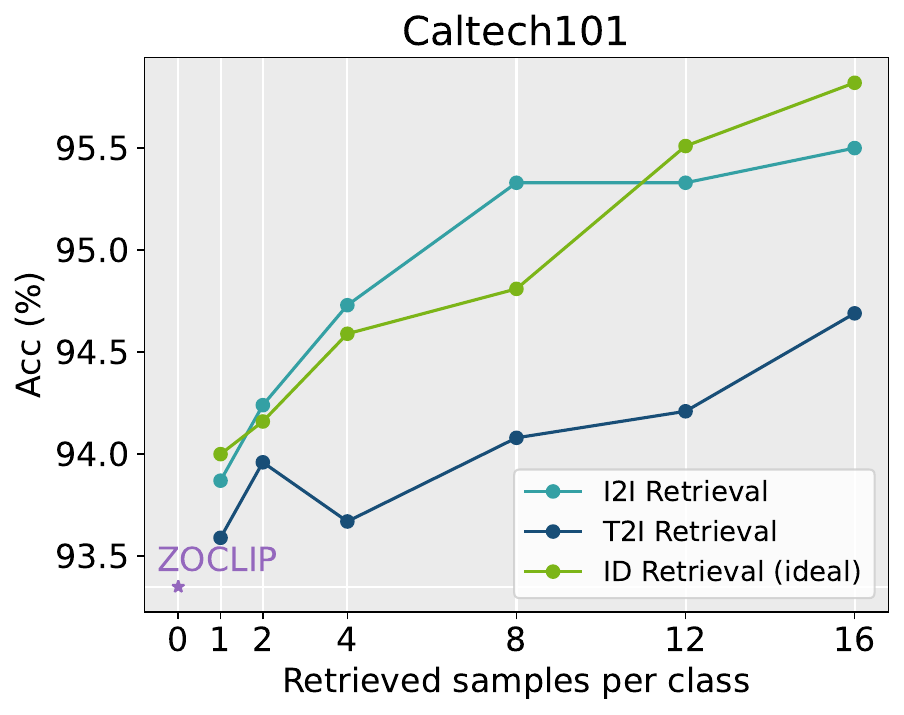}
\end{subfigure}
\begin{subfigure}[b]{0.24\textwidth}
    \includegraphics[width=\textwidth]{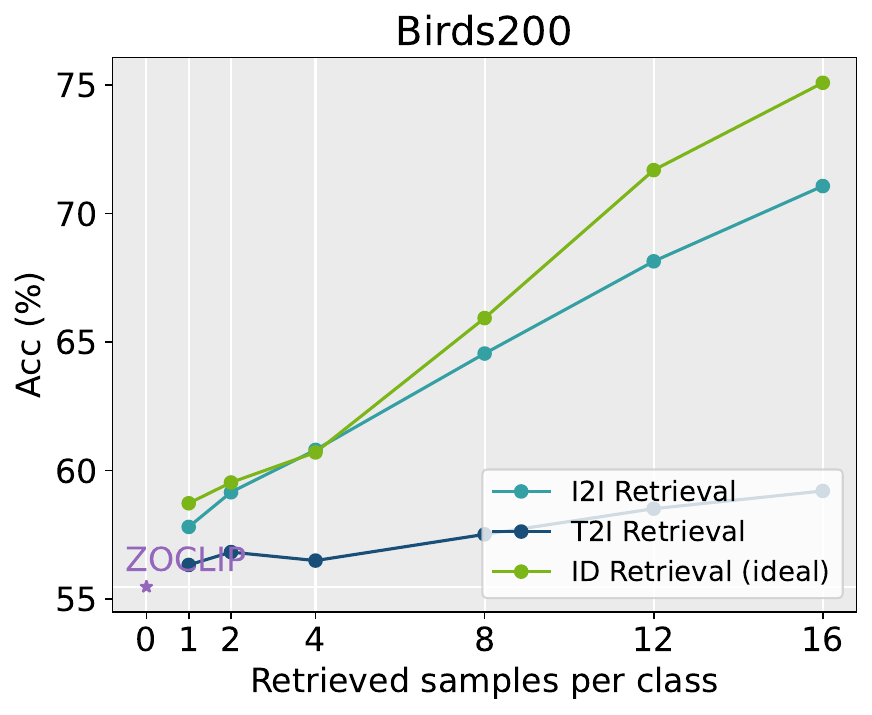}
\end{subfigure}
\begin{subfigure}[b]{0.24\textwidth}
    \includegraphics[width=\textwidth]{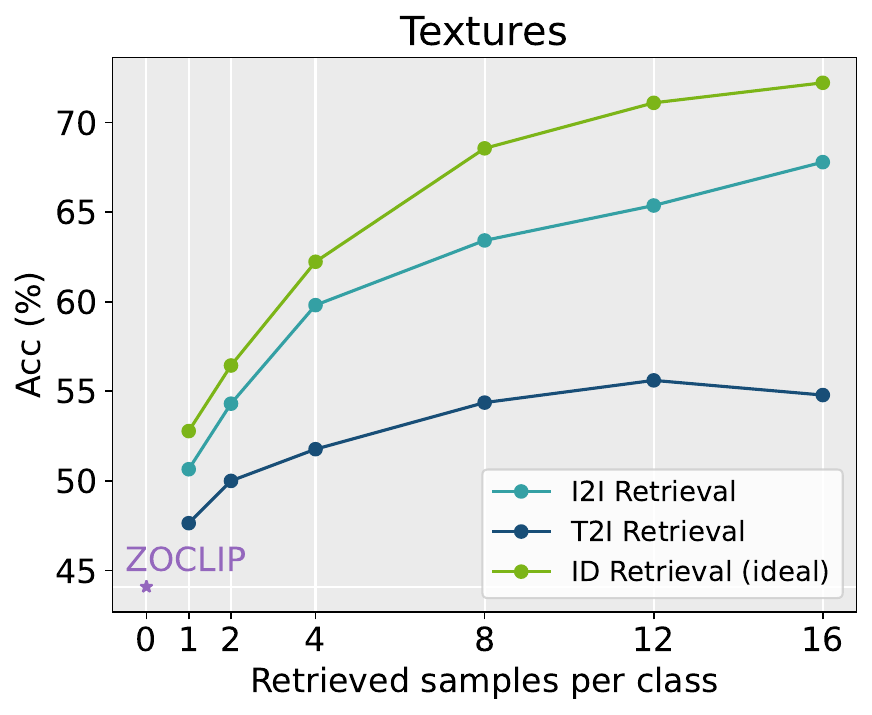}
\end{subfigure}
\begin{subfigure}[b]{0.24\textwidth}
    \includegraphics[width=\textwidth]{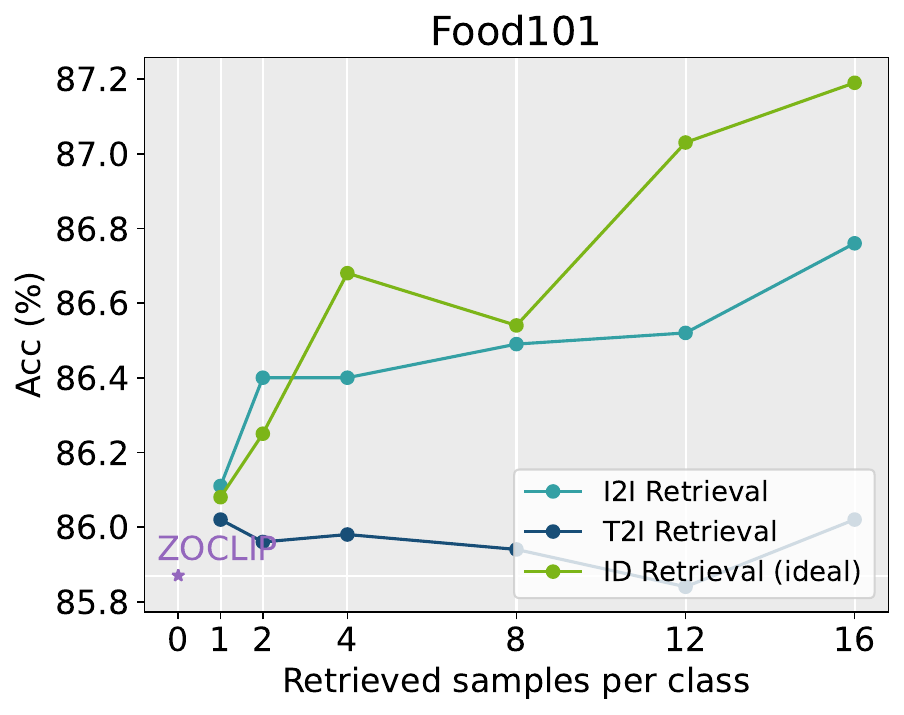}
\end{subfigure}
\begin{subfigure}[b]{0.24\textwidth}
    \includegraphics[width=\textwidth]{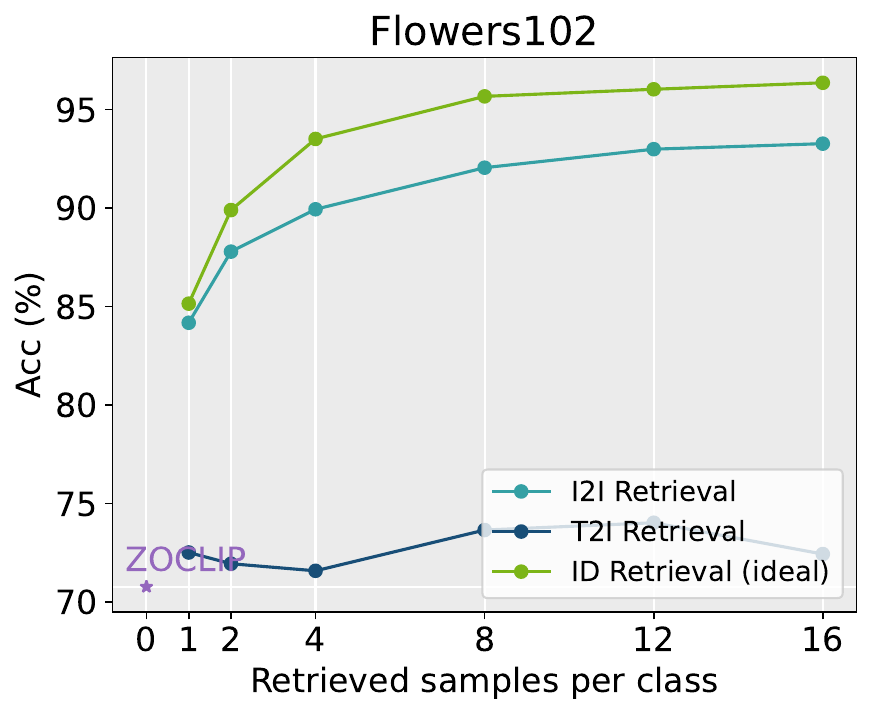}
\end{subfigure}
\begin{subfigure}[b]{0.24\textwidth}
    \includegraphics[width=\textwidth]{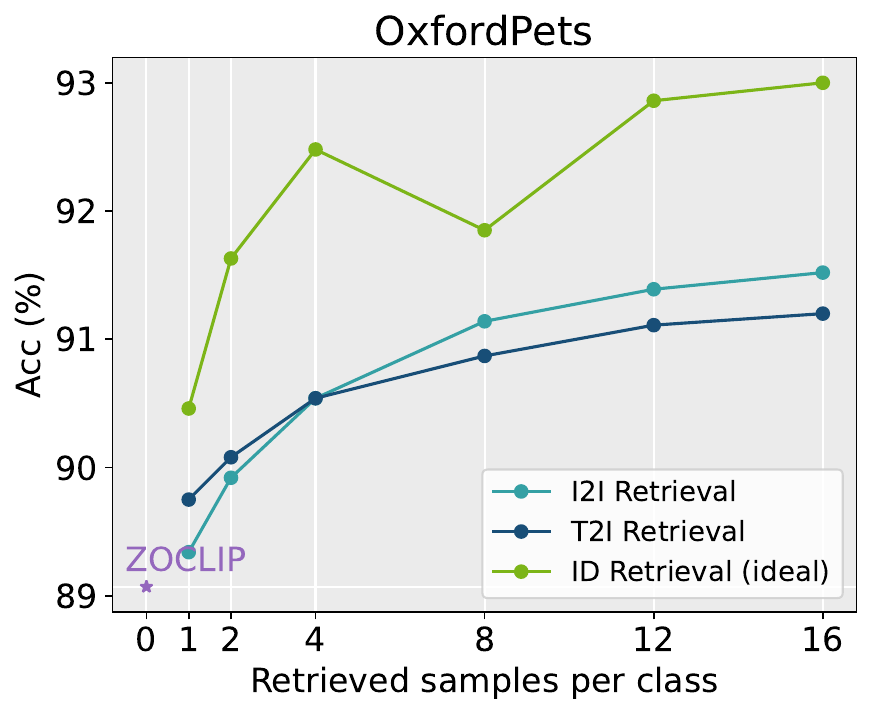}
\end{subfigure}
\begin{subfigure}[b]{0.24\textwidth}
    \includegraphics[width=\textwidth]{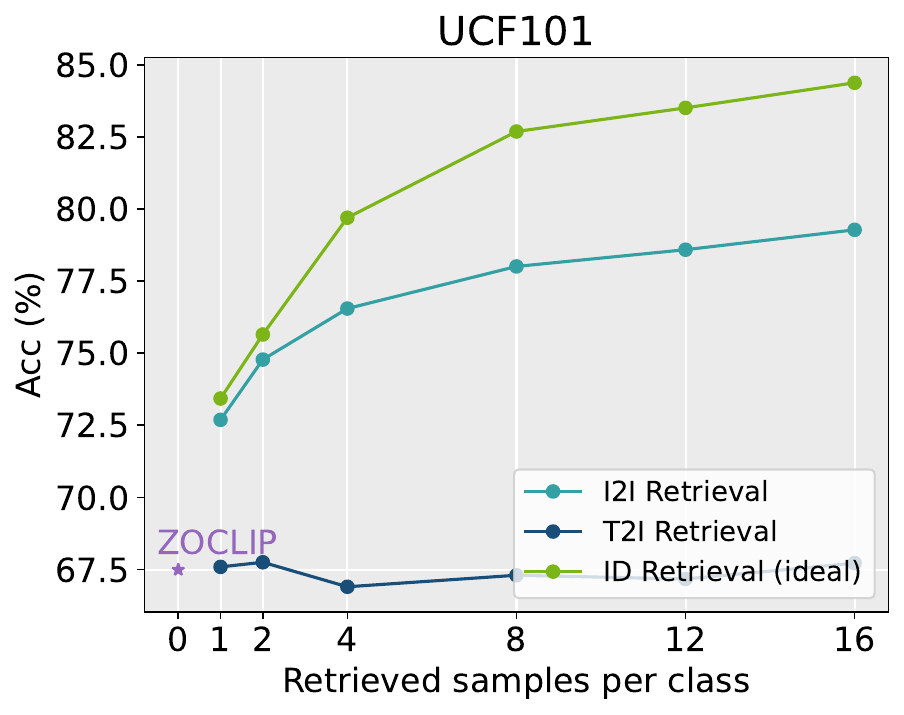}
\end{subfigure}
\caption{Impact of model architecture. Results are based on ViT-B/16 (feature cache finetuned).}
\label{fig:b16-adaptf}
\end{figure*}

\begin{figure*}[ht]
\centering
\begin{subfigure}[b]{0.24\textwidth}
    \includegraphics[width=\textwidth]{figures/ablation/vitl14/tip/average.pdf}
\end{subfigure}
\begin{subfigure}[b]{0.24\textwidth}
    \includegraphics[width=\textwidth]{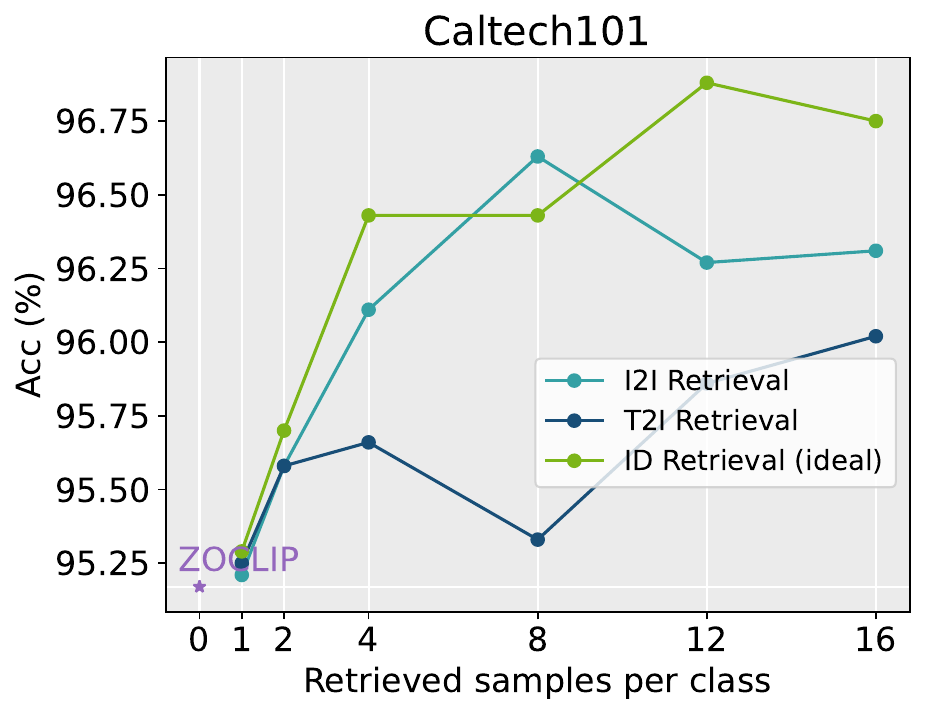}
\end{subfigure}
\begin{subfigure}[b]{0.24\textwidth}
    \includegraphics[width=\textwidth]{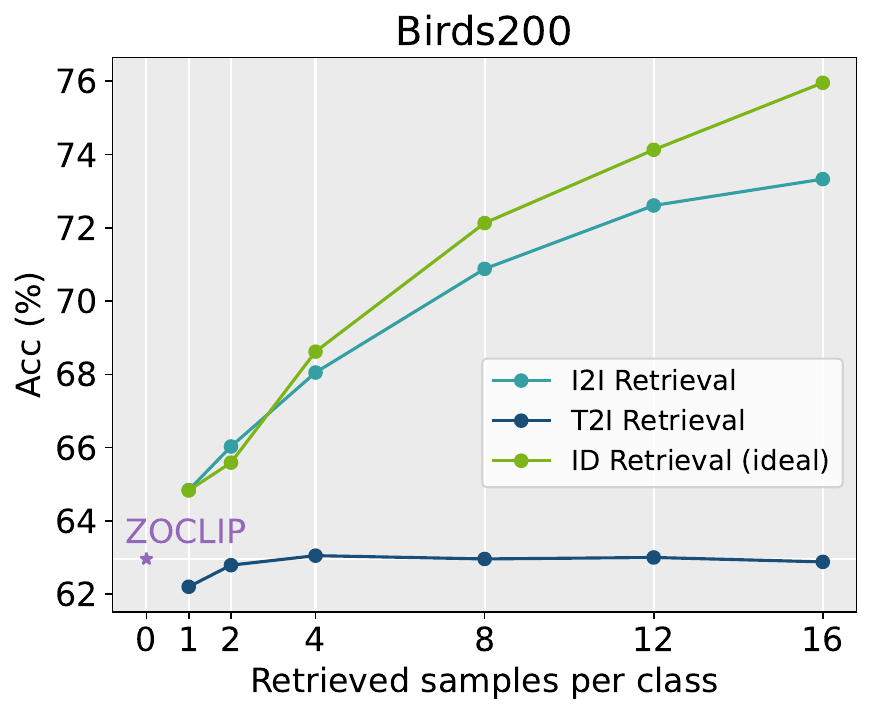}
\end{subfigure}
\begin{subfigure}[b]{0.24\textwidth}
    \includegraphics[width=\textwidth]{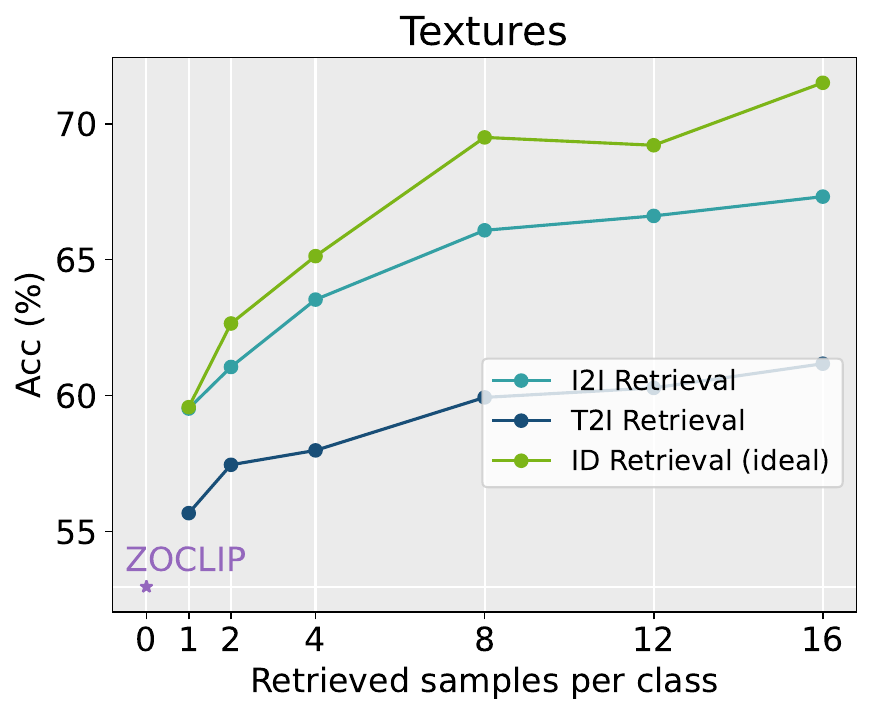}
\end{subfigure}
\begin{subfigure}[b]{0.24\textwidth}
    \includegraphics[width=\textwidth]{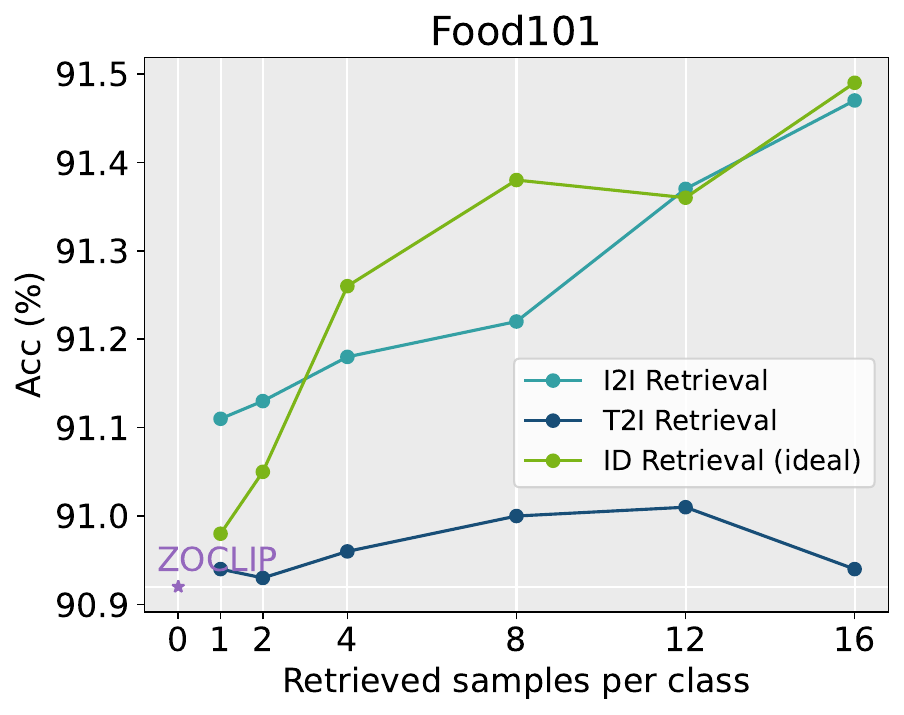}
\end{subfigure}
\begin{subfigure}[b]{0.24\textwidth}
    \includegraphics[width=\textwidth]{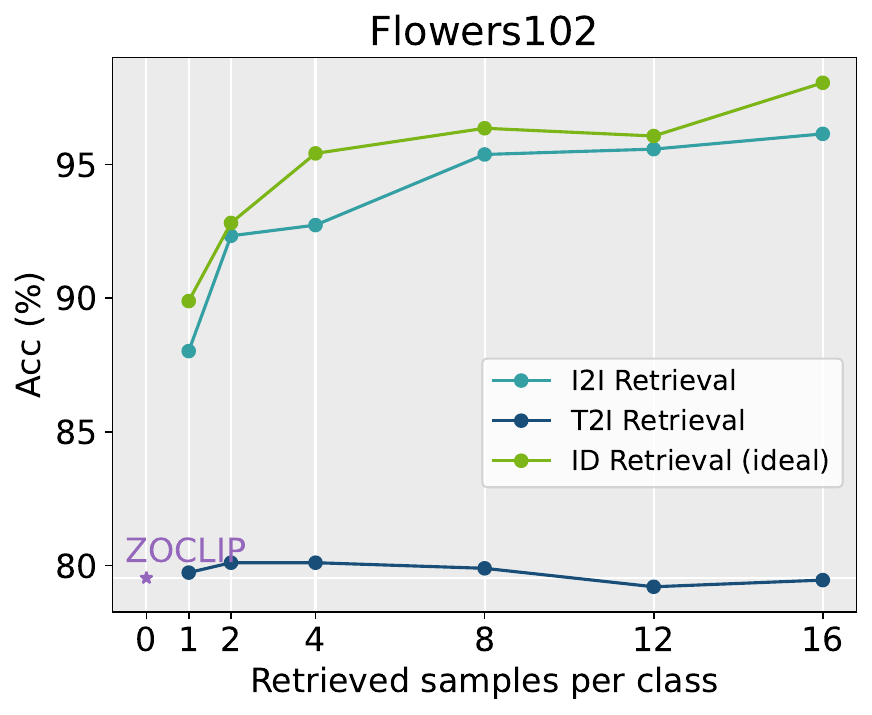}
\end{subfigure}
\begin{subfigure}[b]{0.24\textwidth}
    \includegraphics[width=\textwidth]{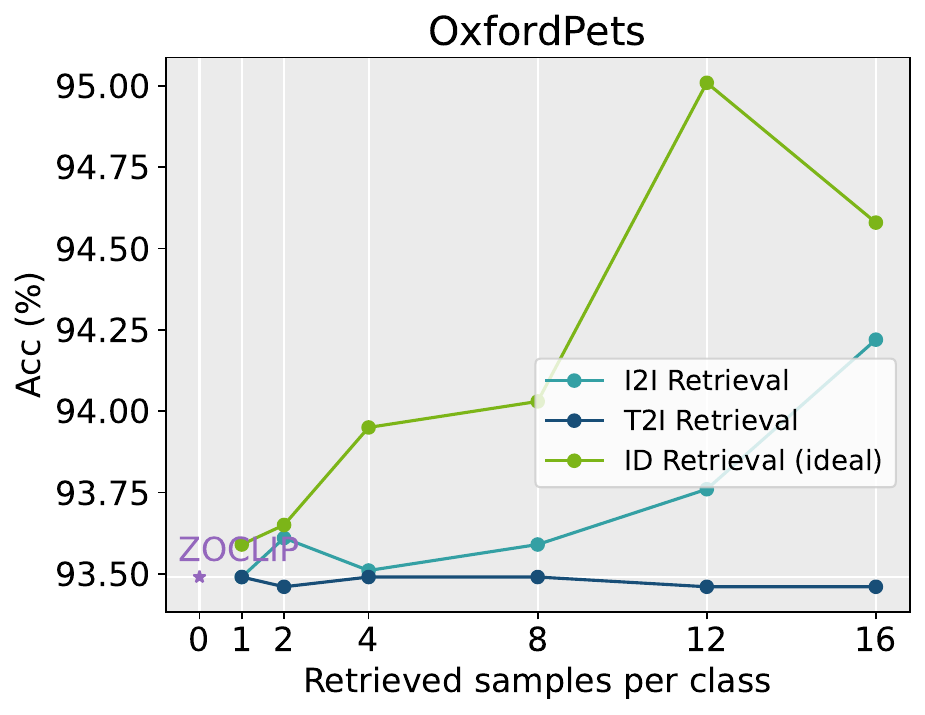}
\end{subfigure}
\begin{subfigure}[b]{0.24\textwidth}
    \includegraphics[width=\textwidth]{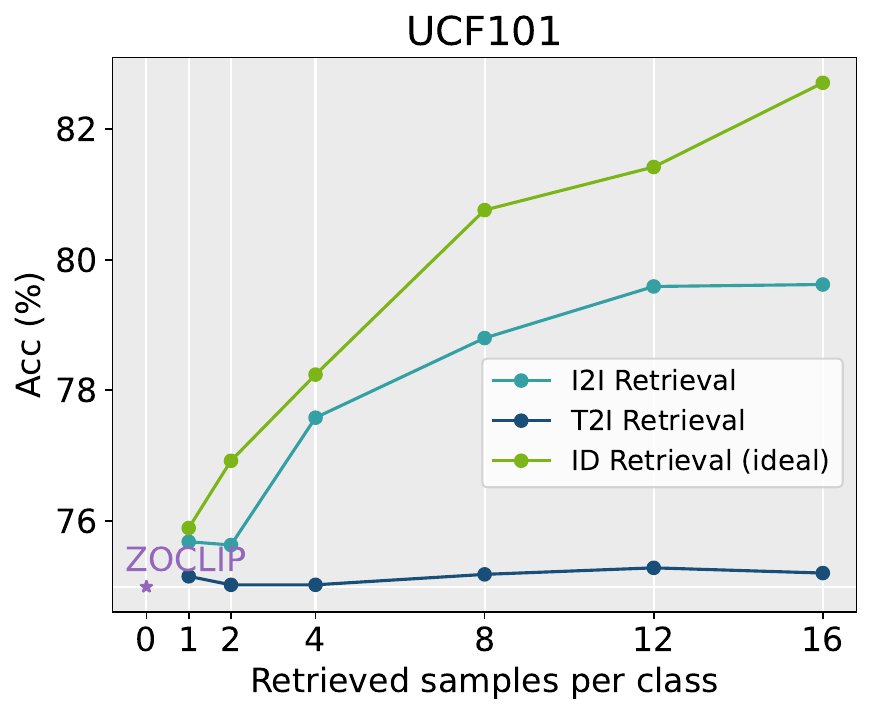}
\end{subfigure}
\caption{Impact of model architecture. Results are based on ViT-L/14 (training-free).}
\label{fig:l14-adapt}
\end{figure*}

\begin{figure*}[ht]
\centering
\begin{subfigure}[b]{0.24\textwidth}
    \includegraphics[width=\textwidth]{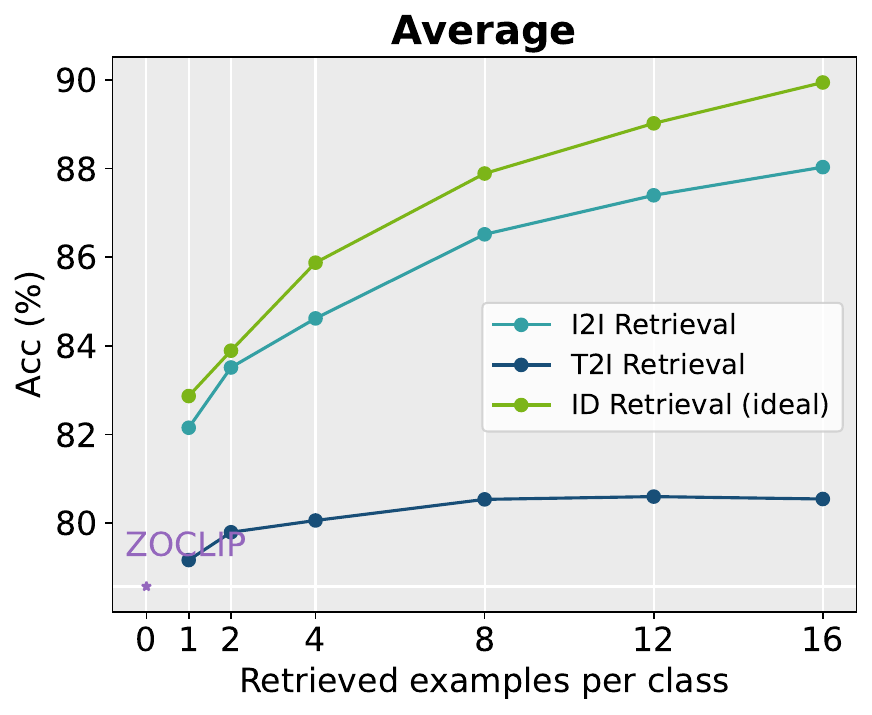}
\end{subfigure}
\begin{subfigure}[b]{0.24\textwidth}
    \includegraphics[width=\textwidth]{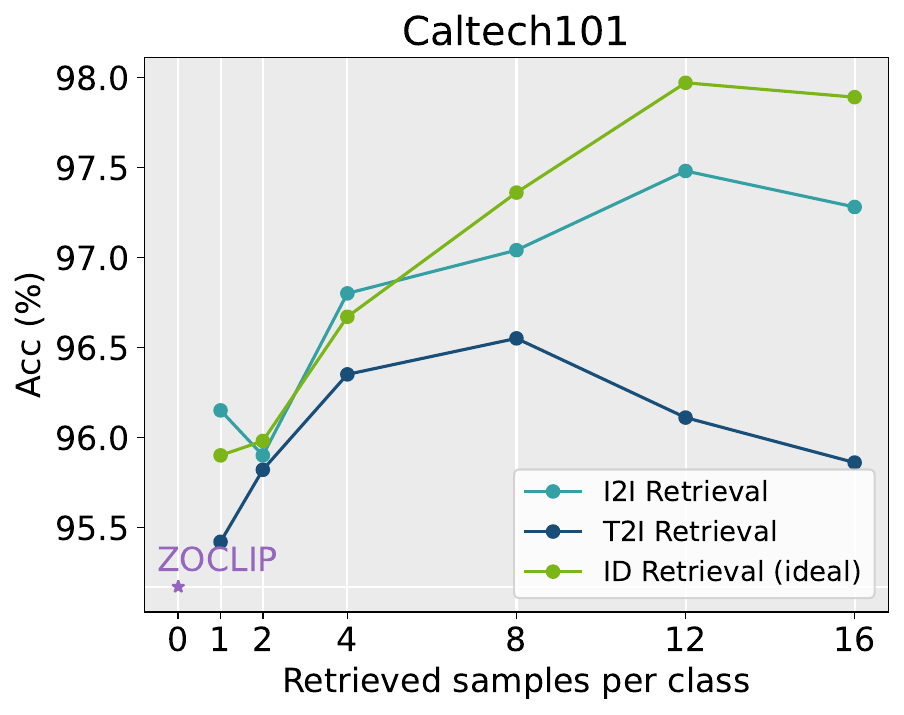}
\end{subfigure}
\begin{subfigure}[b]{0.24\textwidth}
    \includegraphics[width=\textwidth]{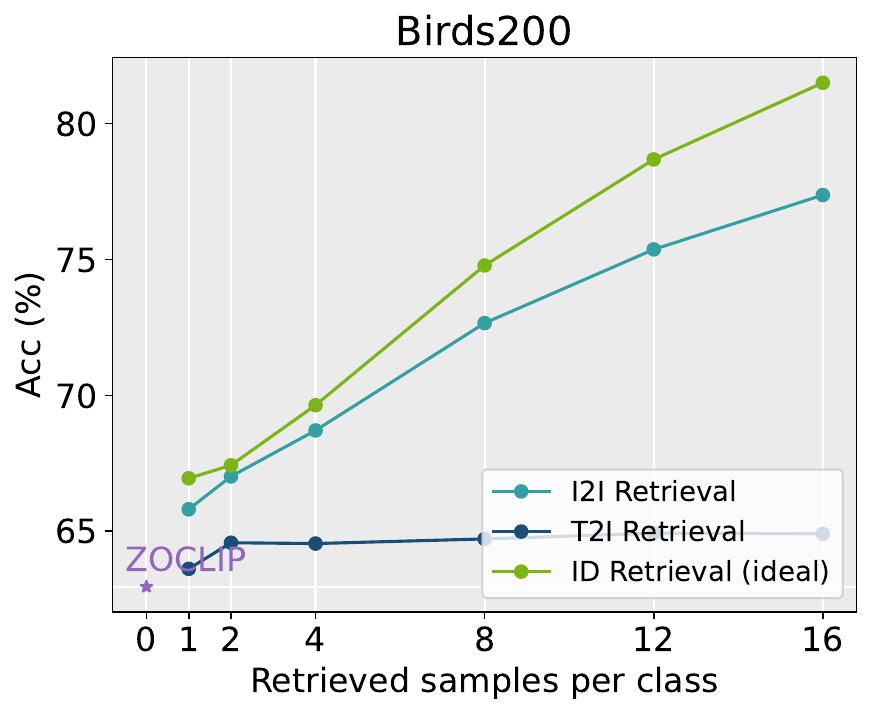}
\end{subfigure}
\begin{subfigure}[b]{0.24\textwidth}
    \includegraphics[width=\textwidth]{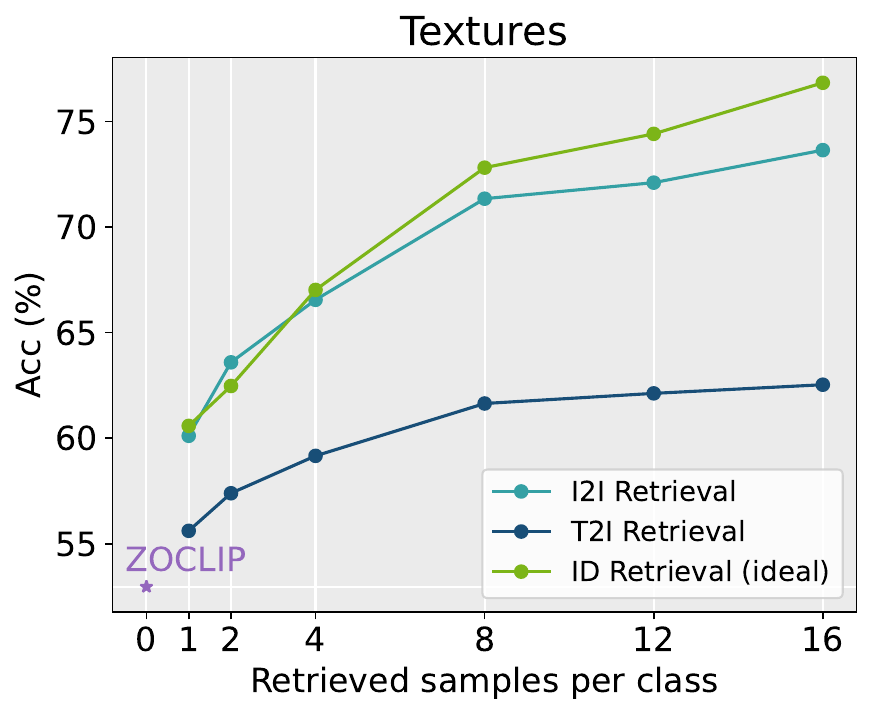}
\end{subfigure}
\begin{subfigure}[b]{0.24\textwidth}
    \includegraphics[width=\textwidth]{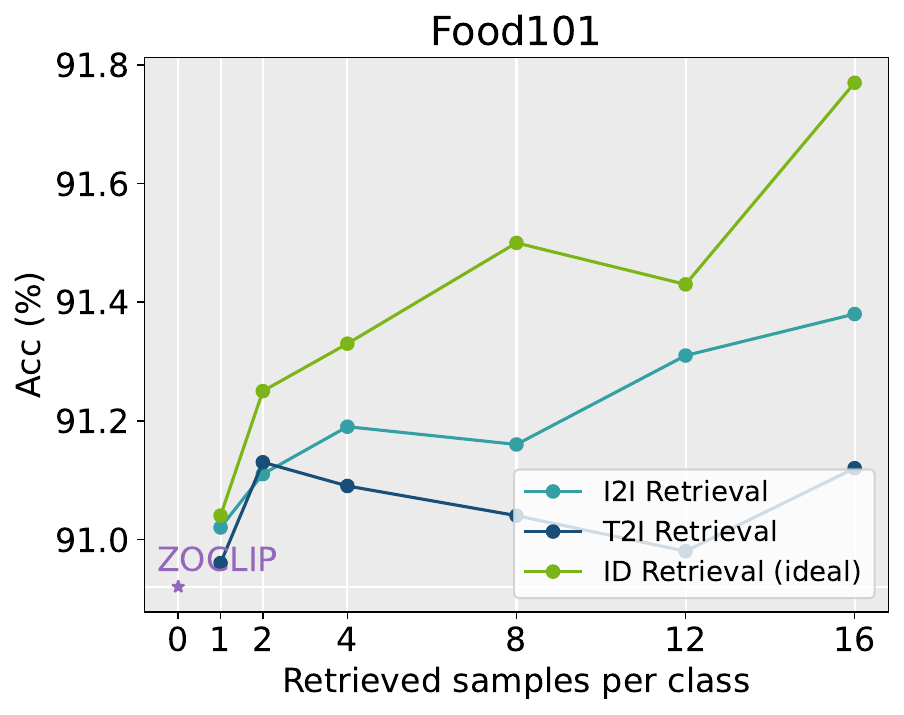}
\end{subfigure}
\begin{subfigure}[b]{0.24\textwidth}
    \includegraphics[width=\textwidth]{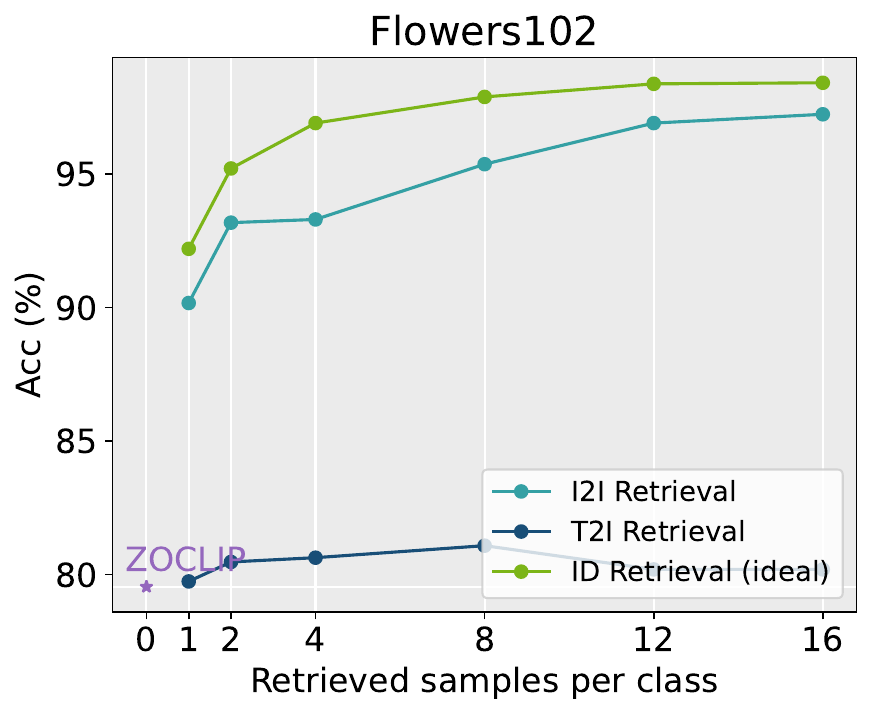}
\end{subfigure}
\begin{subfigure}[b]{0.24\textwidth}
    \includegraphics[width=\textwidth]{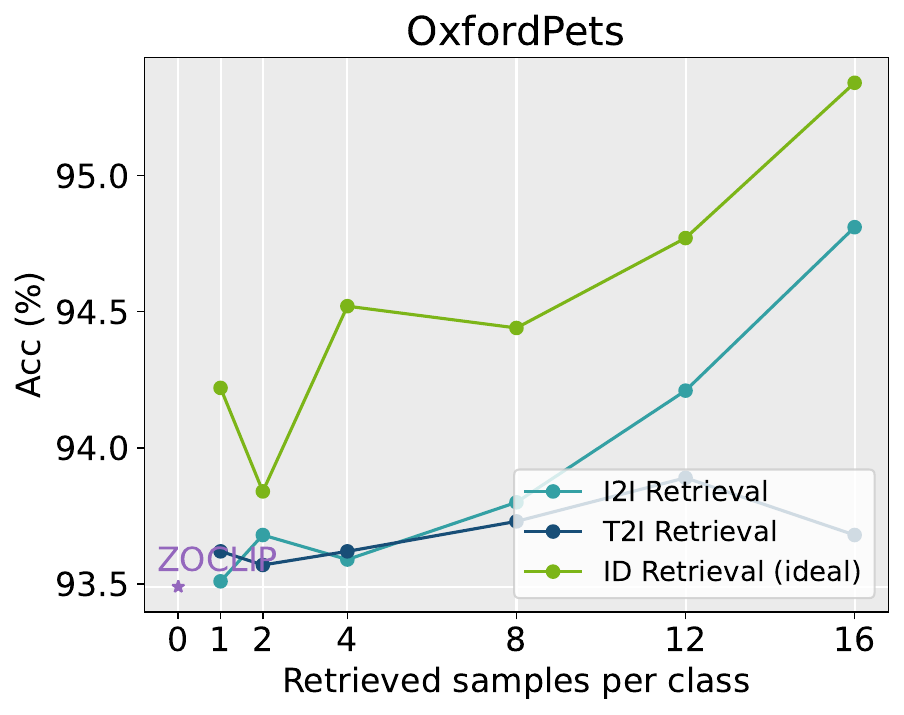}
\end{subfigure}
\begin{subfigure}[b]{0.24\textwidth}
    \includegraphics[width=\textwidth]{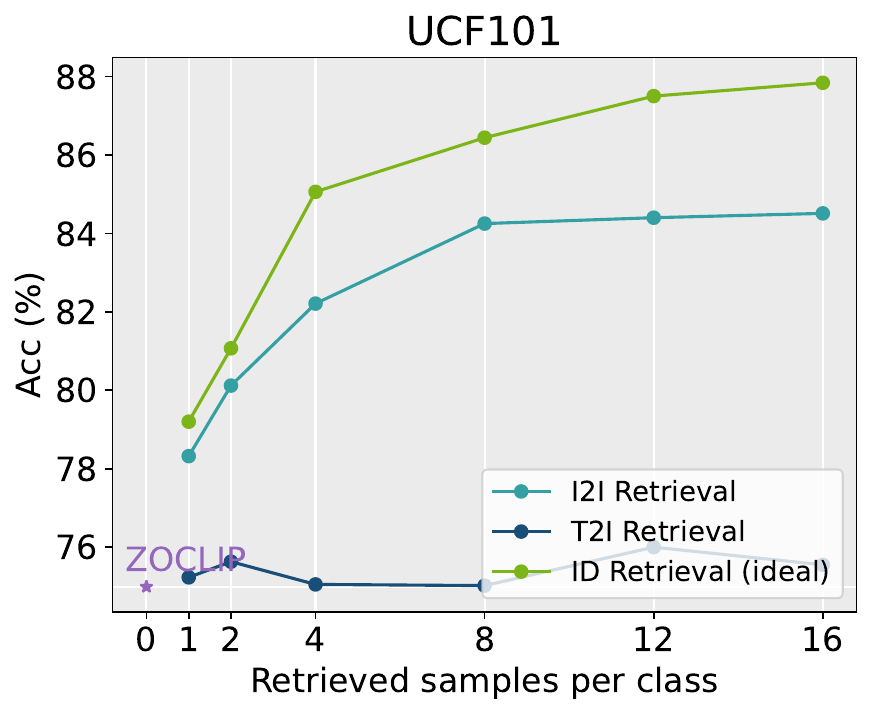}
\end{subfigure}
\caption{Impact of model architecture. Results are based on ViT-L/14 (feature cache finetuned).}
\label{fig:l14-adaptf}
\end{figure*}

\section{Impact of Architecture}
\label{sec:arch_ind}
In Section~\ref{sec:discuss}, we show the average performance over all datasets for I2I retrieval and T2I retrieval under different CLIP backbones and observe consistent trends. The results for individual datasets can be seen in Figure~\ref{fig:b32-adapt} (training-free adaptation based on ViT-B/32), Figure~\ref{fig:b32-adaptf} (training-based adaptation based on ViT-B/32), Figure~\ref{fig:b16-adapt} (training-free adaptation based on ViT-B/16), Figure~\ref{fig:b16-adaptf} (training-based adaptation based on ViT-B/16), Figure~\ref{fig:l14-adapt} (training-free adaptation based on ViT-L/14), and Figure~\ref{fig:l14-adaptf} (training-based adaptation based on ViT-L/14).

\clearpage

\section{Ablation on the ensemble weight scale $\gamma:\alpha$}
\label{sec:ablation_ratio}

In the main paper, we set the ensemble weights as tunable hyperparameters. In this section, we conduct an additional ablation study on the ratio of $\gamma:\alpha \in\{0.1,0.5,1,2,5,7.5,10,15,20,50\}$ across different $\omega \in \{0.1,0.5,1,2,5,10,20,50\}$.  We observe that a moderate $\gamma:\alpha$ ratio yields superior performance and the optimal ratio is dataset-dependent. As a concrete example, Table~\ref{tab:weight_ratio} displays the accuracy on each dataset for various $\gamma:\alpha$. The results are based on the RN50 backbone, 8 shot, and $\omega=2$ with I2I retrieval. For most datasets, a relatively larger ratio (e.g., 5) yields better performance. For Food101, a smaller $\gamma:\alpha$ ratio (e.g., 0.5) suffices.

\begin{table}[hbt]
\centering
\begin{tabular}{lcccccccc}
\toprule
\multirow{2}{*}{$\mathbf{\gamma}:\mathbf{\alpha}$} & \multicolumn{7}{c}{\textbf{Dataset}}\\
\cmidrule(lr){2-8}
 & Caltech 101 & Birds200  & Textures  & Food101   & Flowers102 & OxfordPets & UCF101    \\
\midrule
0.1             & 86.09       & 47.95     & 43.32     & 77.44     & 66.42      & 85.77      & 62.41     \\
0.5             & 87.42       & 49.57     & 44.98     & \textbf{77.65} & 67.52      & 86.37      & 64.23     \\
1.0               & 88.44       & 50.74     & 46.75     & 77.54     & 69.02      & 87.14      & 65.61     \\
2.0               & 89.01       & 52.78     & 48.58     & 77.17     & 71.54      & \textbf{87.35}  & 66.90     \\
5.0              & \textbf{89.21}   & \textbf{53.94} & \textbf{50.77} & 74.52     & 78.08      & 85.04      & \textbf{66.98} \\
7.5             & 88.03       & 53.64     & 50.00     & 71.85     & \textbf{78.28}  & 82.20      & 66.19     \\
10              & 86.77       & 52.23     & 49.17     & 69.33     & 76.17      & 78.6       & 65.27     \\
15              & 85.31       & 50.17     & 47.87     & 65.08     & 74.18      & 73.15      & 63.79     \\
20              & 84.38       & 48.55     & 46.87     & 62.11     & 75.96      & 69.04      & 62.46     \\
50              & 82.11       & 43.53     & 44.62     & 54.17     & 76.98      & 56.75      & 58.68    \\
\bottomrule
\end{tabular}
\caption{Ablation on the ensemble weight scale $\gamma:\alpha$.}
\label{tab:weight_ratio}
\end{table}

\section{Extension beyond CLIP-like models}
\label{sec:ext_blip2}
In the main paper, we mainly consider pre-trained CLIP-like models due to their wide applicability. To explore whether our findings can be generalized to other vision-language models, in this section, we conduct experiments based on BLIP-2~\cite{li2023blip}. Our experiments are based on the feature extraction pipeline from \url{https://github.com/salesforce/LAVIS}. Table~\ref{tab:blip2} displays the performance (accuracy) when only using the logit from the zero-shot model (ZOC), only using the logit from the retrieval cache (RET), and using an ensemble of logits (Ensemble) for T2I and I2I retrieval, respectively.  The same observations also hold for BLIP-2: (1) I2I retrieval consistently outperforms T2I retrieval; (2) Ensemble with the zero-shot prediction is essential. The results are based on 8 shot, and we observe that similar trends hold consistently across other shots from 2 to 16.

\begin{table}[hbt]
\centering
\begin{tabular}{lccccc}
\toprule
\multirow{2}{*}{\textbf{Dataset}} & \multicolumn{5}{c}{\textbf{Method}}\\
\cmidrule(lr){2-6}
   & ZOCLIP & RET (T2I) & RET (I2I) & Ensemble (T2I) & Ensemble (I2I) \\ \hline
Caltech101 & 88.19  & 86.61     & 91.44     & 90.14          & 91.76          \\ \hline
Textures   & 46.16  & 50.95     & 58.10      & 53.31          & 61.41          \\ \hline
Food101    & 73.39  & 75.81     & 71.94     & 79.66          & 80.56          \\ \hline
Flowers102    & 41.41  & 59.44     & 83.56     & 62.53          & 85.87          \\ \hline
UCF101     & 67.57  & 68.49     & 73.78     & 70.63          & 73.17          \\ 
\bottomrule
\end{tabular}
\caption{Extension of findings beyond CLIP-like models. We evaluate the performance of pre-trained BLIP-2 on diverse datasets. The two key observations still hold: (1) I2I retrieval consistently outperforms T2I retrieval; (2) Ensemble with the zero-shot prediction is essential.}
\label{tab:blip2}
\end{table}

\end{document}